\documentclass[9pt]{article}
\usepackage{enumerate}
\usepackage[top=0.8in, bottom=1.1in]{geometry}
\usepackage[OT1]{fontenc}
\usepackage{amsmath,amssymb}
\usepackage{mathtools}
\usepackage{natbib}
\usepackage[usenames]{color}
\usepackage{dsfont}
\usepackage{pifont}
\usepackage{pgfplots}
\usepackage{smile}
\usepackage{multirow}
\usepackage{rotating}
\usepackage{enumerate}
\usepackage{esvect}
\usepackage{tikz}
\usetikzlibrary{patterns}
\usetikzlibrary{arrows}
\usepackage{color,xcolor}
\usepackage{colortbl}
\usepackage[colorlinks,
            pagebackref, 
            linktoc=page,
            linkcolor=red,
            anchorcolor=red,
            citecolor=blue
            ]{hyperref}
\usepackage{algorithm}
\usepackage{algorithmic}
\usepackage{cancel}

\usepackage{paralist}

\definecolor{green2}{HTML}{c2fdfe}
\definecolor{green1}{HTML}{bce672}
\definecolor{yellow1}{HTML}{f5dd6f}
\definecolor{blue1}{HTML}{3eede7}
\definecolor{orange1}{HTML}{e77c4b}
\definecolor{red1}{HTML}{f47983}

\def\supp{\mathop{\text{supp}}}

\long\def\comment#1{}

\def\cS{{\mathcal{S}}}

\newcommand{\bel}{\begin{eqnarray}\label}
\newcommand{\eel}{\end{eqnarray}}
\newcommand{\bes}{\begin{eqnarray*}}
\newcommand{\ees}{\end{eqnarray*}}

\let\emptyset\varnothing
\let\hat\widehat % overwrite original hat 
\let\tilde\widetilde
\def\mid{\,|\,}

\def\GG{{\mathbb G}}

\def\supp{\mathop{\text{supp}\kern.2ex}}
\def\argmin{\mathop{\text{\rm arg\,min}}}
\def\argmax{\mathop{\text{\rm arg\,max}}}

\def\supp{\mathop{\text{supp}}}

\def\M{{\mathrm M}}

\newcommand{\CR}{\textnormal{CR}}

\newcommand{\eqw}[1]{\stackrel{\mathclap{\textnormal{#1}}}{=}}
\newcommand{\Blin}{\mathbb{B}_{\textnormal{lin}}}
\newcommand{\Glin}{\mathbb{G}_{\textnormal{lin}}}
\newcommand{\Pilin}{\Pi_{\textnormal{lin}}}

\def\##1\#{\begin{align}#1\end{align}}
\def\$#1\${\begin{align*}#1\end{align*}}

\usepackage{undertilde}
\title{Pessimism in the Face of Confounders: Provably Efficient\\ Offline Reinforcement Learning in Partially Observable Markov Decision Processes}
\author{Miao Lu\thanks{Stanford University. Email: \texttt{miaolu@stanford.edu}} \qquad Yifei Min\thanks{Yale University. Email: \texttt{yifei.min@yale.edu}} 
\qquad Zhaoran Wang\thanks{Northwestern University. Email: \texttt{zhaoranwang@gmail.com}} \qquad Zhuoran Yang\thanks{Yale University. Email: \texttt{zhuoran.yang@yale.edu}}}
\date{\today}
\begin{document}
\maketitle

\begin{abstract}
We study offline reinforcement learning (RL) in partially observable Markov decision processes. In particular, we aim to learn an optimal policy from a dataset collected by a behavior policy which possibly depends on the latent state. Such a dataset is confounded in the sense that the latent state simultaneously affects the action and the observation, which is prohibitive for existing offline RL algorithms. To this end, we propose the \underline{P}roxy variable \underline{P}essimistic \underline{P}olicy \underline{O}ptimization (\texttt{P3O}) algorithm, which addresses the confounding bias and the distributional shift between the optimal and behavior policies in the context of general function approximation. 
At the core of \texttt{P3O} is a coupled sequence of pessimistic confidence regions  constructed via proximal causal inference, which is  formulated as  minimax estimation. 
Under a partial coverage assumption on the confounded dataset, we prove that \texttt{P3O} achieves a $n^{-1/2}$-suboptimality, where $n$ is the number of trajectories in the dataset. To our best knowledge, \texttt{P3O} is the first provably efficient offline RL algorithm for POMDPs with a confounded dataset.
\end{abstract}

\tableofcontents
\newpage
\section{Introduction}\label{sec: intro}

Offline reinforcement learning (RL) \citep{sutton2018reinforcement} aims to learn an   optimal policy of a sequential decision making problem  purely from an offline dataset collected a priori, without any  further interactions with the environment. 
Offline RL is particularly pertinent to applications in critical domains such as precision medicine \citep{gottesman2019guidelines} and autonomous driving \citep{shalev2016safe}.
In particular, in these scenarios, 
interacting with the environment via online experiments might be  risky, slow, 
or even possibly unethical, 
but oftentimes offline datasets consisting of past interactions, e.g.,   treatment records for precision medicine \citep{chakraborty2013statistical, chakraborty2014dynamic} and  human driving data for autonomous driving \citep{sun2020scalability},   are   adequately available.
As a result, offline RL has attracted substantial research interest recently  \citep{levine2020offline}.

Most of the existing works on offline RL develop algorithms and theory on   the model of Markov decision processes (MDPs). 
However, in many real-world applications, due to certain  privacy concerns or limitations of the sensor  apparatus, the states of the environment cannot be directly stored in the offline datasets. Instead, only partial  observations generated from the states  of the environments are stored \citep{dulac2021challenges}. 
For example, in precision  medicine, a physician's treatment might consciously or subconsciously depend on the patient's mood and socioeconomic status 
\citep{zhang2016markov}, which are not recorded in the data due to privacy concerns. 
As another example, 
in autonomous driving, a human driver  makes decisions based on multimodal information that is  not limited to visual and auditory inputs, but only observations captured by LIDARs and cameras are stored in the datasets \citep{sun2020scalability}. 
In light of the partial observations in the datasets, these situations are better modeled as partially observable Markov decision processes (POMDPs) \citep{lovejoy1991survey}. Existing offline RL methods for MDPs, which fail to handle partial observations, are thus not applicable.

In this work,  we make the initial  step towards studying offline RL in POMDPs where the datasets   only contain  partial observations of the states. 
In particular, motivated from the aforementioned real-world applications, we consider the case where the behavior policy takes actions based on the states of the environment, which are not part of  the dataset and  thus are  latent variables. 
Instead, the trajectories in datasets consist of 
partial observations emitted from the latent states, as well as the actions and rewards. 
For such a dataset, our  goal is to learn an  optimal policy  in the context of general function approximation.

Furthermore, offline RL in PODMP suffers from  several challenges. 
First of all, it is known that both planning and estimation in PODMPs are intractable in the worst case \citep{papadimitriou1987complexity, burago1996complexity, goldsmith1998complexity, mundhenk2000complexity, vlassis2012computational}. Thus, we have to identify a set of sufficient conditions that warrants offline RL. 
More importantly, our problem faces the unique challenge of the confounding issue  caused by the latent states, which does not appear in either online and offline MDPs or online POMDPs. 
In particular, both the actions and observations in the offline dataset depend on the unobserved latent states, and thus are confounded \citep{pearl2009causality}. 
Such a confounding issue is illustrated by a causal graph in Figure~\ref{fig: reactive simple}. 
As a result, directly applying offline RL methods for MDPs will nevertheless incur a considerable confounding bias. Besides, since the latent states evolve according to the Markov transition kernel, the causal structure is thus dynamic, which makes the confounding issue more challenging to handle than that in  static causal problems. 
Furthermore, apart from the confounding issue, since we aim to learn the optimal policy, our algorithm also needs to handle 
the distributional shift between the trajectories induced by the behavior policy  and the family of  target policies. 
Finally, to handle large   observation spaces, we need to employ  powerful function approximators. 
As a result, the coupled challenges due to  (i) the confounding bias, (ii) distributional shift, and (iii) large observation spaces that are distinctive in our problem necessitates new algorithm design and theory.

To this end, by leveraging tools from  proximal causal inference \citep{lipsitch2010negative, tchetgen2020introduction}, we propose the  \underline{P}roxy variable \underline{P}essimistic \underline{P}olicy \underline{O}ptimization (\texttt{P3O}) algorithm,
which provably addresses the confounding bias and the  distributional shift in the context of general function approximation. 
Specifically, we focus on a benign class of POMDPs where the causal structure involving latent states can be captured by 
the past and current observations, which serves as the negative control action and outcome  respectively \citep{ miao2018identifying,miao2018confounding,cui2020semiparametric,singh2020kernel,  kallus2021causal,  bennett2021proximal,shi2021minimax}. 
Then the value of each policy can be identified by a set of  confounding bridge functions corresponding to that policy, which satisfy a sequence of backward moment equations that are similar to the celebrated Bellman equations in classical RL \citep{bellman1965dynamic}. 
Thus, by estimating these confounding  bridge functions from offline data, we can estimate the value of each policy without incurring the confounding bias.

\begin{figure}[t]
    \centering
    \begin{tikzpicture}[->,>=stealth', very thick, main node/.style={circle,draw}]
    
       \node[main node, text=black, minimum width =30pt, 
minimum height =30pt] (-3) at  (-2,-0.875) {\small $R_{h-1}$};
    \node[main node, text=black, minimum width =30pt ,
minimum height =30pt] (-2) at  (2,-0.875) {\small $R_{h}$};

    \node[main node, text=black, minimum width =30pt, 
minimum height =30pt, style=dotted] (1) at  (-4,0) {\small$S_{h-1}$};
    \node[main node, text=black, minimum width =30pt ,
minimum height =30pt, style=dotted] (2) at  (0,0) {\small$S_{h}$};
\node[main node, text=black, minimum width =30pt ,
minimum height =30pt, style=dotted] (3) at  (4,0) {\small$S_{h+1}$};

\node[main node, text=black, minimum width =30pt ,
minimum height =30pt] (4) at  (-4,1.5) {\small$O_{h-1}$};
\node[main node, text=black, minimum width =30pt ,
minimum height =30pt] (5) at  (0,1.5) {\small$O_{h}$};
\node[main node, text=black, minimum width =30pt ,
minimum height =30pt] (6) at  (4,1.5) {\small$O_{h+1}$};

\node[main node, text=black, minimum width=30pt ,
minimum height =30pt] (7) at  (-2,1.5) {\small $A_{h-1}$};
\node[main node, text=black, minimum width =30pt ,
minimum height =30pt ] (8) at  (2,1.5) {\small $A_{h}$};
% \node[main node, text=black, minimum width =35pt ,
% minimum height =35pt] (9) at  (4,3.5) {$A_{h+1}$};

    \draw[->, style, draw = red1] (1) --(-3);
    \draw[->, style, draw = red1] (2) --(-2);

    \draw[->, style] (-5,0) --(1);
    \draw[->, style] (1) --(2);
    \draw[->, style] (2) --(3);
    
    \draw[->, style, draw = red1] (1) --(4);
    \draw[->, style, draw = red1] (2) --(5);
    \draw[->, style, draw = red1] (3) --(6);

    %\draw[->] (6) --(9);
    
    \draw[->] (7) --(-3);
    \draw[->] (8) --(-2);
    
    \draw[->, style] (7) --(2);
    \draw[->, style] (8) --(3);
    %\draw[->, style=dashed] (9) --(3);
    
     \draw[->, style, draw = blue1] (1) --(7);
    \draw[->, style, draw = blue1] (2) --(8);

    \end{tikzpicture}
    \caption{Causal graph of  the  data generating process for offline learning in POMDP. The dotted nodes indicate that these variables are not stored in the offline dataset. Here $S_{h}$ is the  state  of the environment at  step $h$. Besides,  $A_{h}$, $R_h$, $O_h$  are the action, immediate reward, and observation at the $h$-th  step, respectively. These variables are stored in the offline dataset and thus are represented by black solid circles. 
    We use  the solid arrows to indicate the dependency among  the variables. In specific, 
    the action $A_h$ is specified by the behavior policy which is a function of $S_h$, such dependency is depicted in the \textcolor{blue1}{\bf blue} arrows. 
    Moreover, both the observations and rewards depends on the state $S_h$ and such dependency is depicted in  \textcolor{red1}{\bf red}. 
 We would like to highlight that $S_h$ affects both $A_h$ and $O_h$ and thus serves as an \textbf{unobserved confounder}.}
    \label{fig: reactive simple}
\end{figure}
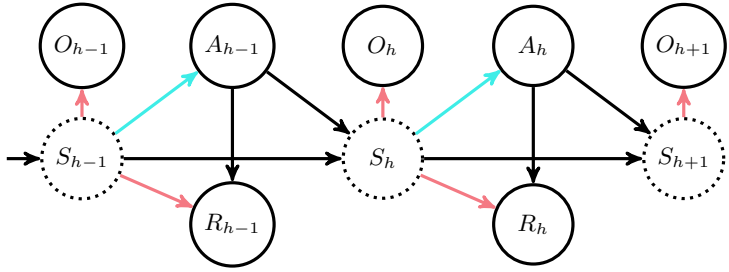

More concretely,  \texttt{P3O} involves two components  ---  policy evaluation via minimax estimation and policy optimization via pessimism. 
Specifically, to tackle the distributional shift, \texttt{P3O} returns the policy that maximizes pessimistic estimates of the values obtained by  policy evaluation. 
Meanwhile, in policy evaluation, to ensure pessimism, we construct a coupled sequence of confidence regions for  the confounding bridge functions 
via minimax estimation, using function approximators. 
Furthermore, under a partial coverage assumption on the confounded dataset, we prove that  \texttt{P3O}  achieves a $\tilde \cO( H    \sqrt{\log(\mathcal{N}_{\mathrm{fun}} )/ n} )$ suboptimality, where  $n$ is the number of trajectories, $H$ is the length of each trajectory, $\mathcal{N}_{\mathrm{fun}}$ stands for the complexity of the employed function classes (e.g., the covering number), and $\tilde \cO(\cdot)$ hides logarithmic factors. 
When specified  to linear function classes, the suboptimality of \texttt{P3O}  becomes $\tilde \cO ( \sqrt{H^3 d /n})$, where $d$ is the dimension of the feature mapping.  
To our best knowledge, we establish the first provably efficient offline RL algorithm for POMDP with a confounded dataset.

\subsection{Overview of Techniques}
To deal with the coupled  challenges of confounding bias, distributional shift, and large observational spaces,  our algorithm and analysis rely on the following technical ingredients. 
Although some aspects of these techniques  have been studied in existing works for other problems, 
to our best knowledge, these techniques are novel to  the literature of offline policy learning   in POMDPs.

\vspace{2mm}
\noindent
\textbf{Confidence regions based on minimax estimation via proximal causal inference.}
In order to handle the confounded offline dataset, we use the proxy variables from proximal causal inference \citep{miao2018identifying,miao2018confounding, cui2020semiparametric,kallus2021causal}, which allows us to identify the value of  each policy  by a set of confounding bridge functions. 
These bridge functions only depend on observed variables and satisfy  a set of backward conditional moment equations.
We then estimate these bridge functions via   minimax estimation \citep{dikkala2020minimax,chernozhukov2020adversarial, uehara2021finite}. 
More importantly, to handle the distributional shift,  we propose a sequence of novel confidence regions for the bridge functions, which quantifies the uncertainty of minimax estimation based on finite data. 
This sequence of new confidence regions has not been  considered in the previous works on off-policy evaluation (OPE) in POMDPs \citep{bennett2021proximal, shi2021minimax} as pessimism seems unnecessary in these works.  
Meanwhile, the confidence regions are constructed as a level set with respect to the loss functions of the minimax estimation. 
Such construction contrasts sharply with previous  works on offline RL  which build confidence regions via either least square regression or maximum likelihood estimation~\citep{xie2021bellman, uehara2021pessimistic, liu2022partially}.
Furthermore, we develop a novel theoretical analysis to show that any function in the confidence regions enjoys a  fast statistical rate of convergence  \citep{uehara2021finite}.
Finally, leveraging the backwardly  inductive nature of the bridge functions, our proposed confidence regions and analysis take the temporal structure   into consideration, which might be of independent interest to the research on dynamic causal inference \citep{friston2003dynamic}.

\vspace{2mm}
\noindent
\textbf{Pessimism principle for learning POMDPs.}
To learn the optimal policy in the face of distributional shift, we adopt the pessimism principle which is shown to be effective in offline RL in MDPs    \citep{PessFQI,jin2021pessimism,MACONG-2021new-Offline+IL,uehara2021pessimistic,xie2021bellman,yin2021towards,zanette2021provable,yin2022near,yan2022efficacy}.
Specifically, the newly proposed confidence regions, combined with the identification result based on proximal causal inference, allows us to construct a novel pessimistic estimator for the value of each target policy.
From a theoretical perspective, the identification result and the backward induction property of the bridge functions  provide  a way of decomposing the suboptimality of the learned policy in terms of   statistical  errors of the bridge functions.
When combined with the pessimism and the fast statistical rates  enjoyed by any functions in the   confidence regions, we show that our proposed \texttt{P3O} algorithm efficiently learns the optimal policy under a partial coverage assumption of the confounded dataset. 
We highlight that we firstly extend the pessimism principle to offline RL in POMDPs with confounded data.

\begin{table}[t]
  \centering
  \begin{tabular}{||c|c|c|c|c||}
  \hline
    &    Offline& Partial Observations & Confounded Data &  Policy Optimization    \\
    \hline
    \cite{xie2021bellman} &\ding{51}  &\ding{56} & \ding{56}   & \ding{51}     \\
    \hline
    \cite{uehara2021pessimistic} &\ding{51}  &\ding{56} & \ding{56}   & \ding{51}     \\
    \hline
    \hline
     \cite{jin2020sample}   &\ding{56}  &\ding{51} & \ding{56}   & \ding{51}     \\
    \hline
    \cite{efroni2022provable} &\ding{56}  &\ding{51} & \ding{56}   & \ding{51}     \\
    \hline
    \cite{liu2022partially}   &\ding{56}  &\ding{51} & \ding{56}   & \ding{51}     \\
    \hline
    \hline
     \cite{bennett2021proximal}   &\ding{51}  &\ding{51} & \ding{51}   & \ding{56}     \\
    \hline
    \cite{shi2021minimax}   &\ding{51}  &\ding{51} & \ding{51}   & \ding{56}     \\
    \hline
    \hline
    \rowcolor{green2}
      Our algorithm: \texttt{P3O}  &\ding{51}     & \ding{51}  &  \ding{51}   & \ding{51}  \\
    \hline
  \end{tabular}
\caption{We compare with most related representative works in closely related lines of research. The first line of research  studies offline RL in standard MDPs without any partial observability. The second line of research  studies online RL in POMDPs where the actions are specified by history-dependent policies.  
Thus, the actions does not directly depends on the latent states and thus these works  do  not involve the challenge due to  confounded data. The third line of research studies OPE in POMDPs where the goal is to learn the value of the target policy as opposed to learning the optimal policy. As a result, these works do not  to need to handle the challenge of  distributional shift via pessimism.}
  \label{table: related work}
\end{table}

\subsection{Related Works}
Our work is closely related to the bodies of literature on (i) reinforcement learning POMDPs, (ii) offline reinforcement learning (in MDPs), and (iii) OPE via causal inference. 
For a comparison, we summarize and contrast with most related existing works in Table \ref{table: related work}. 
Compared to the literature, our work simultaneously involve partial observability, confounded data, and offline policy optimization simultaneously, and thus involves the challenges faced by (i)--(iii).
In the sequel, we discuss the related works in detail.

\vspace{2mm}
\noindent
\textbf{Reinforcement learning in POMDPs.} Our work is related to the recent line of research on developing provably efficient online RL methods for POMDPs  \citep{guo2016pac,krishnamurthy2016pac,jin2020sample,xiong2021sublinear,jafarnia2021online,efroni2022provable,liu2022partially}. 
In the online setting, the actions are specified by history-dependent policies and thus the latent state does not directly affect the actions. 
Thus, the actions and observations in the online setting are not confounded by latent states. 
Consequently, although these work also  conduct uncertainty quantification  to encourage exploration, 
the  confidence regions are not  based on confounded data and are thus constructed differently.

\vspace{2mm}
\noindent
\textbf{Offline reinforcement learning and pessimism.}
Our work is also  related  to the literature on offline RL and  particularly related to the works based on the pessimism principle \citep{uniform3,uniform4,chen2019information,buckman2020importance,PessFQI,jin2021pessimism,zanette2021exponential,jin2021pessimism,xie2021bellman,uehara2021pessimistic,yin2021towards,MACONG-2021new-Offline+IL,zhan2022offline,yin2022near,yan2022efficacy}. 
Offline RL faces the challenge  of the distributional shift between the behavior policy and the family of  target policies. 
Without any coverage assumption on the offline data,   the number of data   needed  to find a  near-optimal policy  can be exponentially large \citep{buckman2020importance,zanette2021exponential}. 
To circumvent this problem, a few existing  works study offline RL under a uniform  coverage assumption, which requires the concentrability coefficients  between the behavior  and target policies are  uniformly bounded. See, e.g., \cite{uniform3,uniform4,chen2019information} and the references therein. 
Furthermore, a more recent line of work aims to weaken the uniform coverage assumption by adopting the pessimism principle in algorithm design  \citep{PessFQI,jin2021pessimism,MACONG-2021new-Offline+IL,uehara2021pessimistic,xie2021bellman,yin2021towards,zanette2021provable,yin2022near,yan2022efficacy}.
In particular, these works proves theoretically that pessimism is effective in tackling the distributional shift of the  offline dataset. 
In particular, by constructing pessimistic value function  estimates, 
these works establish  upper bounds on the suboptimality of the proposed methods based on significantly weaker partial coverage assumption. 
That is, these methods can find a near-optimal policy as long as the dataset covers the optimal policy. 
The efficacy of  pessimism 
has  also been validated empirically in 
\cite{kumar2020conservative,kidambi2020morel,yu2021combo,janner2021offline}. 
Compared with these works on pessimism, we focus on the more challenging setting of POMDP with a  confounded dataset. To perform pessimism in the face of confounders, we conduct uncertainty quantification for the minimax estimation regarding the confounding bridge functions. 
Our work complements this line of research by successfully applying  pessimism to confounded data.

%Alternatively, some recent works have shown that by adopting the pessimism principle of  in the model estimation or value estimation, it is possible to efficiently learn the optimal policy under partial coverage assumptions~
%The partial coverage only requires the boundedness of the concentrability coefficient between the behavior policy and the optimal policy, which is much weaker than the full coverage.
%However, it remains unknown how to extend the principle of pessimism-in-face-of-uncertainty to offline policy optimization in POMDPs with confounded data.

\vspace{2mm}
\noindent
\textbf{OPE via causal inference.}
Our work is closely related to the line of research that employing tools from  causal inference  \citep{pearl2009causality} for studying OPE with unobserved confounders \citep{oberst2019counterfactual,kallus2020confounding,bennett2021off,kallus2021minimax,mastouri2021proximal,shi2021minimax,bennett2021proximal,shi2022dynamic}.
Among them,  \cite{bennett2021proximal,shi2021minimax} are most relevant to our work.
In particular, these works also leverage 
   proximal causal inference \citep{lipsitch2010negative, miao2018identifying,miao2018confounding,cui2020semiparametric,tchetgen2020introduction,singh2020kernel}
to identify the value of the target policy in POMDPs. 
See \cite{tchetgen2020introduction} for a detailed survey of proximal causal inference.
In comparison, this line of research only focuses on evaluating a single policy, whereas we focus on learning the optimal policy within a class of target policies.  
As a result, we need to handle a more challenging distributional shift problem between the behavior policy and an entire class of target policies, as opposed to a single target policy in OPE. 
However, thanks to the pessimism, we establish theory based on a partial coverage assumption that is similar to that in the OPE literature. 
To achieve such a goal, we conduct uncertainty quantification for the bridge function estimators, which is absent in the the works on OPE. As a result, our analysis is different from that in \cite{bennett2021proximal,shi2021minimax}.

\section{Preliminaries}\label{sec: prelim}

\textbf{Notations.} 
In the sequel, we use lower case letters (i.e., $s$, $a$, $o$, and $\tau$) to represent dummy variables and upper case letters (i.e., $S$, $A$, $O$, and $\Gamma$) to represent random variables. We use the variables in the calligraphic font (i.e., $\cS$, $\cA$, and $\cO$) to represent the spaces of variables, and the blackboard bold font (i.e., $\mathbb{P}$ and $\mathbb{O}$) to represent probability kernels.  

\subsection{Episodic Partially Observable Markov Decision Process}\label{subsec: POMDP}
We consider an episodic, finite-horizon POMDP, specified by a tuple $(\mathcal{S}, \mathcal{O}, \mathcal{A}, H, \mu_1, \mathbb{P},\mathbb{O}, R)$.
Here we let  $\mathcal{S}$, $\cA$, and $\cO$   denote the  state, action, and  observation spaces, respectively. 
%and let $\cA$ denote the set $\mathcal{O}$ denotes an observation space, and the set $\mathcal{A}$ denotes a finite action space. 
The integer $H\in\mathbb{N}$ denotes the length of each episode.
The distribution $\mu_1\in\Delta(\mathcal{S})$ denotes the distribution of the initial state. The set $\mathbb{P}=\{\mathbb{P}_h\}_{h\in[H]}$ denotes the collection of state transition kernels where each kernel $\mathbb{P}_h(\cdot|s,a):\mathcal{S}\times\mathcal{A}\mapsto\Delta(\mathcal{S})$ characterizes the distribution of the next state $s_{h+1}$ given that the agent takes action  $a_h = a\in\mathcal{A}$ at state $s_h = s\in\mathcal{S}$ and  step $h\in[H]$.
The set $\mathbb{O}=\{\mathbb{O}_h\}_{h=1}^H$ denotes the observation  emission kernels where each kernel $\mathbb{O}_h(\cdot|s):\mathcal{S}\mapsto\Delta(\mathcal{O})$ characterizes the distribution over observations given the current state $s\in\mathcal{S}$ at step $h\in[H]$.
Finally, the set $R=\{R_h\}_{h=1}^H$ denotes the collection of reward functions where each function $R_h(s,a):\mathcal{S}\times\mathcal{A}\mapsto[0,1]$ specifies the reward the agent receives when taking action $a\in\mathcal{A}$ at state $s\in\mathcal{S}$ and step $h\in[H]$.

Different from an MDP, in a POMDP, only the observation $o$, the action $a$, and the reward $r$ are observable, while the state variable $s$ is unobservable. 
In each episode, the environment first samples an initial state $S_1$ from $\mu_1(\cdot)$.
At each step $h\in[H]$, the environment emits an observation $O_h$ from $\mathbb{O}_h(\cdot|S_h)$.
If an action $A_h$ is taken, then the environment samples the next state $S_{h+1}$ from $\mathbb{P}_h(\cdot|S_h,A_h)$ and assign a reward $R_h$ given by $R_h(S_h,A_h)$.
In our setting, we also let $O_0\in\mathcal{O}$ denote the prior observation before step $h=1$. 
We assume that $O_0$ is independent of other random variables in this episode given the first state $S_1$.

\subsection{Offline Data Generation: Confounded Dataset}\label{subsec: data generation}
Now we describe the data generation process. 
Motivated by real-world examples such as precision medicine and autonomous driving discussed in Section~\ref{sec: intro}, we assume that the offline data is generated by some behavior policy $\pi^b$ which has access to the latent states.
Specifically, we let  $\pi^b=\{\pi_h^b\}_{h=1}^H$ denote a collection of policies such that $\pi_h^b(\cdot |s):\mathcal{S}\mapsto\Delta(\mathcal{A})$ specifies  the probability of taking action each $a\in\cA $ at state $s$ and step $h$. 
This behavior policy induces a set of probability distribution $\mathcal{P}^b=\{\mathcal{P}_h^b\}_{h=1}^H$ 
on the trajectories of the POMDP, where $\mathcal{P}_h^b$ is the distribution of the variables at step $h$ when following the policy $\pi^b$. 
Formally, we assume that the offline data is denoted by $\mathbb{D}=\{(o_0^k,(o_1^k,a_1^k,r_1^k),\cdots,(o_H^k,a_H^k,r_H^k))\}_{k=1}^n$, where $n$ is the number of trajectories, and  for each $k \in [n ] $, $(o_0^k,(o_1^k,a_1^k,r_1^k),\cdots,(o_H^k,a_H^k,r_H^k))$ is independently sampled from $\mathcal{P}^b$.
We highlight that such an offline dataset is confounded since the latent state $S_h$, which is not stored  in the dataset, simultaneously affects the control variables (i.e., action $A_h$) and the outcome variables (i.e., observation $O_h$ and reward $R_h$)~\citep{pearl2009causality}.
Such a   setting is prohibitive for existing  offline RL algorithms for MDPs  as directly applying them will nevertheless incur a confounding bias that is not negligible.

\subsection{Learning Objective}\label{subsec: goal}
The goal of offline RL is to learn an optimal policy from the offline dataset  which maximizes the expected cumulative reward. 
For POMDPs, the learned policy can only depend on the observable information mentioned in Section \ref{subsec: POMDP}. 
% To formally define the policies to be learned, we first define the observable history information used by the agent as $\mathcal{H}=\{\mathcal{H}_{h}\}_{h=0}^{H-1}$, where $\mathcal{H}_h$ is a subset of $\{(O_1,A_1),\cdots,(O_h,A_h)\}$.
To formally define the set policies of interest, we first define the space of observable history as $\mathcal{H}=\{\mathcal{H}_{h}\}_{h=0}^{H-1}$, where each element $\tau_h \in \mathcal{H}_h$ is a (partial) trajectory such that $\tau_h \subseteq  \{(o_1,a_1),\cdots,(o_h,a_h)\}$.

We denote by $\Pi(\cH)$ the class of policies that make decisions based on the current observation $o_h\in\mathcal{O}$ and the history information $\tau_{h-1}\in\cH_{h-1}$. 
That means, a policy $\pi=\{\pi_h\}_{h=1}^H\in\Pi(\cH)$ is a collection of policies where $\pi_h(\cdot |o,\tau):\mathcal{O}\times\mathcal{H}_{h-1}\mapsto\Delta(\mathcal{A})$ denotes the probability of taking each  action $a\in \cA $ given observation $o\in\mathcal{O}$ and history $\tau\in\cH_{h-1}$. 
The choice of $\cH$ induces the policy set $\Pi (\cH)$ by specifying the input of the policies. 
We now  introduce three  examples of $\cH$ and the corresponding $\Pi(\cH)$.

\begin{example}[Reactive policy]\label{example: reactive} The policy only depends on the current observation $O_h$. Formally, we have $\mathcal{H}_{h-1}=\{\emptyset\}$ and therefore $\tau_{h-1}=\emptyset$ for each $h\in[H]$.
\end{example}

\begin{example}[Finite-history 
policy]\label{example: finite history}
The policy depends on the current observation and the history of length at most $k$. Formally, we have $\mathcal{H}_{h-1}=(\mathcal{O}\times\mathcal{A})^{\otimes\min\{k, h-1\}}$ and 
$\tau_{h-1}=((o_l,a_l),\cdots,(o_{h-1},a_{h-1}))$ for some $k\in\mathbb{N}$, where the index $l = \max\{1, h-k\}$.
\end{example}

\begin{example}[Full-history policy]\label{example: full history}
The policy depends on the current observation and the full history.
Formally, we have
$\mathcal{H}_{h-1}=(\mathcal{O}\times\mathcal{A})^{\otimes(h-1)}$,
and $\tau_{h-1}=((o_1,a_1),\cdots,(o_{h-1},a_{h-1}))$.
\end{example}

We illustrate these examples with causal graphs and a more detailed discussion in Appendix~\ref{subsec: illustration plot}. 
% In the sequel, we use lower case letters $s$, $a$, $o$, and  $\tau$ to denote dummy variables and upper case letters $S$, $A$, $O$, and $\Gamma$ to denote random variables. 
Now given a policy $\pi\in\Pi(\cH)$, we denote by $J(\pi)$ the value of $\pi$ that characterizes the expected cumulative
rewards the agent receives by following $\pi$.
Formally, $J(\pi)$ is defined as 
\begin{align}\label{eq: value}
    J(\pi)\coloneqq\mathbb{E}_{\pi}\left[\sum_{h=1}^H\gamma^{h-1}R_h(S_h,A_h)\right],
\end{align}
where $\gamma\in(0,1]$ denotes the discount factor, $\mathbb{E}_{\pi}$ denotes the expectation with respect to $\mathcal{P}^\pi=\{\mathcal{P}_h^{\pi}\}_{h=1}^H$ which is the distribution of the trajectories induced by $\pi$.
% Our goal is to find some policy $\hat{\pi}\in\Pi(\cH)$ that minimizes the suboptimality gap between the value of $\hat \pi$ and the optimal policy $\pi^{\star}$ in $\Pi(\cH)$, that is,
We now define the suboptimality gap of any policy $\hat \pi$ as 
%between the values of a policy $\hat \pi$ and the optimal policy $\pi^{\star}$ in $\Pi(\cH)$ as
\begin{align}\label{eq: suboptimality}
    \textnormal{SubOpt}(\hat{\pi})\coloneqq J(\pi^{\star})-J(\hat \pi),\quad \textnormal{where}\quad \pi^{\star}\in\argmax_{\pi\in\Pi(\cH)}J(\pi).
\end{align}
Here $\pi^\star $ is the optimal policy within $\Pi (\cH)$.
Our goal is to find some policy $\hat{\pi}\in\Pi(\cH)$ that minimizes the  suboptimality gap in \eqref{eq: suboptimality} based on the offline dataset $\mathbb{D}$.

\section{Algorithm}\label{sec: algorithms}

The offline RL problem introduced in   Section~\ref{sec: prelim} for POMDPs 
suffers from three coupled challenges ---  (i) the confounding bias, (ii) distributional shift, and (iii) large observation spaces. 
Among these, the confounding bias is caused by the fact that the latent state $S_h$ simultaneous affects the action variable (i.e., $A_h$) the outcome variables (i.e., $O_h$ and $R_h$). 
The distributional shift exists between the trajectories induced by the behavior policy $\pi^b$ and the family of  target policies $\Pi(\cH)$. The challenge of large observation spaces necessitates the adoption of function approximators.  In the sequel,  we  introduce an algorithm that addresses all three challenges simultaneously. 

Offline RL for POMDPs is known to be intractable in the worst case~\citep{krishnamurthy2016pac}. So we first identify a benign class of POMDPs where the causal structure involving latent states can be captured by only the observable variables available in the dataset $\mathbb{D}$. 
For such a  class of POMDPs, by leveraging tools from proximal causal inference~\citep{lipsitch2010negative, tchetgen2020introduction}, we then seek to identify the value $J(\pi)$ of the policy $\pi \in \Pi(\mathcal{H})$ via  some confounding bridge functions $\bbb^\pi$ (Assumption~\ref{assump: bridge functions exist}) which only depend on the observable variables and thus can be estimated using  $\mathbb{D}$~(Theorem~\ref{thm: identification}). 
Identification via proximal causal inference will be  discussed in Section~\ref{subsec: identification}.

In addition, to  estimate these confounding bridge functions, 
we utilize the    the fact that these functions satisfy a sequence of conditional moment equations which resembles the Bellman equations in classical MDPs~\citep{bellman1965dynamic}. 
Then we adopt the idea of  minimax estimation \citep{dikkala2020minimax,kallus2021causal,uehara2021finite, duan2021risk} which  formulates the bridge functions as the solution to  a series of minimax optimization problems in  \eqref{eq: hat b}. 
Additionally, the loss function in   minimax estimation readily incorporates  function approximators and thus  addresses the challenge of large observation spaces.

To further handle the distributional shift, we extend the pessimism principle~\citep{PessFQI,jin2021pessimism,MACONG-2021new-Offline+IL,uehara2021pessimistic,xie2021bellman,yin2021towards,zanette2021provable} to POMDPs with the help of the confounding bridge functions.
In specific, based on the confounded dataset,  we first     constructing a novel confidence region $\CR^\pi(\xi)$ for $\bbb^\pi$ based on   level sets  with respect to the loss functions of the minimax estimation (See    \eqref{eq: define CR} for details). 
Our algorithm, \underline{P}roxy variable \underline{P}essimistic \underline{P}olicy \underline{O}ptimization (\texttt{P3O}), outputs the policy that maximizes   pessimistic estimates the values of the policies within $\Pi(\cH)$. 
Concretely,  \texttt{P3O}  outputs a policy $\hat\pi\in\Pi(\cH)$ via 
\begin{align*}
    \hat{\pi}=\argmax_{\pi\in\Pi(\cH)}\min_{\bbb\in\CR^{\pi}(\xi)}\hat F(\bbb).
\end{align*} 
where $\hat F(\cdot )$ is a functional that plays a similar  role as $J(\pi)$ and its details will be presented in  Section \ref{subsec: minimax estimator}.
The details of \texttt{P3O} is summarized by Algorithm~\ref{alg: main} in Section \ref{subsec: policy optimization}.

\subsection{Policy Value Identification via Proximal Causal Inference}\label{subsec: identification}
To handle the confounded dataset $\mathbb{D}$, we first identify the policy value $J(\pi)$ for each $\pi\in\Pi(\cH)$ using the idea of proxy variables.
Following the notions of proximal causal inference \citep{miao2018confounding,miao2018identifying,cui2020semiparametric}, we assume that there exists negative control actions $\{Z_h\}_{h=1}^H$ and negative control outcomes $\{W_h\}_{h=1}^H$ satisfying the following independence assumption.

\begin{assumption}[Negative control]\label{assump: negative control cond independence}
    We assume there exist negative control variables $\{W_h\}_{h=1}^H$ and $\{Z_h\}_{h=1}^H$ measurable w.r.t. the observed trajectories, such that under $\cP^{b}$, it holds that
    \begin{align*}
        Z_h \perp O_h, R_h, W_h,W_{h+1} \mid A_h, S_h, \Gamma_{h-1}, \quad W_h \perp A_h, \Gamma_{h-1},S_{h-1} \mid S_h.
    \end{align*}
\end{assumption}

We briefly explain the existence of such negative control variables for the three different choices of $\cH$ in Example \ref{example: reactive}, \ref{example: finite history}, and \ref{example: full history}, respectively.

\begin{example}[Example \ref{example: reactive} revisited]\label{example: reactive negative} For reactive policies, we choose the negative control action as $Z_h=O_{h-1}$ and the negative control outcome as $W_h=O_h$.
\end{example}

\begin{example}[Example \ref{example: finite history} revisited]\label{example: finite history negative}
For policies depending on a finite history of length $k$, we choose the  negative control action as
$Z_h=O_{l-1}$ and the negative control outcome as $W_h=O_h$, where the index $l = \max\{1, h-k\}$.
\end{example}

\begin{example}[Example \ref{example: full history} revisited]\label{example: full history negative}
For policies depending on the  full history,
we choose  the negative control action as $Z_h=O_0$ and the  negative control outcome as $W_h=O_h$.
\end{example}

% For finite-length history policy, one can see from Figure \ref{fig: finite history main} that the choice of the negative control variables satisfy Assumption \ref{assump: negative control cond independence}.
These examples are discussed in detail and illustrated with causal graphs in Appendix \ref{subsec: illustration plot}, where we also show that all three examples provably satisfy Assumption~\ref{assump: negative control cond independence}. 
We note that by adopting the notion of negative control variables, we are featuring more generalities. 
Exploring other possible choices of negative control variables is an interesting future direction.

Besides Assumption \ref{assump: negative control cond independence}, 
our identification of policy value also relies on the notion of confounding bridge functions \citep{kallus2021causal, shi2021minimax}, for which we make the following assumption.

\begin{assumption}[Confounding bridge functions]\label{assump: bridge functions exist}
    For any history-dependent policy $\pi\in\Pi(\cH)$, we assume the existence of the value bridge functions $\{b^{\pi}_h:\mathcal{A}\times\mathcal{W}\mapsto\mathbb{R}\}_{h=1}^H$ and the weight bridge functions $\{q^{\pi}_h:\mathcal{A}\times\mathcal{Z}\mapsto\mathbb{R}\}_{h=1}^H$ which are defined as the solution to the following equations: $\cP^{b}$-almost surely, 
\begin{align}\label{eq: def value bridge}
    \EE_{\pi^b}\left[b^{\pi}_{h}(A_h, W_h)\middle| A_h, Z_h\right] = \EE_{\pi^b}\Big[ R_h \pi_h(A_h|O_h, \Gamma_{h-1}) + \gamma \sum_{a'}b^{\pi}_{h+1}(a', W_{h+1}) \pi_h(A_h|O_h, \Gamma_{h-1}  )\Big| A_h, Z_h \Big],
\end{align}
\begin{align}\label{eq: def weight bridge}
    \EE_{\pi^b}\left[ q^{\pi}_h(A_h, Z_h) \middle| A_h, S_h, \Gamma_{h-1}  \right] = \frac{\mu_h(S_h, \Gamma_{h-1})}{\pi^b_h(A_h|S_h)},
\end{align}
where $b_{H+1}^{\pi}$ is a zero function and $\mu_h(S_h, \Gamma_{h-1})$ in \eqref{eq: def weight bridge} is defined as
\begin{align}\label{eq: def mu_h}
    \mu_h(S_h, \Gamma_{h-1}) \coloneqq \frac{\cP^{\pi}_h (S_h, \Gamma_{h-1}) }{\cP^{b}_h (S_h, \Gamma_{h-1})} . 
\end{align}
\end{assumption}

% The significance of bridge function: can be estimated from data.

We use ``confounding bridge function'' and ``bridge function'' interchangeably throughout the paper. 
We remark that in the proximal causal inference literature, the existence of such bridge functions bears more generality than assuming certain complex completeness conditions \citep{cui2020semiparametric}, 
as discussed by \cite{kallus2021causal}.
The existence of such bridge functions is justified, e.g., by conditions on the the rank of certain conditional probabilities or singular values of certain conditional expectation linear operators.
We present the following examples to explain the existence in the tabular case with reactive policies.
We refer the readers to \cite{kallus2021causal} for more detailed discussions on the conditions for the existence of bridge functions.

\begin{example}[Example \ref{example: reactive} revisited]\label{exampe: bridge exist}
For the tabular setting (i.e., $\mathcal{S}$, $\mathcal{A}$, and $\mathcal{O}$ are finite spaces) and reactive policies (i.e., $\pi_h:\mathcal{O}\mapsto\Delta(\mathcal{A})$), the sufficient condition under which Assumption \ref{assump: bridge functions exist} holds is that 
\begin{align}\label{eq: example rank condition}
    \mathbf{rank}(\mathcal{P}_h^{b}(O_h|S_h))=|\mathcal{S}|,\quad \mathbf{rank}(\mathcal{P}_h^{b}(O_{h-1}|S_h))=|\mathcal{S}|,
\end{align}
where $\mathcal{P}_h(O_h|S_h)$ denote an $|\mathcal{S}|\times|\mathcal{O}|$ matrix whose $(s,o)$-th element is $\mathcal{P}_h^b(O_h=o|S_h=s)$, and $\mathcal{P}_h^b(O_{h-1}|S_h)$ is defined similarly. See Appendix \ref{subsec: proximal causal inference proof} for a deduction from condition \eqref{eq: example rank condition} to Assumption \ref{assump: bridge functions exist}.
\end{example}

\begin{example}[Example \ref{example: finite history} revisited]\label{exampe: bridge exist 2}
    For the tabular setting (i.e., $\mathcal{S}$, $\mathcal{A}$, and $\mathcal{O}$ are finite spaces) and finite length policies (i.e., $\pi_h:\mathcal{O}\times(\mathcal{O}\times\mathcal{A})^{\min\{k, h-1\}}\mapsto\Delta(\mathcal{A}))$, the sufficient condition for which Assumption~\ref{assump: bridge functions exist} holds is that, for any action $a\in\mathcal{A}$, 
    \begin{align}\label{eq: example rank condition 2}
        \mathbf{rank}(\mathcal{P}_h^b(O_h|A_h = a, O_{h-k-1})) = |\mathcal{O}|,\quad \mathbf{rank}(\mathcal{P}_h^b(O_{h-k-1}|A_h = a, S_h, \Gamma_{h-1})) = |\mathcal{O}|,
    \end{align}
    where $\mathcal{P}_h^b(O_h|A_h = a, O_{h-k-1})$ is a $|\mathcal{O}|\times|\mathcal{O}|$ matrix with $(o,o')$-th element is $\mathcal{P}_h^b(O_h = o|A_h = a, O_{h-k-1} = o')$ and $\mathcal{P}_h^b(O_{h-k-1}|A_h = a, S_h, \Gamma_{h-1})$ is a $|\mathcal{S}||\mathcal{H}_{h-1}|\times|\mathcal{O}|$ matrix defined similarly.
    Refer to Appendix \ref{subsec: proximal causal inference proof} for a deduction from condition \eqref{eq: example rank condition 2} to Assumption \ref{assump: bridge functions exist}.
\end{example}

Now given Assumption \ref{assump: negative control cond independence} and Assumption \ref{assump: bridge functions exist} on the existence of proxy variables and bridge functions, we are ready to present the main identification result. 
It represents the true policy value $J(\pi)$ via the value bridge functions \eqref{eq: def value bridge},
as is concluded in the following theorem.

\begin{theorem}[Identification of policy value]\label{thm: identification}
Under Assumption \ref{assump: negative control cond independence} and \ref{assump: bridge functions exist}, for any history-dependent policy $\pi\in\Pi(\cH)$, it holds that
\begin{align}\label{eq: identifiaction}
    J(\pi)=\mathbb{E}_{\pi^b}\left[\sum_{a\in\mathcal{A}}b^{\pi}_{1}(a,W_1)\right].
\end{align}
\end{theorem}
\begin{proof}[Proof of Theorem \ref{thm: identification}]
See Appendix \ref{sec: proof identification} for a detailed proof.
\end{proof}

Note that though we have assumed the existence of both the value bridge functions \eqref{eq: def value bridge} and the weight bridge functions \eqref{eq: def weight bridge}, Theorem \ref{thm: identification} represents $J(\pi)$ using only the value bridge functions. 
In equation \eqref{eq: def value bridge} all the random variables  involved are observed by the learner and distributed according to  the data distribution $\mathcal{P}^b$, by which the value bridge functions can be estimated from data. 
This overcomes the confounding issue.

\subsection{Policy Evaluation via Minimax Estimation with Uncertainty Quantification}
\label{subsec: minimax estimator}

According to Theorem~\ref{thm: identification} and Assumption~\ref{assump: bridge functions exist},
to estimate the value $J(\pi)$ of $\pi\in\Pi(\cH)$, it suffices to estimate the value bridge functions $\{b_{h}^{\pi}\}_{h=1}^H$ by solving \eqref{eq: def value bridge}, which is a conditional moment equation. To this end, we adopt the method of minimax estimation  \citep{dikkala2020minimax,uehara2021finite, duan2021risk}. 
Furthermore, in order to handle the distributional shift between behavior policy and target policies, we construct a sequence of confidence regions for $\{b_{h}^{\pi}\}_{h=1}^H$ based on minimax estimation, which allows us to apply the pessimism principle by finding the most pessimistic estimates within the confidence regions.

Specifically,   minimax estimation involves two function classes $\mathbb{B}\subseteq\{b:\mathcal{A}\times\mathcal{W}\mapsto\mathbb{R}\}$ and $\mathbb{G}\subseteq\{g:\mathcal{A}\times\mathcal{Z}\mapsto\mathbb{R}\}$, interpreted as the primal and dual function classes, respectively.
Theoretical assumptions on $\mathbb{B}$ and $\mathbb{G}$ are presented in Section~\ref{sec: theoretical results}.
In order to find functions that satisfy   \eqref{eq: def value bridge}, it suffices to find $\mathbf{b}=(b_1,\cdots,b_H)\in\mathbb{B}^{\otimes H}$ such that the following conditional moment
\begin{align}\label{eq: bellman}
    \ell^{\pi}_h(b_h,b_{h+1})(A_h,Z_h)&\coloneqq \mathbb{E}_{\pi^{b}}\Big[b_{h}(A_h,W_h)-R_h\pi_h(A_h|O_h,\Gamma_{h-1}) \notag\\
    & \qquad -\gamma\sum_{a^{\prime}\in\cA}b_{h+1}(a^{\prime},W_{h+1})\pi_h(A_h|O_h,\Gamma_{h-1})\Big|A_h,Z_h\Big]
\end{align}
is equal to zero almost surely for all $h\in[H]$, where $b_{H+1}$ is a zero function.
Intuitively, the quantity \eqref{eq: bellman} can be interpreted as the ``Bellman residual" of the value bridge functions $\bbb$.
Notice that \eqref{eq: bellman} being zero   almost surely for all $h\in[H]$ is actually a conditional moment equation, which is
equivalent to finding   $\bbb\in\mathbb{B}^{\otimes H}$ such that the following residual mean squared error (RMSE) is minimized  for all $h$:
\begin{align}\label{eq: RMSE loss} 
    \mathcal{L}_h^{\pi}(b_h,b_{h+1})\coloneqq\mathbb{E} _{\pi^b}\bigl [ \bigl( \ell_h^\pi(b_h,b_{h+1})(A_h,Z_h) \bigr )^2  \bigr ]    .
\end{align}

It might seem tempting to directly minimize the empirical version of \eqref{eq: RMSE loss}. 
However, this is not viable as one would obtain a biased estimator due to an additional variance term. 
The reason is that the quantity defined by \eqref{eq: bellman} is a conditional expectation and therefore RMSE defined by \eqref{eq: RMSE loss} cannot be directly unbiasedly estimated from data  \citep{farahmand2016regularized}.
% Using the method of minimax estimation allows us to obtain a fast statistical rate (i.e., $\mathcal{O}(n^{-1/2})$).
In the sequel, we adopt the technique of  minimax estimation to circumvent this issue.
In particular, we first use Fenchel duality to write \eqref{eq: RMSE loss} as 
\begin{align}\label{eq: RMSE 2}
    \mathcal{L}_h^{\pi}(b_h,b_{h+1})=4\lambda \mathbb{E}_{\pi^b}\left[\max_{g\in\mathbb{G}}\ell_h^{\pi}(b_h,b_{h+1})(A_h,Z_h)\cdot g(A_h,Z_h)-\lambda g(A_h,Z_h)^2\right],\ \lambda>0,
\end{align}
which holds when the dual function class $\mathbb{G}$ is expressive enough such that $\ell_h^{\pi}(b_h,b_{h+1})/2\lambda\in\mathbb{G}$.
Then thanks to the interchangeability principle \citep{rockafellar2009variational,dai2017learning,shapiro2021lectures}, we can interchange the order of maximization and expectation and derive that 
\begin{align}\label{eq: RMSE 3}
    \mathcal{L}_h^{\pi}(b_h,b_{h+1})=4\lambda \max_{g\in\mathbb{G}}\mathbb{E}_{\pi^b}\left[\ell_h^{\pi}(b_h,b_{h+1})(A_h,Z_h)\cdot g(A_h,Z_h)-\lambda g(A_h,Z_h)^2\right].
\end{align}
The core idea of minimax estimation is to minimize the empirical version of \eqref{eq: RMSE 3} instead of \eqref{eq: RMSE loss}, and
the benefit of doing so is a fast statistical rate of $\tilde{\mathcal{O}}(n^{-1/2})$ \citep{dikkala2020minimax, uehara2021finite}, as we can see in the sequel.
For simplicity, in the following, we define $\Phi^{\lambda}_{\pi,h}:\mathbb{B}\times\mathbb{B}\times\mathbb{G}\mapsto\mathbb{R}$ with parameter $\lambda>0$ as
\begin{align}\label{eq: population phi lambda}
    \Phi_{\pi,h}^{\lambda}(b_h,b_{h+1};g)\coloneqq\mathbb{E}_{\pi^b}\big[\ell_h^{\pi}(b_h,b_{h+1})(A_h,Z_h)\cdot g(A_h,Z_h)-\lambda g(A_h,Z_h)^2\big],
\end{align}
and we denote by $\hat{\Phi}_{\pi,h}^{\lambda}:\mathbb{B}\times\mathbb{B}\times\mathbb{G}\mapsto\mathbb{R}$ the empirical version of $\Phi_{\pi,h}^{\lambda}$, i.e.,
\begin{align}\label{eq: empirical phi lambda}
    \hat{\Phi}^{\lambda}_{\pi,h}(b_{h},&b_{h+1};g)\coloneqq\hat{\mathbb{E}}_{\pi^b}\Big[\Big(b_{h}(A_h,W_h)- R_h \pi_h (A_h|O_h,\Gamma_{h-1}) \notag\\
    &\qquad -\gamma\sum_{a^{\prime}\in\cA}b_{h+1}(a^{\prime},W_{h+1})\pi_h(A_h|O_h,\Gamma_{h-1})\Big)\cdot g(A_h,Z_h)-\lambda g(A_h,Z_h)^2\Big],
\end{align}
where $\hat{\mathbb{E}}_{\pi^b}$ denotes the empirical version of $\EE_{\pi^b}$ based on dataset $\mathbb{D}$ described in Section \ref{subsec: data generation}.

Furthermore, note  that the  value bridge functions $(b_1^{\pi},\cdots,b_h^{\pi})$ admit a sequential dependence structure.
To handle such dependency, for  any $\pi\in\Pi(\cH)$, $h\in[H]$,  and $b_{h+1}\in\mathbb{B}$, we first define the minimax estimator $\hat{b}_h(b_{h+1})$ as
\begin{align}\label{eq: hat b}
    \hat{b}_h(b_{h+1})\coloneqq\argmin_{b\in\mathbb{B}} \Bigl\{ \max_{g\in\mathbb{G}}\hat{\Phi}_{\pi,h}^{\lambda}(b_h,b_{h+1};g) \Bigr \}.
\end{align}
Based on \eqref{eq: hat b}, we propose a confidence region for $\bbb^\pi\coloneqq(b_1^{\pi},\cdots,b_H^{\pi})\in\mathbb{B}^{\otimes H}$ as
\begin{align}\label{eq: define CR}
    \CR^\pi(\xi) \coloneqq \left\{ \bbb \in \mathbb{B}^{\otimes H}\middle|\ \max_{g \in \mathbb{G}} \hat{\Phi}_{\pi, h}^\lambda ( b_h, b_{h+1};g) -  \max_{g \in \mathbb{G}}  \hat{\Phi}_{\pi, h}^\lambda(\hat{b}_{h}(b_{h+1}), b_{h+1};g) \leq \xi, \forall  h\right\}.
\end{align}
From the above definition, one can see that $\CR^{\pi}(\xi)$ is actually a coupled sequence of $H$ confidence regions, where each single confidence region aims to cover a function $b_h^\pi$. 
For notational simplicity, we use a single notation $\CR^{\pi}(\xi)$ to denote all the $H$ confidence regions.
Intuitively, the confidence region $\CR^{\pi}(\xi)$ contains all $\bbb\in\mathbb{B}^{\otimes H}$ whose RMSE   does not exceed that of $(\hat{b}_{h}(b_{h+1}), \ b_{h+1})$ by too much at each $h\in[H]$.
The confidence region takes the sequential dependence of confounding bridge functions into consideration in the sense that each $\bbb\in\CR^{\pi}(\xi)$ is restricted through the minimax estimation loss between continuous steps.
As we show in Section \ref{sec: proof sketch}, with high probability, the confidence region $\CR^\pi(\xi)$ contains the true bridge value functions $\bbb^{\pi}$. 
More importantly, every $\bbb\in\CR^{\pi}(\xi)$ enjoys a fast statistical rate of $\tilde{\mathcal{O}}(n^{-1/2})$.

Now combining the confidence region \eqref{eq: define CR} and the identification formula \eqref{eq: identifiaction}, for any policy $\pi\in\Pi(\cH)$, we adopt an pessimistic estimate of the value of $J(\pi)$ as 
\begin{align}\label{eq: pess J}
    \hat{J}_{\textnormal{Pess}}(\pi)\coloneqq  \min_{\bbb\in\CR^{\pi}(\xi)} \hat F(\bbb) , \ \ \textnormal{where} \ \ \hat F(\bbb) \coloneqq \hat{\EE}_{\pi^b}\left[\sum_{a\in\mathcal{A}}b_1(a,W_1)\right] .
\end{align}

\subsection{Policy Optimization}\label{subsec: policy optimization}
Finally, given the pessimistic value estimate \eqref{eq: pess J}, \texttt{P3O} chooses $\hat{\pi}$ which maximizes $\hat{J}_{\textnormal{Pess}}(\pi)$, that is,
% \vspace{-10pt}
\begin{align}\label{eq: def hat pi}
     \hat\pi \coloneqq \argmax_{\pi\in\Pi(\cH)}\, \hat{J}_{\textnormal{Pess}}(\pi).
\end{align}
We summarize the entire \texttt{P3O} algorithm in Algorithm \ref{alg: main}.
In Section \ref{sec: theoretical results}, we show that under some mild assumptions on the function classes $\mathbb{B}$ and $\mathbb{G}$ and under only a partial coverage assumption on the dataset $\mathbb{D}$, we can show that the suboptimality \eqref{eq: suboptimality} of Algorithm \ref{alg: main}
decays at the fast statistical rate of $\tilde{\mathcal{O}}(n^{-1/2})$, where $\tilde \cO(\cdot)$ omits $H$ and factors that characterize the complexity of the function classes.

\begin{algorithm}[t]
	\caption{Proxy variable Pessimistic Policy Optimization (\texttt{P3O})}
	\label{alg: main}
	\begin{algorithmic}[1]
	\STATE \textbf{Input}: confidence parameter $\xi>0$, regularization parameter $\lambda>0$, dataset $\mathbb{D}$, classes $\mathbb{B}$ and $\mathbb{G}$.
    \STATE Construct minimax estimator confidence region $\CR^{\pi}(\xi)$ for each policy $\pi\in\Pi(\cH)$ by \eqref{eq: define CR}.
    \STATE \textcolor{red1}{\bf Policy evaluation:} pessimistically estimate value $\hat{J}_{\textnormal{Pess}}(\pi)$ for each policy $\pi\in\Pi(\cH)$ by \eqref{eq: pess J}.
    \STATE \textcolor{red1}{\bf Policy optimization:} set $\hat{\pi}$ by \eqref{eq: def hat pi}.
    \STATE \textbf{Output}: $\hat\pi=\{\hat\pi_h\}_{h=1}^H$.
	\end{algorithmic}
\end{algorithm}

\section{Theoretical Results}\label{sec: theoretical results}
In this section, we present our theoretical results. 
For ease of presentation, we first assume that both the primal function class $\mathbb{B}$ and the policy class $\Pi(\cH)$ are finite sets with cardinality $|\mathbb{B}|$ and $|\Pi(\cH)|$, respectively. But we allow the dual function class $\mathbb{G}$ to be an infinite set.
Our results can be easily extended to infinite $\mathbb{B}$ and $\Pi(\cH)$ using the notion of covering numbers \citep{wainwright2019high}, which we demonstrate with linear function approximation in Section \ref{sec: linear case}.

We first introduce some necessary assumptions for efficient learning of the optimal policy.
To begin with, the following Assumption \ref{assump: partial coverage} ensures that the offline data generated by $\pi^b$ has a good coverage over $\pi^{\star}$. The problem would become intractable without such an assumption~\citep{chen2019information,jin2021pessimism}.

\begin{assumption}[Partial coverage]\label{assump: partial coverage} 
We assume that the concentrability coefficient for the optimal policy $\pi^{\star}$, defined as $C^{\pi^{\star}}\coloneqq\max_{h\in[H]}\mathbb{E}_{\pi^b}\left[(q^{\pi^{\star}}_h(A_h,Z_h))^2\right]$, satisfies that $C^{\pi^{\star}}<+\infty$.
\end{assumption}

Very importantly, Assumption \ref{assump: partial coverage} only assumes the partial coverage, i.e., the optimal policy $\pi^\star$ is well covered by $\pi^b$~\citep{jin2021pessimism,uehara2021pessimistic}, which is significantly weaker than the uniform coverage, i.e., the entire policy class $\Pi(\cH)$ is covered by $\pi^b$~\citep{munos2008finite, chen2019information} in the sense that $\max_{\pi \in \Pi(\cH)} C^{\pi} < +  \infty.$. 

The next assumption is on the functions classes $\BB$ and $\GG$. 
We require that $\BB$ and $\GG$ are uniformly bounded, and that $\GG$ is symmetric, star-shaped, and has bounded localized Rademacher complexity.

\begin{assumption}[Function classes $\mathbb{B}$ and $\mathbb{G}$]\label{assump: dual function class}
    We assume the classes $\mathbb{B}$ and $\mathbb{G}$ satisfy that:
    i) There exist $M_{\mathbb{B}},M_{\mathbb{G}}<+\infty$ such that $\mathbb{B}$, $\mathbb{G}$ are bounded by $
    \sup_{b\in\mathbb{B}}\sup_{w\in\mathcal{W}}|\sum_{a\in\mathcal{A}}b(a,w)|\leq M_{\mathbb{B}}
    $\footnote{The constant $\M_{\BB}$ might seem to be proportional to $|\cA|$ due to the summation $\sum_{a \in \cA} b(a,w)$, but it is not. The reason is that the definition of the true confounding bridge function in \eqref{eq: def value bridge} involves a product with $\pi_h(\cdot|O_h, \Gamma_{h-1})$ which is a distribution over $\cA$. Thus the summation over $\cA$ is essentially an average over $\cA$.},
    $\sup_{g\in\mathbb{G}}\sup_{(a,z)\in\mathcal{A}\times\mathcal{Z}}|g(a,z)|\leq M_{\mathbb{G}}$;
    ii) $\mathbb{G}$ is star-shaped, i.e., for any $g\in\mathbb{G}$ and $\lambda\in[0,1]$, it holds that $\lambda g\in\mathbb{G}$; 
    iii) $\mathbb{G}$ is symmetric, i.e., for any $g\in\mathbb{G}$, it holds that $-g\in\mathbb{G}$; 
    iv) For any step $h\in[H]$, $\mathbb{G}$ has bounded critical radius $\alpha_{\mathbb{G},h,n}$ which solves inequality $\mathcal{R}_n(\mathbb{G};\alpha)\leq \alpha^2 /M_{\mathbb{G}}$, where $\mathcal{R}_n(\mathbb{G},\alpha)$ is the localized population Rademacher complexity \comment{\footnote{The localized rademacher complexity \citep{wainwright2019high} generalizes the standard rademacher complexity, which is key to the development of sharp statistical rates. We refer the readers to Appendix \ref{sec: auxiliary lemmas} for more about this complexity measure.}} of $\mathbb{G}$ under the distribution of $(A_h,Z_h)$ induced by $\pi^b$, that is,
    \begin{align*}
     \mathcal{R}_n(\mathbb{G},\alpha)=\mathbb{E}_{\pi^b,\epsilon_i}\bigg[\sup_{g\in\mathbb{G}:\|g\|_2\leq \alpha}\bigg|\frac{1}{n}\sum_{i=1}^n\epsilon_i g(A_h,Z_h)\bigg|\Bigg],
    \end{align*}
    with $\|g\|_2$ defined as $(\mathbb{E}_{\pi^b}[g^2(A_h,Z_h)])^{1/2}$, random variables $\{\epsilon_i\}_{i=1}^n$ independent of $(A_h,Z_h)$ and independently uniformly distributed on $\{+1,-1\}$. Also, we denote $\alpha_{\mathbb{G},n} \coloneqq\max_{h\in[H]}\alpha_{\mathbb{G},h,n}$.
\end{assumption}

Finally, to ensure that the minimax estimation discussed in Section \ref{subsec: minimax estimator} learns the value bridge functions unbiasedly, 
we make the following completeness and realizability assumptions on the function classes $\mathbb{B}$ and $\mathbb{G}$ which are standard in the literature~\citep{dikkala2020minimax,xie2021bellman, shi2021minimax}.

\begin{assumption}[Completeness and realizability]\label{assump: completeness and realizability}
% \label{assump: completeness for optimal g}
        We assume that, i) completeness: for any  $h\in[H]$, any $\pi \in \Pi(\cH)$, and any  $b_h,b_{h+1}\in\mathbb{B}$, it holds that $\frac{1}{2\lambda}\ell_h^{\pi}(b_h,b_{h+1})\in\mathbb{G}$
    % \begin{align*}
    %   \frac{1}{2\lambda}\ell_h^{\pi}(b_h,b_{h+1})\in\mathbb{G},
    % \end{align*}
% \label{assump: realizability}
    where $\ell_h^{\pi}$ is defined in \eqref{eq: bellman}; ii) realizability: for any $h\in[H]$, any $\pi \in \Pi(\cH)$, and any $b_{h+1} \in \mathbb{B}$, there exists $b^{\star} \in \mathbb{B}$ such that $\cL_h^{\pi} (b^{\star},b_{h+1}) \leq \epsilon_{\mathbb{B}}$ for some $ \epsilon_{\mathbb{B}}<+\infty$, i.e., we assume that
    \begin{align*}
        0\leq \epsilon_{\mathbb{B}}\coloneqq\max_{h\in[H],\pi\in\Pi(\mathcal{H}),b_{h+1}\in\mathbb{B}}\,\min_{b_h\in\mathbb{B}}\,\,\mathbb{E}_{\pi^b}\big[\ell_h^{\pi} (b_h,b_{h+1})(A_h,Z_h)^2\big] <+\infty.
    \end{align*}
\end{assumption}

Here the completeness assumption means the dual function class $\GG$ is rich enough 
which guarantees the equivalence between $\cL_h^{\pi}(\cdot,\cdot)$ and $\max_{g\in\GG}\Phi_{\pi,h}^{\lambda}(\cdot,\cdot;g)$.
The realizability assumption means that the primal function class $\mathbb{B}$ is rich enough such that \eqref{eq: def value bridge} always admits an (approximate) solution.

With these technical assumptions, we can establish our main theoretical results in the following theorem, which gives an upper bound of the suboptimality \eqref{eq: suboptimality} of the policy $\hat\pi$ output by Algorithm~\ref{alg: main}.
% Our main theoretical result is the following upper bound on the suboptimality of $\hat \pi$ obtained by \eqref{eq: def hat pi}.

\begin{theorem}[Suboptimality]\label{thm: suboptimality}
Under Assumptions \ref{assump: negative control cond independence}, \ref{assump: bridge functions exist}, 
\ref{assump: partial coverage},
\ref{assump: dual function class}, and
\ref{assump: completeness and realizability},
by setting the regularization parameter $\lambda$ and the confidence parameter $\xi$ as $\lambda=1$ and  
\begin{align*}
 \xi=C_1\cdot M_{\mathbb{B}}^2M_{\mathbb{G}}^2 \cdot \log(|\mathbb{B}||\Pi(\cH)|H/\zeta) / n,
\end{align*}
then probability at least $1-3\delta$, it holds that
\begin{align*}
    \textnormal{SubOpt}(\hat{\pi})\leq
   C_1^{\prime}\sqrt{C^{\pi^{\star}}}H\cdot M_{\mathbb{B}}M_{\mathbb{G}}\cdot \sqrt{ \log(|\mathbb{B}||\Pi(\cH)|H/\zeta) / n}+C_1^{\prime}\sqrt{C^{\pi^{\star}} M_{\mathbb{G}}}H \epsilon_{\mathbb{B}}^{1/4},
\end{align*}
where $\zeta = \min\{\delta,4c_1\exp(-c_2n\alpha_{\mathbb{G},n}^2)\}$.
Here $C^{\pi^{\star}}$, $\alpha_{\mathbb{G},n}$, and $\epsilon_{\mathbb{B}}$ are defined in Assumption \ref{assump: partial coverage}, \ref{assump: dual function class}, and \ref{assump: completeness and realizability}, respectively. 
And $C_1$, $C'_1$, $c_1$, and $c_2$ are some problem-independent universal constants.
\end{theorem}

We introduce all the key technical lemmas and sketch the proof of Theorem \ref{thm: suboptimality} in Section \ref{sec: proof sketch}. 
We refer to Appendix \ref{sec: suboptimality proof} for a detailed proof. 
When it holds that $\alpha_{\mathbb{G},n}=\tilde{\mathcal{O}}(n^{-1/2})$ and $\epsilon_{\mathbb{B}}=0$, Theorem \ref{thm: suboptimality} implies that $\textnormal{SubOpt}(\hat{\pi})\leq\tilde{\mathcal{O}}(n^{-1/2})$, which corresponds to a ``fast statistical rate'' for minimax estimation \citep{uehara2021finite}.
The derivation of such a fast rate relies on a novel analysis for the risk of functions in the confidence region, which is shown by Lemma \ref{lem: accuracy of CR} in Section \ref{sec: proof sketch}.
Meanwhile, for many choices of the dual function class $\mathbb{G}$, it holds that $\alpha_{\mathbb{G},n}$ scales with $\sqrt{\log \mathcal{N}_{\mathbb{G}}}$ where $\mathcal{N}_{\mathbb{G}}$ denotes the complexity measure of the class $\mathbb{G}$.
In such cases, the suboptimality also scales with $\sqrt{\log \mathcal{N}_{\mathbb{G}}}$, without explicit dependence on the cardinality of the spaces $\mathcal{S}$, $\mathcal{A}$, or $\mathcal{O}$.
Finally, we highlight that, thanks to the principle of pessimism, the suboptimality of \texttt{P3O}  depends only on the partial coverage concentrability coefficient $C^{\pi^{\star}}$, which can be significantly smaller than the uniform coverage concentrability coefficient  $\sup_{\pi\in\Pi(\cH)}C^{\pi}$.
In conclusion, when $\alpha_{\mathbb{G},n}=\tilde{\mathcal{O}}(\sqrt{\log \mathcal{N}_{\mathbb{G}}/n})$ and $\epsilon_{\mathbb{B}}=0$,
the \texttt{P3O} algorithm enjoys a $\tilde{\mathcal{O}}(H\sqrt{C^{\pi^{\star}}\log \mathcal{N}_{\mathbb{G}}/n})$ suboptimality.

In the above, we assume that the bridge function class $\mathbb{B}$ and policy class $\Pi(\cH)$ are finite. Next, we show that Theorem~\ref{thm: suboptimality} can be readily extended to the case of linear function approximation (LFA) with infinite-cardinality $\mathbb{B}$ and $\Pi(\cH)$, which yields an $\tilde \cO ( \sqrt{H^3 d /n})$ suboptimality.

\subsection{Case Study: Linear Function Approximation}\label{sec: linear case}

In this subsection, we extend Theorem \ref{thm: suboptimality} to primal function class $\BB$, dual function class $\mathbb{G}$, and policy class $\Pi(\cH)$ with linear structures \citep{jin2021pessimism, xie2021bellman,zanette2021provable,duan2021risk}.
The following definition characterizes such linear function classes $\Blin$, $\Glin$ and $\Pilin$.

\begin{definition}[Linear function approximation]\label{def: LFA}
Let $\phi:\cA \times \cW \to \RR^d$ be a feature mapping for some integer $d\in\mathbb{N}$. 
We let the primal function class be $\BB = \Blin$ where
\begin{align*}
    \Blin \coloneqq \left\{b  \ \middle| \ b(\cdot,\cdot) = \langle\phi(\cdot,\cdot), \theta \rangle, \theta \in \RR^d,\|\theta\|_2 \leq L_b, \ \sup_{w\in\mathcal{W}}|\sum_{a\in\mathcal{A}}b(a,w)|\leq M_{\mathbb{B}} \right\}.
\end{align*}
Let $\psi=\{\psi_h:\cA \times \cO \times \cH_{h-1} \to \RR^d\}_{h=1}^H$ be $H$ feature mappings. 
We let the policy function class be $\Pi(\cH) =\Pilin$ where $\Pilin=\{\Pi_{\textnormal{lin},h}\}_{h=1}^H$ and each $\Pi_{\textnormal{lin},h}$ is defined as
\begin{align*}
    \Pi_{\textnormal{lin},h} \coloneqq \left\{ \pi_h \ \middle| \ \pi_h(a|o, \tau) = \frac{\exp{(\langle \psi_h(a,o,\tau),\beta\rangle )}}{\sum_{a' \in \cA} \exp({\langle \psi_h(a',o,\tau),\beta\rangle }) }, \ \beta\in\RR^d, \ \|\beta\|_2\leq L_\pi \right\}. 
\end{align*}
Finally, let $\nu:\cA \times \cZ \to \RR^d$ be another feature mapping. We let the dual function class be $\GG=\Glin$ where 
\begin{align*}
    \Glin \coloneqq \left\{g  \ \middle| \ g(\cdot,\cdot) = \langle\nu(\cdot,\cdot), \omega \rangle, \omega \in \RR^d,\|\omega\|_2 \leq L_g \right\}.
\end{align*}Assume without loss of generality that these feature mappings are normalized, i.e., $\|\phi\|_2, \|\psi\|_2, \|\nu\|_2 \leq 1$.
\end{definition}

We note that Definition \ref{def: LFA} is consistent with Assumption \ref{assump: dual function class}. One can see that $\Blin$ and $\Glin$ is uniformly bounded, $\Glin$ is symmetric and star-shaped.
And for other more detailed theoretical properties of $\Blin$, $\Glin$, and $\Pilin$, we refer the readers to Appendix \ref{sec: basics for LFA} for corresponding results.

Under linear function approximation, we can extend Theorem \ref{thm: suboptimality} to the following corollary, which characterizes the suboptimality \eqref{eq: suboptimality} of $\hat{\pi}$ found by $\texttt{P3O}$ when using 
$\Blin$, $\Glin$, and $\Pilin$ as function classes.

\begin{corollary}[Suboptimality analysis: linear function approximation]\label{thm: LFA subopt linear}
With linear function approximation (Definition \ref{def: LFA}),
under Assumption \ref{assump: negative control cond independence}, \ref{assump: bridge functions exist}, \ref{assump: partial coverage}, 
% \ref{assump: dual function class}, 
and \ref{assump: completeness and realizability}, 
by setting the regularization parameter $\lambda$ and the confidence parameter $\xi$ as $\lambda=1$ and 
\begin{align*}
    \xi = C_2M_{\BB}^2 \cdot M_{\GG}^2 \cdot dH\cdot  \log\left({1 + L_b L_\pi H n/\delta}\right)/n,
\end{align*}
then with probability at least $1-\delta$, it holds that
    \begin{align*}
      \textnormal{SubOpt}(\hat{\pi})& \leq C_2^{\prime}\sqrt{C^{\pi^{\star}}}HM_{\BB} L_g \sqrt{ dH\log\left(1+L_b L_{\pi} H n/\delta \right) / n} + C_2^{\prime}\sqrt{C^{\pi^{\star}} L_g}H \epsilon_{\BB}^{1/4}.
    \end{align*}
    Here $C_2$ and $C_2^{\prime}$ are problem-independent universal constants.
\end{corollary}

\begin{proof}[Proof of Corollary~\ref{thm: LFA subopt linear}]
See Appendix \ref{sec: proof of LFA subopt} for a detailed proof.
\end{proof}

The guarantee of Corollary~\ref{thm: LFA subopt linear} is structurally similar to that of Theorem~\ref{thm: suboptimality}, except that we can explicitly compute the complexity of the linear function classes and policy class. 
When $\epsilon_{\mathbb{B}}=0$, \texttt{P3O} algorithm enjoys a $\tilde{\cO}(\sqrt{C^{\pi^{\star}}H^3d/n})$ suboptimality
under the linear function approximation.
Compared to Theorem~\ref{thm: suboptimality}, Corollary~\ref{thm: LFA subopt linear} does not explicitly assume Assumption~\ref{assump: dual function class} since it is implicitly satisfied by Definition~\ref{def: LFA}.

\section{Proof Sketches}\label{sec: proof sketch}
In this section, we sketch the proof of the main theoretical result Theorem \ref{thm: suboptimality}, and we refer to Appendix \ref{sec: suboptimality proof} for a detailed proof.
For simplicity, we denote that for any $\pi\in\Pi(\mathcal{H})$ and $\mathbf{b}\in\mathbb{B}^{\otimes H}$,
\begin{align}\label{eq: define J}
    F(\bbb)\coloneqq\mathbb{E}_{\pi^b}\left[ \sum_{a\in\cA} b_1 (a, W_1) \right],\quad
    \hat{F}(\bbb)\coloneqq\hat{\EE}_{\pi^b}\left[ \sum_{a\in\cA} b_1 (a, W_1) \right].
\end{align}
By the definition \eqref{eq: define J} and Theorem \ref{thm: identification}, for any policy $\pi\in\Pi(\cH)$, it holds that $J(\pi)=F(\bbb^{\pi})$, where we have denoted by $\bbb^{\pi}=(b^{\pi}_1,\cdots,b^{\pi}_H)$ the vector of true value bridge functions of $\pi$ which are given in \eqref{eq: def value bridge}.

Our proof to Theorem \ref{thm: suboptimality} relies on the following three key lemmas.
The first lemma relates the different values of mapping $F(\cdot)$ induced by a true value bridge function $\bbb^{\pi}$ and any other functions $\bbb\in\mathbb{B}^{\otimes H}$ to the RMSE loss   which we aim to minimize by algorithm design.
This indeed decomposes the suboptimality \eqref{eq: suboptimality}.
\begin{lemma}[Suboptimality decomposition]\label{lem: regret decomposition}
    Under Assumption \ref{assump: negative control cond independence}, \ref{assump: bridge functions exist}, for any policy $\pi\in\Pi(\cH)$ and $\bbb\in\BB^{\otimes H}$, it holds that 
    \begin{align*}
       F(\bbb^{\pi})-F(\bbb)\leq \sum_{h=1}^H\gamma^{h-1}\sqrt{C^{\pi}}\cdot\sqrt{\mathcal{L}_h^{\pi}(b_h,b_{h+1})},
    \end{align*}
    where the concentrability coefficient $C^{\pi}$ is defined as $
      C^{\pi}\coloneqq\sup_{h\in[H]}\mathbb{E}_{\pi^b}\left[(q^{\pi}_h(A_h,Z_h))^2\right]$.
    
\end{lemma}
\begin{proof}[Proof of Lemma \ref{lem: regret decomposition}]
See Appendix \ref{subsec: regret decomposition proof} for a detailed proof.
\end{proof}

The following two lemmas characterize the theoretical properties of the confidence region $\CR^{\pi}(\xi)$.
Specifically, Lemma \ref{lem: true in CR} shows that with high probability the confidence region of $\pi$ contains the true value bridge function $\bbb^{\pi}$.
Besides, Lemma \ref{lem: accuracy of CR} shows that each bridge function vector $\bbb\in\CR^{\pi}(\xi)$ enjoys a fast statistical rate \citep{uehara2021finite} for its RMSE loss $\mathcal{L}_h^{\pi}$ defined in \eqref{eq: RMSE loss}.
To obtain such a fast rate, we develop novel proof techniques in 
Appendix \ref{subsec: accuracy of CR proof}.

\begin{lemma}[Validity of confidence regions]\label{lem: true in CR} 
Under Assumption \ref{assump: bridge functions exist} and \ref{assump: dual function class}, for any $0<\delta<1$, by setting 
\begin{align*}
    \xi= C_1(\lambda+1/\lambda)\cdot M_{\mathbb{B}}^2 \cdot M_{\mathbb{G}}^2\cdot \log(|\mathbb{B}||\Pi(\cH)|H/\zeta) / n,
\end{align*}
for some problem-independent universal constant $C_1>0$ and $\zeta = \min\{\delta,4c_1\exp(-c_2n\alpha_{\mathbb{G},n}^2)\}$, it holds with probability at least $1-\delta$ that 
$\bbb^{\pi}\in\CR^{\pi}(\xi)$ for any policy $\pi\in\Pi(\mathcal{H})$.
\end{lemma}

\begin{proof}[Proof of Lemma \ref{lem: true in CR}]
See Appendix \ref{subsec: true in CR proof} for a detailed proof.
\end{proof}

\begin{lemma}[Accuracy of confidence regions]\label{lem: accuracy of CR}
Under Assumption \ref{assump: bridge functions exist}, \ref{assump: dual function class}, and \ref{assump: completeness and realizability}, by setting the same $\xi$ as in Lemma \ref{lem: true in CR}, with probability at least $1-\delta/2$, for any policy $\pi\in\Pi(\mathcal{H})$, $\bbb\in\CR^{\pi}(\xi)$, and step $h$,
\begin{align*}
    \sqrt{\mathcal{L}^{\pi}_h(b_h,b_{h+1})}\leq \tilde{C}_1 M_{\mathbb{B}}M_{\mathbb{G}}\sqrt{ \left(\lambda+1/\lambda\right) \cdot \log(|\mathbb{B}||\Pi(\cH)|H/\zeta) / n}+\tilde{C}_1\epsilon_{\mathbb{B}}^{1/4}M_{\mathbb{G}}^{1/2},
\end{align*}
for some problem-independent universal constant $\tilde{C}_1>0$, and $\zeta = \min\{\delta,4c_1\exp(-c_2n\alpha_{\mathbb{G},n}^2)\}$.
\end{lemma}

\begin{proof}[Proof of Lemma \ref{lem: accuracy of CR}]
See Appendix \ref{subsec: accuracy of CR proof} for a detailed proof.
\end{proof}

When $\alpha_{\mathbb{G},n}\in{\mathcal{O}}(n^{-1/2})$ and $\epsilon_{\mathbb{B}}=0$, Lemma \ref{lem: accuracy of CR} implies that $\mathcal{L}^{\pi}_h(b_h,b_{h+1})\leq\tilde{\mathcal{O}}(n^{-1})$.
Now with Lemma \ref{lem: regret decomposition}, Lemma \ref{lem: true in CR}, and Lemma \ref{lem: accuracy of CR}, by the choice of $\hat{\pi}$ in \texttt{P3O}, we can show that 
\begin{align}\label{eq: sketch}
    J(\pi^{\star})-J(\hat{\pi})&\leq \Tilde{\mathcal{O}}(n^{-1/2})+ \max_{\bbb\in\CR^{\pi^{\star}}(\xi)}F(\bbb)-\min_{\bbb\in\CR^{\hat{\pi}}(\xi)}F(\bbb)\notag\\
    &\leq \Tilde{\mathcal{O}}(n^{-1/2})+\max_{\bbb\in\CR^{\pi^{\star}}(\xi)}F(\bbb)-\min_{\bbb\in\CR^{\pi^{\star}}(\xi)}F(\bbb)\notag\\
    &\leq \Tilde{\mathcal{O}}(n^{-1/2})+
    2\max_{\bbb\in\CR^{\pi^{\star}}(\xi)}\left|F(\bbb)-F(\bbb^{\pi^{\star}})\right|\notag\\
    &\leq \Tilde{\mathcal{O}}(n^{-1/2})+2\max_{\bbb\in\CR^{\pi^{\star}}(\xi)}
    \sum_{h=1}^H\gamma^{h-1}\sqrt{C^{\pi^{\star}}}\cdot\sqrt{\mathcal{L}_h^{\pi^{\star}}(b_h,b_{h+1})},
\end{align}
where the first inequality holds  by Lemma \ref{lem: true in CR}, the second inequality holds from the optimality of $\hat{\pi}$ in Algorithm \ref{alg: main}, the third inequality holds directly, and the last inequality holds by Lemma \ref{lem: regret decomposition}.
Finally, by applying Lemma \ref{lem: accuracy of CR} to the right hand side of \eqref{eq: sketch}, we conclude the proof of Theorem \ref{thm: suboptimality}.
    
\section{Conclusion}

In this work, we propose the first provably efficient offline RL algorithm for POMDPs with confounded datasets.
Such a problem involves the coupled challenges of confounding bias, distributional shift, and function approximation. 
We propose a novel policy optimization algorithm which leverages proximal causal inference for handling the confounding bias, and adopts pessimism to tackle the distributional shift. 
The core of our algorithm is a coupled sequence of confidence regions constructed via proximal causal inference and minimax estimation, which is able to incorporate general function approximation and enables pessimistic policy optimization. 
We prove that the proposed algorithm achieves $n^{-1/2}$-suboptimality under a partial coverage assumption on the offline dataset. We believe the novel algorithm design and analysis that leverage techniques from causal inference will be promising for future research on offline reinforcement learning with partial observations.

\clearpage
\bibliographystyle{ims}
\bibliography{ref}
\newpage
\appendix
\section{Table of Notations}\label{sec: table of notations} 
In this section, we provide a comprehensive clarification on the use of notation in this paper. 
We use lower case letters (i.e., $s$, $a$, $o$, and $\tau$) to represent dummy variables and upper case letters (i.e., $S$, $A$, $O$, and $\Gamma$) to represent random variables. We use the variables in the calligraphic font (i.e., $\cS$, $\cA$, $\cO$, and $\cH$) to represent the spaces of variables, and the blackboard bold font (i.e., $\mathbb{P}$ and $\mathbb{O}$) to represent probability kernels.  

We use $\mathcal{H}=\{\mathcal{H}_{h}\}_{h=0}^{H-1}$ to denote the space of observable history, where each element $\tau_h \in \mathcal{H}_h$ is a (partial) trajectory such that $\tau_h \subseteq  \{(o_1,a_1),\cdots,(o_h,a_h)\}$.
We use $\pi^b = \{\pi^b_h\}_{h=1}^{H}$ to denote the behavior policy, 
where $\pi_h^b:\mathcal{S}\mapsto\Delta(\mathcal{A})$.
We use $\pi = \{\pi_h\}_{h=1}^{H}\in\Pi(\cH)$ to denote a history-dependent policy
with $\pi_h:\mathcal{O}\times\mathcal{H}_{h-1}\mapsto\Delta(\mathcal{A})$. Also, we use $\pi^{\star} = \{\pi_h^{\star}\}_{h=1}^{H}$ to denote the optimal history-dependent policy.
Offline data $\mathbb{D}$ is collected by $\pi^b$, as described in Section \ref{subsec: data generation}.

We use $\cP^{b}=\{\cP_h^b\}_{h=1}^H$ and $\cP^{\pi}=\{\cP_h^\pi\}_{h=1}^H$ to denote the distribution of trajectories under the policy $\pi^b$ and $\pi$, respectively, where
$\cP^{b}_h$ and $\cP^{\pi}_h$ denote the density of corresponding variables at step $h$.
Also, we use $\EE_{\pi^b}$ and $\EE_{\pi}$ to denote the expectation w.r.t. the distribution $\cP^{b}$ and $\cP^{\pi}$.
We use $\hat{\mathbb{E}}_{\pi^b}$ to denote the empirical version of $\mathbb{E}_{\pi^b}$, which is calculated on data $\mathbb{D}$.

Through out the paper, we use $\mathcal{O}(\cdot)$
to hide problem-independent constants and use $\tilde{\cO}(\cdot)$ to hide problem-independent constants plus logarithm factors.
The following table summaries the notations we used in our proposed algorithm design and theory.

\begin{table}[!ht]
\centering
\renewcommand*{\arraystretch}{1.5}
\begin{tabular}{ >{\centering\arraybackslash}m{2cm} | >{\centering\arraybackslash}m{11.1cm} } 
\hline\hline
Notation & Meaning \\ 
\hline
 $J(\pi)$ & Policy value $\mathbb{E}_{\pi}[\sum_{h=1}^H\gamma^{h-1}R_h]$\\ 

 $b_h^{\pi}$, $q_h^{\pi}$ & value bridge function, weight bridge function of $\pi$ at step $h$\\ 

 $\bbb^{\pi}$, $\mathbf{q}^{\pi}$ & value bridge function vector, weight bridge function vector of $\pi$\\ 
 
 $\CR^{\pi}(\xi)$ & confidence region of $\bbb^{\pi}$, according to \eqref{eq: define CR}\\
 
 $\bbb$ & an element in the confidence region $\CR^{\pi}(\xi)$\\
 
 $F(\bbb)$, $\hat F(\bbb)$ & a mapping for identification with $J(\pi)=F(\bbb^{\pi})$, according to \eqref{eq: identifiaction}\\
 
 \hline
 
 $\ell_h^{\pi}$ & "Bellman residual" for bridge functions, according to \eqref{eq: bellman}
\\

$\mathcal{L}_h^{\pi}$ & residual mean square loss for $\ell_h^{\pi}$, according to \eqref{eq: RMSE loss}\\
 
 $\Phi_{\pi,h}^{\lambda}$,  $\hat{\Phi}_{\pi,h}^{\lambda}$
 & a mapping for minimax estimation, according to \eqref{eq: population phi lambda}\\
 
$\hat{b}_h(b_{h+1})$ & minimax estimator of $b_h^{\pi}$ given $b_{h+1}$, according to \eqref{eq: hat b}
 \\ 
\hline

$\hat{J}_{\text{pess}}(\pi)$ & pessimistic estimator of $J(\pi)$, according to \eqref{eq: pess J}\\
$\hat{\pi}$ & policy returned by \texttt{P3O} algorithm, according to \eqref{eq: def hat pi}\\
\hline \hline
\end{tabular}
\caption{Table of Notations}
\label{tab:notation}
\end{table}

\section{Further Discussions}

\subsection{Discussions about the Partial Coverage Assumption}

Our work assumes the partial coverage of $\mathbb{D}$ according to Assumption \ref{assump: partial coverage}, where we implicitly requires that $\mathcal{P}_h^\pi\left(S_h, \Gamma_{h-1}\right) / \mathcal{P}_h^b\left(S_h, \Gamma_{h-1}\right)<+\infty$ for all $\pi\in\Pi(\mathcal{H})$ (we call it the finite-ratio condition from here). 
We note that this finite-ratio condition can NOT be regarded as the full coverage assumption. Instead, this is a regularity condition that arises from causal inference. 

First of all, the finite-ratio condition is different from the full coverage assumption in standard MDPs. 
The Full coverage assumption in standard MDPs usually takes the form that 
$$\max _{\pi \in \Pi} \frac{\mathcal{P}_{h}^\pi(s, a) }{ \mathcal{P}_{h}^b(s, a)}<C,$$
for some fixed $C>0$. 
This condition means the density ratio of the marginal distributions of $(s,a)$ between any target policy $\pi$
and the behavior policy $\pi^b$ is uniformly bounded by a constant. 
This condition (or some similar form) is a common and widely accepted form of full coverage in the MDP literature, e.g. \citep{chen2019information, xie2020q}. 
Note that this constant $C$ is a uniform upper bound over the candidate policy class.
Very importantly, this constant $u^{\prime}$ appears in the final error bound. 
The partial coverage assmuption in MDP, on the other hand, is commonly formulated as 
$$ \frac{\mathcal{P}_{h}^{\pi^{\star}}(s, a) }{ \mathcal{P}_{h}^b(s, a)}<C,$$
This condition means the density ratio of the marginal distributions of $(s, a)$ between only the optimal policy $\pi^{\star}$ and the behavior policy $\pi^b$, is bounded by a constant. 
The form of this assumption is very close to Assumption \ref{assump: partial coverage} (Partial coverage) in our paper. 
In other words, our Assumption~\ref{assump: partial coverage} is a version of the partial coverage assumption that is tailored to the POMDP case. 
Notably, this constant $C$ in the partial coverage assumption also appears in the final error bound. 

As a sharp comparison to both the full coverage and partial coverage assumptions, the finite-ratio condition that the quantity $\mathcal{P}_h^\pi\left(S_h, \Gamma_{h-1}\right) / \mathcal{P}_h^b\left(S_h, \Gamma_{h-1}\right)<+\infty$ for all $\pi\in\Pi(\mathcal{H})$ does not result in any constant factor that appears in the final error bound. 
In the case of infinite policy class $\Pi(\mathcal{H})$, we can allow the ratio to be arbitrarily large and that won't hurt our final error bound. 
Therefore, this is not a coverage assumption. 
Our finite-ratio condition is a regularity condition that arises from causal inference. 
This condition is needed to deal with the extra challenge of the confounding issue in our POMP setting. 
In related works studying OPE under confounded POMDP \citep{shi2021minimax}, this finite-ratio condition is also needed. 
Overall, our paper is indeed under partial coverage and the finite ration condition is not a kind of coverage assumption.

\subsection{Discussions on Relations Between Minimax-typed Loss and Least-square-typed Loss} 
During the preparation of this paper, we find that \emph{in the MDP setting}, the least-square-typed loss considered by \citep{xie2021bellman} can be reformulated to the minimax-typed loss that we consider in this paper with a different dual function class. 
To see this, consider the MDP setting with a single transition tuple $(S_h ,A_h, S_{h+1})$. 
The goal is to estimate the Bellman target $(\mathcal{B}  V_{h+1} ) \colon \mathcal{S} \times \mathcal{A} \rightarrow \mathbb{R}$, 
where $\mathcal{B}$ is the Bellman operator and $V_{h+1} \colon \mathcal{S} \rightarrow \mathbb{R}$ is a fixed state-value function.
For each $(s,a)\in\mathcal{S} \times \mathcal{A} $, $(\mathcal{B}^{\pi} f_{h+1} ) (s,a)$ is given by 
\begin{align*}
(\mathcal{B} f_{h+1} )  (s,a) = R_h (s,a) +\int_{\mathcal{S}} P_h (\mathrm{d}s' | s,a) V_{h+1} (s') . 
\end{align*}
Here $R_h$ is the reward function and we can assume it is known for now, and $P_h:\mathcal{S}\times\mathcal{A}\mapsto\Delta(\mathcal{S})$ is the unknown transition kernel. 
We use function class $\mathcal{F}$ to approximate the bellman target.
Then based on the offline transition data $\mathbb{D}=\{(s_h^{\tau} , a_h^{\tau}, s_{h+1}^{\tau} )\}_{\tau=1}^N$, the least-square-typed loss function given in Equation (3.1) of \citep{xie2021bellman} becomes 
\begin{align}\label{eq: ls loss}
    \widehat{\mathcal{L}}_{h}^{\text{ls}}(f_h)=\widehat{\mathbb{E}}_{\mathbb{D}}&\left[\big(f_h(S_h,A_h)-R_h-V_{h+1}(S_{h+1}\big)^2\right]\notag\\&- \min_{f^{\prime}_h\in\mathcal{F}}
    \widehat{\mathbb{E}}_{\mathbb{D}}\left[\big(f_h^{\prime}(S_h,A_h)-R_h-V_{h+1}(S_{h+1}\big)^2\right],
\end{align}
where $R_h$ is an abbreviation for $R_h(S_h,A_h)$. Using the equality $x^2-y^2=(x+y)(x-y)$, we can rewrite the least-square-typed loss \eqref{eq: ls loss} as
\begin{align*}
    \widehat{\mathcal{L}}_{h}^{\text{ls}}(f_h)=\sup_{f_h^{\prime}\in\mathcal{F}}\widehat{\mathbb{E}}_{\mathbb{D}}\Big[\Big((f_h+f_h^{\prime})(S_h,A_h)-2R_h-2V_{h+1}(S_{h+1})\Big)\Big((f_h-f_h^{\prime})(S_h,A_h)\Big)\Big].
\end{align*}
For derivation, we further rewrite first term as 
\begin{align*}
    &(f_h+f_h^{\prime})(S_h,A_h)-2R_h-2V_{h+1}(S_{h+1}) \\
    &\qquad = \Big(2f_h(S_h,A_h)-2R_h-2V_{h+1}(S_{h+1})\Big)-\Big((f_h-f_h^{\prime})(S_h,A_h)\Big).
\end{align*}
With this, we can then rewrite the least-square-typed loss \eqref{eq: ls loss} as
\begin{align*}
    \widehat{\mathcal{L}}_{h}^{\text{ls}}(f_h)&=\sup_{f_h^{\prime}\in\mathcal{F}}\widehat{\mathbb{E}}_{\mathbb{D}}\Big[\Big(2f_h(S_h,A_h)-2R_h-2V_{h+1}(S_{h+1})\Big)\Big((f_h-f_h^{\prime})(S_h,A_h)\Big)\\
    &\qquad-\Big((f_h-f_h^{\prime})(S_h,A_h)\Big)^2\Big].
\end{align*}
Now by defining a new function class $\mathcal{G}_f$ \emph{depending on $f$} as $\mathcal{G}_f=\{f-f^{\prime}:f^{\prime}\in\mathcal{F}\}$, we arrive that 
\begin{align}\label{eq: ls final}
    \frac{1}{2}\widehat{\mathcal{L}}_{h}^{\text{ls}}(f_h)=\sup_{g_h\in\mathcal{G}_{f_h}}\widehat{\mathbb{E}}_{\mathbb{D}}\Big[\Big(f_h(S_h,A_h)-R_h-V_{h+1}(S_{h+1})\Big)g_h(S_h,A_h)-\frac{1}{2}g_h(S_h,A_h)^2\Big].
\end{align}
This shares the same form as the minimax-typed loss $\sup_{g_h\in\mathbb{G}}\hat{\Phi}_{\pi,h}^{1/2}(b_h,b_{h+1};g_h)$ we consider in our work, see \eqref{eq: empirical phi lambda} in the main text. 
But still there are differences. In \eqref{eq: ls final}, the dual function $g_h$ lies in a dual function class $\mathcal{G}_{f_h}$ which depends on the primal function $f_h$.
While in our minimax-typed loss, the dual function class does not depends on the primal function.

Finally, we need to point out that even the two losses share the same form, the form of the confidence region considered by our work is different from that considered by \cite{xie2021bellman}.
To see this, still using the previous notations, the confidence region in \cite{xie2021bellman} (Equation (3.2)) becomes 
\begin{align*}
    \CR_h(\xi)=\left\{f_h\in\mathcal{F}: \widehat{\mathcal{L}}_{h}^{\text{ls}}(f_h)\leq \xi\right\}.
\end{align*}
Meanwhile, if we reduce our confidence region to the above MDP setting, our confidence region should be in the form of 
\begin{align*}
    \CR_h(\xi)=\left\{f_h\in\mathcal{F}: \widehat{\mathcal{L}}_{h}^{\text{mm}}(f_h)-\min_{f_h\in\mathcal{F}}\widehat{\mathcal{L}}_{h}^{\text{mm}}(f_h)\leq \xi\right\},
\end{align*}
where $\mathcal{L}_{h}^{\text{mm}}(f_h)$ denotes the minimax-typed-loss.
Our algorithm and theoretical analysis are based on the second form of confidence region, which is key to the derivation of fast statistical rates for elements in the  confidence region based on minimax estimation.

\subsection{Discussions of Applications: Real-world Examples of Proximal Causal Inference in RL}

POMDP models have widespread practical applications in fields such as autonomous driving, communication systems, and medical treatment planning.
For instance, in the case of autonomous driving, the state of the environment is generally unknown, and the intention to accelerate or decelerate of surrounding vehicles' drivers can be considered a latent state. 
Therefore, observations collected and processed by devices on the autonomous vehicle, such as images and videos, are used to estimate the latent state~\citep{liu2020video, cui2021tf, he2021interpretable, fan2023one}.
Similarly, in the case of medical treatment, the state of the patient is also latent, and doctors rely on devices such as MRI to generate images that provide information about the patient's underlying condition~\citep{chen2017science, benson2020nccn, bi2022detecting, li2023nonconvex,you2023rethinking}. 
These visual data can be processed using machine learning to produce meaningful signals for diagnosis and treatment of cancer~\citep{long2015fully, you2023implicit, you2023action++,zou2023object}.

In the following, we consider a concrete real-world example of applying POMDP models to sepsis treatment, which was first studied by \citet{tsoukalas2015data}. 
In such an example, the state, action, observation, and reward of the POMDP are given by the following:
\begin{itemize}
    \item State variable $S_h$ refers to the clinical state of the patient, e.g., sepsis, SIRS, Bacteremia, etc. 
    \item Observable variable $O_h$ refers to all the information one can read from a medical device, such as the heart rate, the respiratory rate, blood pressure, blood test result of infection, etc. 
    \item Action $A_h$ refers to certain treatment given to the patient. For example, each antibiotic combination can be considered as an action. As mentioned in \citet{tsoukalas2015data}, a total of 48 antibiotics have been included in the patient’s remedy. 
    \item Reward/cost values need to be provided empirically by physicians, based on the severity of each state of the patient. In the example of \citet{tsoukalas2015data}, the states and their corresponding rewards/costs include: Healthy (100,000), No SIRS (50,000), Probable Sepsis (PS, 5000), SIRS (-50), Bacteremia~(-10,000), etc. 
    \item Finally, a history trajectory is the record of antibiotic treatment received by the patient. The behavior policy is some treatment plans that have been applied to some patients to generate the dataset. 
\end{itemize}

When using reactive policies (Example \ref{example: reactive}), the negative control action variable ($Z_h$) is just the observation variable $O_{h-1}$ which reflects the patient’s clinical state at the last treatment time step, and the negative control outcome variable ($W_h$) is just the observation variable $O_h$ at the current time step. Furthermore, when the observation $O$ contains enough information to reflect the underlying state $S$, which basically implies a certain full rank assumption, we can then use Example \ref{exampe: bridge exist} to guarantee the existence of the bridge functions (See Appendix \ref{sec: proximal causal inference}).

\section{Proximal Causal Inference}
\label{sec: proximal causal inference}

In this Section, we complement the discussion of proximal causal inference in Section \ref{subsec: identification}.

\subsection{Illustration of Examples}\label{subsec: illustration plot}

In this subsection, we give detailed discussions for the three examples of history-dependent policies mentioned in Section \ref{subsec: goal}.
In particular, we give causal graphs of the POMDP when adopting these policies. Also, we explain the choice of negative control variables for these policies in Section \ref{subsec: identification}.

% \yf{3 examples revisited: 1. refer to the figure and explain; 2. explain why the independence assumption is satisfied.\\ also explain in the caption of the figure what these arrows, dotted line and actual line}

\subsubsection{Reactive Policy (Example \ref{example: reactive} and Example \ref{example: reactive negative} revisited)}
When the target policy is reactive policy, it only depends on the current observation $O_h$. That is, $\mathcal{H}_{h-1}=\{\emptyset\}$ and $\Gamma_{h-1}=\emptyset$ for each $h\in[H]$.
The causal graph for such  a target policy  is shown in Figure \ref{fig: reactive}. 
In this case, we choose the negative control action as  $Z_h=O_{h-1}$ (node in \textcolor{green1}{\bf green}) and the negative control outcome as $W_h=O_h$ (node in \textcolor{yellow1}{\bf yellow}).
By this choice, we can check the independence condition in Assumption \ref{assump: negative control cond independence} via Figure \ref{fig: reactive}, i.e., under $\cP^{\pi}$,
\begin{align*}
    O_{h-1} \perp O_h, R_h, O_{h+1} \mid  S_h, A_h \quad O_h \perp A_h,S_{h-1} \mid S_h.
\end{align*}

\tikzset{global scale/.style={
    scale=#1,
    every node/.append style={scale=#1}
  }
}

\begin{figure}[h]
    \centering
    \begin{tikzpicture}[->,>=stealth', very thick, main node/.style={circle,draw}, global scale = 0.60]
    
       \node[main node, text=black, minimum width =35pt, 
minimum height =35pt] (-3) at  (-1.5,-1.5) {$R_{h-1}$};
    \node[main node, text=black, minimum width =35pt ,
minimum height =35pt] (-2) at  (1.5,-1.5) {$R_{h}$};

    \node[main node, text=black, minimum width =35pt, 
minimum height =35pt, style=dotted] (1) at  (-3,0) {$S_{h-1}$};
    \node[main node, text=black, minimum width =35pt ,
minimum height =35pt, style=dotted] (2) at  (0,0) {$S_{h}$};
\node[main node, text=black, minimum width =35pt ,
minimum height =35pt, style=dotted] (3) at  (3,0) {$S_{h+1}$};

\node[main node, text=black, minimum width =35pt ,
minimum height =35pt, fill=green1] (4) at  (-3,2) {$O_{h-1}$};
\node[main node, text=black, minimum width =35pt ,
minimum height =35pt, fill=yellow1] (5) at  (0,2) {$O_{h}$};
\node[main node, text=black, minimum width =35pt ,
minimum height =35pt] (6) at  (3,2) {$O_{h+1}$};

\node[main node, text=black, minimum width =35pt ,
minimum height =35pt] (7) at  (-1.5,3.5) {$A_{h-1}$};
\node[main node, text=black, minimum width =35pt ,
minimum height =35pt ] (8) at  (1.5,3.5) {$A_{h}$};
% \node[main node, text=black, minimum width =35pt ,
% minimum height =35pt] (9) at  (4,3.5) {$A_{h+1}$};

    \draw[->, style] (1) --(-3);
    \draw[->, style] (2) --(-2);

    \draw[->, style] (-4.5,0) --(1);
    \draw[->, style] (1) --(2);
    \draw[->, style] (2) --(3);
    \draw[->, style] (3) --(4.5,0);
    
    \draw[->, style] (1) --(4);
    \draw[->, style] (2) --(5);
    \draw[->, style] (3) --(6);
    
    \draw[->] (4) --(7);
    \draw[->, red1] (5) --(8);
    %\draw[->] (6) --(9);
    
    \draw[->] (7) --(-3);
    \draw[->] (8) --(-2);
    
    \draw[->, style] (7) --(2);
    \draw[->, style] (8) --(3);
    %\draw[->, style=dashed] (9) --(3);
    
     \draw[->, style, draw = blue1] (1) --(7);
    \draw[->, style, draw = blue1] (2) --(8);

    \end{tikzpicture}
    \caption{Causal graph for reactive policy. 
    The dotted nodes indicate that the variables are not stored in the offline dataset. 
    Solid arrows indicate the dependency among  the variables. Specifically,
    The \textcolor{red1}{\bf red} arrows depict the dependence of the target policy on the observable variables. The \textcolor{blue1}{\bf blue} arrows depict the dependence of the behavior policy on the latent state. The negative control action and outcome variables at the $h$-th step are filled in \textcolor{green1}{\bf green} and \textcolor{yellow1}{\bf yellow}, respectively.}
    \label{fig: reactive}
\end{figure}
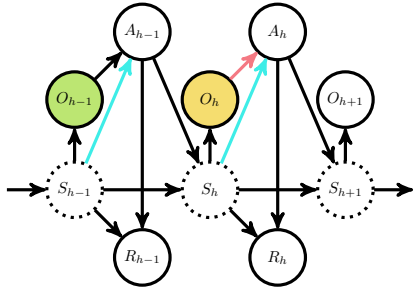

\subsubsection{Finite-history Policy (Example \ref{example: finite history} and Example \ref{example: finite history negative} revisited)}

When the target policy is finite-history policy, it depends on the current observation and history of length at most $k$. That is, $\mathcal{H}_{h-1}=(\mathcal{O}\times\mathcal{A})^{\otimes\min\{k, h-1\}}$ for some $k\in\mathbb{N}$,
$\Gamma_{h-1}=((O_l,A_l),\cdots,(O_{h-1},A_{h-1}))$ where the index $l = \max\{1, h-k\}$.
The causal graph for such a target policy  is shown in Figure \ref{fig: finite history}. 
In this case, we choose the negative control action as  $Z_h=O_{l-1}$ (node in \textcolor{green1}{\bf green}) and the negative control outcome as $W_h=O_h$ (node in \textcolor{yellow1}{\bf yellow}).
By this choice, we can check the independence condition in Assumption \ref{assump: negative control cond independence} via Figure \ref{fig: finite history}, i.e., under $\cP^{\pi}$,
\begin{align*}
    O_{l-1} &\perp O_h, R_h, O_{h+1} \mid S_h,A_h,O_{h-1},A_{h-1},\cdots,O_l,A_l, \\
    O_h &\perp A_h,S_{h-1},O_{h-1},A_{h-1},\cdots,O_l,A_l \mid S_h.
\end{align*}

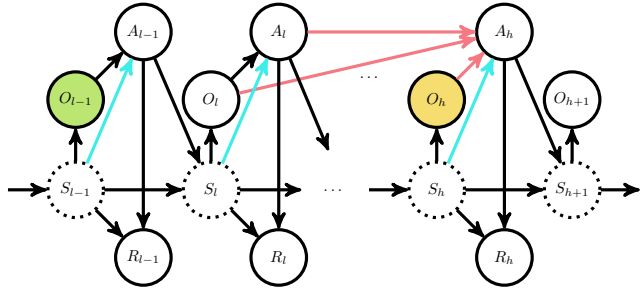
\begin{figure}[h]
    \centering
    \begin{tikzpicture}[->,>=stealth', very thick, main node/.style={circle,draw}, global scale = 0.6]
    
    \node[main node, text=black, minimum width =35pt, 
minimum height =35pt] (-4) at  (-6.5,-1.5) {$R_{l-1}$};
    \node[main node, text=black, minimum width =35pt, 
minimum height =35pt] (-3) at  (-3.5,-1.5) {$R_{l}$};
    \node[main node, text=black, minimum width =35pt ,
minimum height =35pt] (-2) at  (1.5,-1.5) {$R_{h}$};

    \node[main node, text=black, minimum width =35pt, 
minimum height =35pt, style=dotted] (1) at  (-8,0) {$S_{l-1}$};
    \node[main node, text=black, minimum width =35pt, 
minimum height =35pt, style=dotted] (2) at  (-5,0) {$S_{l}$};
    \node[main node, text=black, minimum width =35pt ,
minimum height =35pt, style=dotted] (3) at  (0,0) {$S_{h}$};
\node[main node, text=black, minimum width =35pt ,
minimum height =35pt, style=dotted] (4) at  (3,0) {$S_{h+1}$};

\node[main node, text=black, minimum width =35pt ,
minimum height =35pt, fill=green1] (5) at  (-8,2) {$O_{l-1}$};
\node[main node, text=black, minimum width =35pt ,
minimum height =35pt ] (6) at  (-5,2) {$O_{l}$};
\node[main node, text=black, minimum width =35pt ,
minimum height =35pt,  fill=yellow1] (7) at  (0,2) {$O_{h}$};
\node[main node, text=black, minimum width =35pt ,
minimum height =35pt] (8) at  (3,2) {$O_{h+1}$};

\node[main node, text=black, minimum width =35pt ,
minimum height =35pt] (9) at  (-6.5,3.5) {$A_{l-1}$};
\node[main node, text=black, minimum width =35pt ,
minimum height =35pt ] (10) at  (-3.5,3.5) {$A_{l}$};
\node[main node, text=black, minimum width =35pt ,
minimum height =35pt ] (11) at  (1.5,3.5) {$A_{h}$};
% \node[main node, text=black, minimum width =35pt ,
% minimum height =35pt] (9) at  (4,3.5) {$A_{h+1}$};

     \draw[->, style] (1) --(-4);
    \draw[->, style] (2) --(-3);
    \draw[->, style] (3) --(-2);

    \draw[->, style] (-9.5,0) --(1);
    \draw[->, style] (1) --(2);
    \draw[->](2) --(-3.0,0);
    \draw[style] (-1.5,0) --(3);
    \draw[->, style] (3) --(4);
    \draw[->, style] (4) --(4.5, 0);
    \node[text=black] at  (-2.3,0) {$\ldots$};
    \node[text=black] at  (-1.5,2.5) {$\ldots$};
    
    \draw[->, style] (1) --(5);
    \draw[->, style] (2) --(6);
    \draw[->, style] (3) --(7);
    \draw[->, style] (4) --(8);
    
    \draw[->] (5) --(9);
    \draw[->] (6) --(10);
    \draw[->,draw=red1 ] (7) --(11);
    \draw[->,draw=red1 ] (6) --(11);
    \draw[->,draw=red1 ] (10) --(11);
    %\draw[->] (6) --(9);
    
    \draw[->, style] (9) --(2);
    \draw[->, style] (10) --(-2.4,0.875);
    \draw[->, style] (11) --(4);
    %\draw[->, style=dashed] (9) --(3);
    
    \draw[->, style, draw = blue1] (1) --(9);
    \draw[->, style, draw = blue1] (2) --(10);
    \draw[->, style, draw = blue1] (3) --(11);
    
    \draw[->, style] (9) --(-4);
    \draw[->, style] (10) --(-3);
    \draw[->, style] (11) --(-2);
    % \node[draw,trapezium,minimum  width=3cm] (n0) at (0,0);

    \end{tikzpicture}
    \caption{Causal graph for finite-length history policy. Index $l=\max\{1,h-k\}$.
    The dotted nodes indicate that the variables are not stored in the offline dataset. 
    Solid arrows indicate the dependency among the variables.
    Specifically,
    The \textcolor{red1}{\bf red} arrows depict the dependence of the target policy on the observable variables. The \textcolor{blue1}{\bf blue} arrows depict the dependence of the behavior policy on the latent state. The negative control action and outcome variables at step $h$ are filled in \textcolor{green1}{\bf green} and \textcolor{yellow1}{\bf yellow}, respectively.
    }
    \label{fig: finite history}
\end{figure}

\subsubsection{Full-history Policy (Example \ref{example: full history} and Example \ref{example: full history negative} revisited)}

When the target policy is full-history policy, it depends on the current observation and full history information. That is, $\mathcal{H}_{h-1}=(\mathcal{O}\times\mathcal{A})^{\otimes(h-1)}$ and 
$\Gamma_{h-1}=((O_1,A_1),\cdots,(O_{h-1},A_{h-1}))$.
The causal graph for such a target policy  is shown in Figure \ref{fig: full history}. 
In this case, we choose the negative control action  as $Z_h=O_{0}$ (node in \textcolor{green1}{\bf green}) and the negative control outcome as $W_h=O_h$ (node in \textcolor{yellow1}{\bf yellow}).
By this choice, we can check the independence condition in Assumption \ref{assump: negative control cond independence} via Figure \ref{fig: full history}, i.e., under $\cP^{\pi}$,
\begin{align*}
    O_{0} &\perp O_h, R_h, O_{h+1} \mid S_h,A_h,O_{h-1},A_{h-1},\cdots,O_1,A_1, \\
    O_h &\perp A_h,S_{h-1},O_{h-1},A_{h-1},\cdots,O_1,A_1 \mid S_h.
\end{align*}
\begin{figure}[h]
    \centering
    \begin{tikzpicture}[->,>=stealth', very thick, main node/.style={circle,draw}, global scale=0.60]
    
    \node[main node, text=black, minimum width =35pt, 
minimum height =35pt] (-3) at  (-3.5,-1.5) {$R_1$};
    \node[main node, text=black, minimum width =35pt, 
minimum height =35pt] (-2) at  (1.5,-1.5) {$R_{h}$};

    \node[main node, text=black, minimum width =35pt, 
minimum height =35pt,  fill=green1] (1) at  (-7,0) {$O_0$};
    \node[main node, text=black, minimum width =35pt, 
minimum height =35pt, style=dotted] (2) at  (-5,0) {$S_{1}$};
    \node[main node, text=black, minimum width =35pt ,
minimum height =35pt, style=dotted] (3) at  (0,0) {$S_{h}$};
\node[main node, text=black, minimum width =35pt ,
minimum height =35pt, style=dotted] (4) at  (3,0) {$S_{h+1}$};

\node[main node, text=black, minimum width =35pt ,
minimum height =35pt ] (6) at  (-5,2) {$O_{1}$};
\node[main node, text=black, minimum width =35pt ,
minimum height =35pt,  fill=yellow1] (7) at  (0,2) {$O_{h}$};
\node[main node, text=black, minimum width =35pt ,
minimum height =35pt] (8) at  (3,2) {$O_{h+1}$};

\node[main node, text=black, minimum width =35pt ,
minimum height =35pt ] (10) at  (-3.5,3.5) {$A_{1}$};
\node[main node, text=black, minimum width =35pt ,
minimum height =35pt ] (11) at  (1.5,3.5) {$A_{h}$};
% \node[main node, text=black, minimum width =35pt ,
% minimum height =35pt] (9) at  (4,3.5) {$A_{h+1}$};

    \draw[->, style] (1) --(2);
    \draw[->](2) --(-3.0,0);
    \draw[style] (-1.5,0) --(3);
    \draw[->, style] (3) --(4);
    \draw[->, style] (4) --(4.5,0);
    \node[text=black] at  (-2.3,0) {$\ldots$};
    \node[text=black] at  (-1.5,2.5) {$\ldots$};
    
    \draw[->, style] (2) --(6);
    \draw[->, style] (3) --(7);
    \draw[->, style] (4) --(8);
    
    \draw[->, style] (2) --(-3);
    \draw[->, style] (3) --(-2);
    
     \draw[->, style] (10) --(-3);
    \draw[->, style] (11) --(-2);

    \draw[->] (6) --(10);
    \draw[->,draw=red1 ] (7) --(11);
    \draw[->,draw=red1 ] (6) --(11);
    \draw[->,draw=red1 ] (10) --(11);

    \draw[->, style] (10) --(-2.4,0.875);
    \draw[->, style] (11) --(4);
    %\draw[->, style=dashed] (9) --(3);

    \draw[->, style, draw = blue1] (2) --(10);
    \draw[->, style, draw = blue1] (3) --(11);
    
    % \node[draw,trapezium,minimum  width=3cm] (n0) at (0,0);

    \end{tikzpicture}
    \caption{Causal graph for full-length history policy.
    The dotted nodes indicate that the variables are not stored in the offline dataset. 
    Solid arrows indicate the dependency among  the variables.
    Specifically,
    The \textcolor{red1}{\bf red} arrows depict the dependence of the target policy on the observable variables. The \textcolor{blue1}{\bf blue} arrows depict the dependence of the behavior policy on the latent state. The negative control action and outcome variables at step $h$ are filled in \textcolor{green1}{\bf green} and \textcolor{yellow1}{\bf yellow}, respectively.}
    \label{fig: full history}
\end{figure}
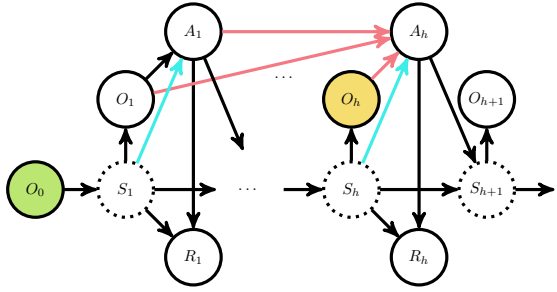

\subsection{Proofs for Examples of Assumption~\ref{assump: bridge functions exist}}\label{subsec: proximal causal inference proof}

\begin{proof}[Proof of Example \ref{exampe: bridge exist}]
Recall that for reactive policies, the history information $\Gamma_{h-1}=\emptyset$. We first show that under condition \eqref{eq: example rank condition}, there exist functions $\{b^{\pi}_h\}_{h=1}^H$ and $\{q^{\pi}_h\}_{h=1}^H$ which solve the following equations 
\begin{align}
    &\EE_{\pi^b}\left[b^{\pi}_{h}(A_h, O_h)\middle| A_h, S_h\right] \notag\\
    &\qquad= \EE_{\pi^b}\Big[ R_h \pi_h(A_h|O_h) + \gamma \sum_{a'}b^{\pi}_{h+1}(a', O_{h+1}) \pi_h(A_h|O_h )\Big| A_h, S_h \Big],\label{eq: def unlearnable value bridge}
\\
    & \EE_{\pi^b}\left[ q^{\pi}_h(A_h, O_{h-1}) \middle| A_h, S_h  \right] = \frac{\mu_h(S_h)}{\pi^b_h(A_h|S_h)}, \label{eq: def unlearnable weight bridge}
\end{align}
Then we can show that the solutions to \eqref{eq: def unlearnable value bridge} and \eqref{eq: def unlearnable weight bridge} also solve \eqref{eq: def value bridge} and \eqref{eq: def weight bridge}.
The difference between \eqref{eq: def unlearnable value bridge} and \eqref{eq: def value bridge} is that in \eqref{eq: def unlearnable value bridge} we condition on the latent state $S_h$ rather than the observable negative control variable $Z_h$. 
In related literature \citep{bennett2021proximal, shi2021minimax}, the solutions to \eqref{eq: def unlearnable value bridge} and \eqref{eq: def unlearnable weight bridge} are referred to as \emph{unlearnable bridge functions}.

We first show the existence of $\{b_h^{\pi}\}_{h=1}^H$ in a backward manner.
Denote by $b_{h+1}^{\pi}$ a zero function.
Suppose that $b_{h+1}^{\pi}$ exists, we show that $b_{h}^{\pi}$ also exists.
Since now spaces $\mathcal{S}$, $\mathcal{A}$, and $\mathcal{O}$ are discrete, we adopt the notation of matrix.
In particular, we denote by 
\begin{align*}
    \mathbf{B}\in\mathbb{R}^{|\mathcal{A}|\times|\mathcal{O}|},&\qquad \mathbf{B}(a,o)=b_h(a,o),\\
    \mathbf{O}\in\mathbb{R}^{|\mathcal{S}|\times|\mathcal{O}|},&\qquad \mathbf{O}(s,o)=\mathcal{P}_h^b(O_h=o|S_h=s),
\end{align*}
\begin{align*}
    \mathbf{R}\in\mathbb{R}^{|\mathcal{A}|\times|\mathcal{S}|},\ \mathbf{R}(s,a)=\EE_{\pi^b}\Big[ R_h \pi_h(A_h|O_h) + \gamma \sum_{a'}b^{\pi}_{h+1}(a', O_{h+1}) \pi_h(A_h|O_h )\Big| A_h=a, S_h=s \Big].
\end{align*}
The existence of $b_h^{\pi}$ satisfying \eqref{eq: def unlearnable value bridge} is equivalent to the existence of $\mathbf{B}$ solving the matrix equation
\begin{align}\label{eq: matrix equation 1}
    \mathbf{B}\,\mathbf{O}^\top=\mathbf{R}.
\end{align}
By condition \eqref{eq: example rank condition}, we known that the matrix $\mathbf{O}^\top$ is of full column rank, which indicates that \eqref{eq: matrix equation 1} admits a solution $\mathbf{B}$. This proves the existence of $b_h^{\pi}$.
For $\{q_h^{\pi}\}_{h=1}^H$, we use a similar method by considering 
\begin{align*}
    \mathbf{Q}\in\mathbb{R}^{|\mathcal{A}|\times|\mathcal{O}|}, &\qquad \mathbf{Q}(a,o)=q_h(a,o),\\
    \mathbf{O}_{-}\in\mathbb{R}^{|\mathcal{S}|\times|\mathcal{O}|},& \qquad \mathbf{O}_-(s,o)=\mathcal{P}_h^b(O_{h-1}=o|S_h=s),\\
    \mathbf{C}\in\mathbb{R}^{|\mathcal{A}|\times|\mathcal{S}|},&\qquad \mathbf{C}(s,a)=\frac{\mu_h(S_h=s)}{\pi^b_h(A_h=a|S_h=s)}.
\end{align*}
The existence of $q_h^{\pi}$ satisfying \eqref{eq: def unlearnable weight bridge} is equivalent to the existence of $\mathbf{Q}$ solving the matrix equation
\begin{align}\label{eq: matrix equation 2}
    \mathbf{Q}\,\mathbf{O}_-^\top=\mathbf{C}
\end{align}
By condition \eqref{eq: example rank condition}, we known that the matrix $\mathbf{O}_-^\top$ is of full column rank, which indicates that \eqref{eq: matrix equation 2} admits a solution $\mathbf{Q}$. This proves the existence of $q_h^{\pi}$. Thus we have shown that there exists $\{b_h^{\pi}\}_{h=1}^H$ and $\{q_h^{\pi}\}_{h=1}^H$ which solve equation \eqref{eq: def unlearnable value bridge} and \eqref{eq: def unlearnable weight bridge}.
Finally, it holds that any solution to \eqref{eq: def unlearnable value bridge} and \eqref{eq: def unlearnable weight bridge} also forms a solution to \eqref{eq: def value bridge} and \eqref{eq: def weight bridge}, which has been shown in Theorem 11 in \cite{shi2021minimax}.
This finishes the proof.
\end{proof}

\begin{proof}[Proof of Example \ref{exampe: bridge exist 2}]
    To see this, we first prove the existence of $\left\{b_n^\pi\right\}$.
    For simplicity, we denote by
    $$
    \mathbf{P}_a=\left(\mathcal{P}_h^b\left(O_h \mid A_h=a, O_{h-k-1}\right)\right) \in \mathbb{R}^{|\mathcal{O}| \times|\mathcal{O}|}
    $$
    for each $a \in \mathcal{A}$. 
    Also, we denote that
    \begin{align*}
        \mathbf{B}_a&=\left(b_h\left(a, O_h\right)\right) \in \mathbb{R}^{| \mathcal{O}| \times 1},\\ \mathbf{R}_a&=\bigg(\mathbb{E}_{\pi^b}\bigg[R_h \pi_h\left(A_h \mid O_h\right)+\gamma \sum_{a^{\prime}} b_{h+1}^\pi\left(a^{\prime}, O_{h+1}\right) \pi_h\left(A_h \mid O_h\right) \mid A_h=a, O_{h-k-1}\bigg]\bigg) \in \mathbb{R}^{|\mathcal{O}| \times 1}.
    \end{align*}    
    Then for each $a \in \mathcal{A}$, the existence of $b_n^\pi(a, \cdot)$ is equivalent to the existence of the solution to 
    \begin{align*}
        \mathbf{P}^a\mathbf{B}^a = \mathbf{R}^a.
    \end{align*}
    Such a linear equation admits a solution due to our assumption on the matrix $\mathbf{P}_a$. This shows the existence of $\left\{b_h^\pi\right\}$.
    For $\left\{q_h^\pi\right\}$, the deduction is similar by considering for each $a \in \mathcal{A}$,
    \begin{align*}
        \mathbf{T}_a&=\left(\mathcal{P}^b_h\left(O_{h-k-1} \mid A_h=a, S_h, \Gamma_{h-1}\right) \right)\in \mathbb{R}^{|\mathcal{S}|\left|\mathcal{H}_{h-1}\right| \times|\mathcal{O}|},\\
        \mathbf{Q}_a&=\left(q_h\left(a, O_{h-k-1}\right)\right) \in \mathbb{R}^{|\mathcal{O}| \times 1}, \\
        \mathbf{C}_a&=\left(\frac{\mu_h\left(S_h, \Gamma_{h-1}\right)}{\pi^b\left(a \mid S_h\right)}\right) \in \mathbb{R}^{|\mathcal{S}| |\mathcal{H}_{h-1} | \times 1}.
    \end{align*}
    By considering the equation that
    \begin{align*}
        \mathbf{T}_a \mathbf{Q}_a=\mathbf{C}_a
    \end{align*}    
    and using the full rank assumption on matrix $\mathbf{T}_a$, we can obtain the existence of $\left\{q_h^\pi\right\}$.
    This finishes the proof of Example \ref{exampe: bridge exist 2}.
\end{proof}

\section{Proof of Theorem \ref{thm: identification}}\label{sec: proof identification}

\begin{proof}[Proof of Theorem \ref{thm: identification}]
For any step $h$, we denote $J_h(\pi)\coloneqq\mathbb{E}_{\pi}[R_h(S_h,A_h)]$. We have that
\begin{align*}
    J_h(\pi)&=\mathbb{E}_{\pi}[R_h(S_h,A_h)]\notag\\
    &=\mathbb{E}_{\pi}\bigl [\mathbb{E}_{\pi}[R_h(S_h,A_h)|O_h,S_h,\Gamma_{h-1}] \bigr ]\notag\\
    &=\mathbb{E}_{\pi}\left[\sum_{a\in\mathcal{A}}R_h(S_h,a)\pi_h(a|O_h,\Gamma_{h-1})\right]\notag\\
    &=\mathbb{E}_{\pi}\left[\mathbb{E}_{\pi}\left[\sum_{a\in\mathcal{A}}R_h(S_h,a)\pi_h(a|O_h,\Gamma_{h-1}) \middle| S_h,\Gamma_{h-1}\right]\right],\notag
\end{align*}
where the second and the last equality follows from the tower property of conditional expectation.
Using the definition of density ratio $\mu_h(S_h,\Gamma_{h-1})$ in Assumption \ref{assump: bridge functions exist}, we can change the outer expectation to $\mathbb{E}_{\pi^b}$ by 
\begin{align*}
    J_h(\pi)&=\mathbb{E}_{\pi^b}\left[\mu_h(S_h,\Gamma_{h-1})\cdot \mathbb{E}_{\pi}\left[\sum_{a\in\mathcal{A}}R_h(S_h,a)\pi_h(a|O_h,\Gamma_{h-1}) \middle| S_h,\Gamma_{h-1}\right]\right],\notag\\
    &=\mathbb{E}_{\pi^b}\left[\sum_{a\in\mathcal{A}}R_h(S_h,a)\cdot \pi_h(a|O_h,\Gamma_{h-1})\cdot \mu_h(S_h,\Gamma_{h-1})\right]\notag\\
    &=\mathbb{E}_{\pi^b}\left[\sum_{a\in\mathcal{A}}\pi_h^b(a|S_h) \cdot R_h(S_h,a)\cdot \frac{\pi_h(a|O_h,\Gamma_{h-1})}{\pi_h^b(a|S_h)}\cdot \mu_h(S_h,\Gamma_{h-1})\right]\notag\\
    &\eqw{(a)}\mathbb{E}_{\pi^b}  \left[\mathbb{E}_{\pi^b}\left[R_h(S_h,A_h) \cdot \frac{\pi_h(A_h|O_h,\Gamma_{h-1})}{\pi_h^b(A_h|S_h)}\cdot\mu_h(S_h,\Gamma_{h-1})\middle| S_h,O_h,\Gamma_{h-1}\right]\right]\notag\\
    &=\mathbb{E}_{\pi^b}\left[R_h(S_h,A_h) \cdot \pi_h(A_h|O_h,\Gamma_{h-1}) \cdot\frac{\mu_h(S_h,\Gamma_{h-1})}{\pi_h^b(A_h|S_h)}\right],\notag
\end{align*}
where step (a) follows from the fact that $A_h\sim\pi_h^b(\cdot|S_h)$ and satisfies $A_h\perp O_h,\Gamma_{h-1}|S_h$ under $\pi^b$.
Now using the definition \eqref{eq: def weight bridge} of weight bridge function $q_h^{\pi}$ in Assumption \ref{assump: bridge functions exist}, we have that 
\begin{align*}
    J_h(\pi)&=\mathbb{E}_{\pi^b}\bigl[R_h(S_h,A_h) \cdot \pi_h(A_h|O_h,\Gamma_{h-1})\cdot \mathbb{E}_{\pi^b}\left[q_h^{\pi}(A_h,Z_h)\middle| S_h,A_h,\Gamma_{h-1}\right]\bigr ]\notag\\
    &\eqw{(a)}\mathbb{E}_{\pi^b}\left[R_h(S_h,A_h)\cdot  \pi_h(A_h|O_h,\Gamma_{h-1}) \cdot  q_h^{\pi}(A_h,Z_h)\right]\notag\\
    &=\mathbb{E}_{\pi^b}\left[\mathbb{E}_{\pi^b}\left[R_h(S_h,A_h) \cdot  \pi_h(A_h|O_h,\Gamma_{h-1})\cdot  q_h^{\pi}(A_h,Z_h)\middle|  A_h,Z_h\right]\right]\notag\\
    &=\mathbb{E}_{\pi^b}\left[\mathbb{E}_{\pi^b}\left[R_h(S_h,A_h) \cdot  \pi_h(A_h|O_h,\Gamma_{h-1}) \cdot \middle|  A_h,Z_h\right]q_h^{\pi}(A_h,Z_h)\right],\notag
\end{align*}
where step (a) follows from the assumption that $Z_h\perp O_h,R_h|S_h,A_h,\Gamma_{h-1}$ by Assumption \ref{assump: negative control cond independence}. 
Now using the definition \eqref{eq: def value bridge} of value bridge function $b_h^{\pi}$ in Assumption \ref{assump: bridge functions exist}, we have that 
\begin{align*}
    J_h(\pi)&=\mathbb{E}_{\pi^b}\left[\mathbb{E}_{\pi^b}\left[
    b^{\pi}_h(A_h,W_h)-\gamma\sum_{a^{\prime}\in\mathcal{A}}b_{h+1}^{\pi}(a^{\prime},W_{h+1})\pi_h(A_h|O_h,\Gamma_{h-1})
    \middle|  A_h,Z_h\right]q_h^{\pi}(A_h,Z_h)\right]\notag\\
    &=\mathbb{E}_{\pi^b}\left[
    f(S_h,A_h,O_h,W_h,W_{h+1},\Gamma_{h-1})\cdot  
   q_h^{\pi}(A_h,Z_h)\right]\notag\\
   &=\mathbb{E}_{\pi^b}\left[
    \mathbb{E}_{\pi^b}\left[f(S_h,A_h,O_h,W_h,W_{h+1},\Gamma_{h-1})\cdot  
   q_h^{\pi}(A_h,Z_h)\middle| S_h,A_h,O_h,W_h,W_{h+1},\Gamma_{h-1}\right]\right],\notag\\
   &=\mathbb{E}_{\pi^b}\left[
    f(S_h,A_h,O_h,W_h,W_{h+1},\Gamma_{h-1})\cdot\mathbb{E}_{\pi^b}\left[
   q_h^{\pi}(A_h,Z_h)\middle| S_h,A_h,O_h,W_h,W_{h+1},\Gamma_{h-1}\right]\right],\notag\\
   &\eqw{(a)}\mathbb{E}_{\pi^b}\left[
    f(S_h,A_h,O_h,W_h,W_{h+1},\Gamma_{h-1})\cdot\mathbb{E}_{\pi^b}\left[
   q_h^{\pi}(A_h,Z_h)\middle| S_h,A_h,\Gamma_{h-1}\right]\right],\notag
\end{align*}
where for simplicity we have denoted that
\begin{align*}
    f(S_h,A_h,O_h,W_h,W_{h+1},\Gamma_{h-1})=b^{\pi}_h(A_h,W_h)-\gamma\sum_{a^{\prime}\in\mathcal{A}}b_{h+1}^{\pi}(a^{\prime},W_{h+1})\pi_h(A_h|O_h,\Gamma_{h-1}),\notag
\end{align*}
and step (a) follows from
the assumption that $Z_h
\perp O_h,W_h,W_{h+1}|S_h,A_h,\Gamma_{h-1}$ by Assumption \ref{assump: negative control cond independence}.
By the definition \eqref{eq: def weight bridge} of weight bridge function $q^{\pi}_h$ in Assumption \ref{assump: bridge functions exist} again, we have that 
\begin{align*}
    J_h(\pi)&=\mathbb{E}_{\pi^b}\left[
    f(S_h,A_h,O_h,W_h,W_{h+1},\Gamma_{h-1})\cdot\frac{\mu_h(S_h,\Gamma_{h-1})}{\pi_h^b(A_h|S_h)}\right]\notag\\
    &\eqw{(a)} \mathbb{E}_{\pi^b}\left[\left(b^{\pi}_h(A_h,W_h)-\gamma\sum_{a^{\prime}\in\mathcal{A}}b_{h+1}^{\pi}(a^{\prime},W_{h+1})\pi_h(A_h|O_h,\Gamma_{h-1})\right)\cdot\frac{\mu_h(S_h,\Gamma_{h-1})}{\pi_h^b(A_h|S_h)}\right],\notag
\end{align*}
where step (a) just applies the definition of $f$.
Now sum $J_h(\pi)$ over $h\in[H]$, we have that \begin{align}\label{eq: thm 2.6 proof 1}
    J(\pi)&=\sum_{h=1}^H\gamma^{h-1}J_h(\pi)=\underbrace{\mathbb{E}_{\pi^b}\left[\frac{\mu_1(S_1,\Gamma_{0})}{\pi_1^b(A_1|S_1)}b_1^{\pi}(A_1,W_1)\right]}_{\displaystyle{(\mathrm{A})}} +\underbrace{\sum_{h=2}^H\gamma^{h-1}\Delta_h}_{\displaystyle{(\mathrm{B})}},
\end{align}
where for simplicity we define $\Delta_h$ for $h=2,\cdots,H$ as
\begin{align*}
    \Delta_h=\mathbb{E}_{\pi^b}\left[
    \frac{\mu_h(S_h,\Gamma_{h-1})}{\pi_h^b(A_h|S_h)}b^{\pi}_h(A_h,W_h)-\frac{\mu_{h-1}(S_{h-1},\Gamma_{h-2})}{\pi_{h-1}^b(A_{h-1}|S_{h-1})}\cdot \sum_{a^{\prime}\in\mathcal{A}}b_h^{\pi}(a^{\prime},W_h)\pi_{h-1}(A_{h-1}|O_{h-1},\Gamma_{h-1})
    \right].\notag
\end{align*}
In the sequel, we deal with term (A) and (B) respectively.
On the one hand, we have that 
\begin{align*}
    \textnormal{(A)}&\eqw{(a)}\mathbb{E}_{\pi^b}\left[\frac{\mathcal{P}_1^{\pi}(S_1,\Gamma_{0})}{\mathcal{P}^b_1(S_1,\Gamma_{0})\pi^b_1(A_1|S_1)}b_1^{\pi}(A_1,W_1)\right]\notag\\
    &\eqw{(b)}\mathbb{E}_{\pi^b}\left[\frac{1}{\pi^b_1(A_1|S_1)}b_1^{\pi}(A_1,W_1)\right]\notag\\
    &=\mathbb{E}_{\pi^b}\left[\mathbb{E}_{\pi^b}\left[\frac{1}{\pi^b_1(A_1|S_1)}b_1^{\pi}(A_1,W_1)\middle| S_1,W_1\right]\right]\notag\\
    &\eqw{(c)}\mathbb{E}_{\pi^b}\left[\sum_{a\in\mathcal{A}}\frac{\pi^b_1(a|S_1)}{\pi^b_1(a|S_1)}b_1^{\pi}(a,W_1)\right]\notag\\
    &=\mathbb{E}_{\pi^b}\left[\sum_{a\in\mathcal{A}}b_1^{\pi}(a,W_1)\right],\notag
\end{align*}
where step (a) follows from the definition of $\mu_1(S_1,\Gamma_{0})$ in Assumption \ref{assump: bridge functions exist}, step (b) follows from the fact that at $h=1$, $\mathcal{P}^b_1(S_1,\Gamma_{0})=\mathcal{P}_1^{\pi}(S_1,\Gamma_{0})$, and step (c) follows from the assumption that $A_1\perp W_1|S_1$ by Assumption~\ref{assump: negative control cond independence}.
On the other hand, term (b) in \eqref{eq: thm 2.6 proof 1} is actually $0$, which we show by proving that $\Delta_h=0$ for any $h\geq 2$.
We denote by $\Delta_h=\Delta_h^1-\Delta_h^2$ and we consider $\Delta_h^1$ and $\Delta_h^2$ respectively, where 
\begin{align*}
    \Delta_h^1 &= \mathbb{E}_{\pi^b}\left[
    \frac{\mu_h(S_h,\Gamma_{h-1})}{\pi_h^b(A_h|S_h)} \cdot b^{\pi}_h(A_h,W_h)\right],\\
    \Delta_h^2 & = \mathbb{E}_{\pi^b}\left[\frac{\mu_{h-1}(S_{h-1},\Gamma_{h-2})}{\pi_{h-1}^b(A_{h-1}|S_{h-1})}\cdot \sum_{a^{\prime}\in\mathcal{A}}b_h^{\pi}(a^{\prime},W_h)\pi_{h-1}(A_{h-1}|O_{h-1},\Gamma_{h-1})
    \right].
\end{align*}
In the sequel, we prove that $\Delta_h^1=\Delta_h^2$ for the three cases of $T_h$ in Example
\ref{example: reactive}, \ref{example: finite history}, and \ref{example: full history}, respectively.

\vspace{2mm}
\noindent
\textbf{Case 1: Reactive policy (Example \ref{example: reactive}).}
We first focus on the simple case when policy $\pi$ is reactive.
 Since for reactive policies $T_h=\emptyset$, we can equivalently write $\mu_h(S_h,\Gamma_{h-1})$ as $\mu_h(S_h)=\mathcal{P}^{\pi}_h(S_h)/\mathcal{P}_h^{b}(S_h)$. 
 Now for $\Delta_h^1$, we can rewrite it as 
\begin{align*}
    \Delta_h^1&=\mathbb{E}_{\pi^b}\left[
    \frac{\mathcal{P}^{\pi}_h(S_h)}{\mathcal{P}_h^{b}(S_h)   \pi_h^b(A_h|S_h)}\cdot  b^{\pi}_h(A_h,W_h)\right]\notag\\
    &\eqw{(a)}\int_{\mathcal{S}}\cancel{\mathcal{P}_h^b(s_h)}\mathrm{d}s_h\sum_{a_h\in\mathcal{A}}\cancel{\pi_h^b(a_h|s_h)}\int_{\mathcal{W}}\textcolor{red}{\mathcal{P}_h^b(w_h|s_h,a_h)}\mathrm{d}w_h\cdot \frac{\mathcal{P}^{\pi}_h(s_h)}{\cancel{\mathcal{P}_h^{b}(s_h)}\cancel{\pi_h^b(a_h|s_h)}}b^{\pi}_h(a_h,w_h)\notag\\
    &\eqw{(b)}\sum_{a_h\in\mathcal{A}}\int_{\mathcal{S}}\mathcal{P}_h^{\pi}(s_h)\mathrm{d}s_h\int_{\mathcal{W}}\textcolor{red}{\mathcal{P}_h^b(w_h|s_h)}\mathrm{d}w_h\cdot b^{\pi}_h(a_h,w_h).\notag
\end{align*}
Here step (a) expands the expectation by using integral against corresponding density functions, and step (b) follows from cancelling the same terms and the fact that $W_h\perp A_h|S_h$ under Assumption \ref{assump: negative control cond independence}.
For $\Delta_h^2$, we can also rewrite it as
\begin{align*}
    \Delta_h^2 &= \mathbb{E}_{\pi^b}\left[\frac{\mathcal{P}^{\pi}_{h-1}(S_{h-1})\pi_{h-1}(A_{h-1}|O_{h-1})}{\mathcal{P}_{h-1}^b(S_{h-1})\pi_{h-1}^b(A_{h-1}|S_{h-1})}\cdot  \sum_{a^{\prime}\in\mathcal{A}}b_h^{\pi}(a^{\prime},W_h)
    \right]\notag\\
    &\eqw{(a)}\int_{\mathcal{S}}\cancel{\mathcal{P}^b_{h-1}(s_{h-1})}\mathrm{d}s_{h-1}\int_{\mathcal{O}}\mathbb{O}_{h-1}(o_{h-1}|s_{h-1})\mathrm{d}o_{h-1}\sum_{a_{h-1}\in\mathcal{A}}\cancel{\pi_{h-1}^b(a_{h-1}|s_{h-1})}\int_{\mathcal{S}}\mathbb{P}_h(s_h|s_{h-1},a_{h-1})\mathrm{d}s_h\notag\\
    &\qquad\int_{\mathcal{W}}\textcolor{red}{\mathcal{P}_h^b(w_h|s_h,s_{h-1},a_{h-1},o_{h-1})} \cdot\frac{\mathcal{P}^{\pi}_{h-1}(s_{h-1})\pi_{h-1}(a_{h-1}|o_{h-1})}{\cancel{\mathcal{P}_{h-1}^b(s_{h-1})}\cancel{\pi_{h-1}^b(a_{h-1}|s_{h-1})}}\sum_{a_h\in\mathcal{A}}b_h^{\pi}(a_h,w_h)\mathrm{d}w_h.\notag
\end{align*}
Here step (a) follows from expanding the expectation. It follows that
\begin{align*}
    \Delta_h^2 &\eqw{(b)}\sum_{a_h\in\mathcal{A}}\int_{\mathcal{S}}\mathcal{P}^{\pi}_{h-1}(s_{h-1})\mathrm{d}s_{h-1}\int_{\mathcal{O}}\mathbb{O}_{h-1}(o_{h-1}|s_{h-1})\mathrm{d}o_{h-1}\sum_{a\in\mathcal{A}}\pi_{h-1}(a_{h-1}|o_{h-1})\notag
    \\ &\qquad \int_{\mathcal{S}}\mathbb{P}_h(s_h|s_{h-1},a_{h-1})\mathrm{d}s^{\prime}\  \int_{\mathcal{W}}\textcolor{red}{\mathcal{P}_h^b(w_h|s_h)}\cdot b_h^{\pi}(a_h,w_h)\notag \\
    &\eqw{(c)}\sum_{a_h\in\mathcal{A}}\int_{\mathcal{S}}\mathcal{P}^{\pi}_{h}(s_h)\mathrm{d}s_h\int_{\mathcal{W}}\textcolor{red}{\mathcal{P}_h^b(w_h|s_h)}\cdot b_h^{\pi}(a_h,w_h)\mathrm{d}w_h.\notag
\end{align*}
Here step (b) follows from cancelling the same terms and using the fact that $W_h\perp S_{h-1},A_{h-1},O_{h-1}|S_h$ by Assumption \ref{assump: negative control cond independence}, and step (d) follows by marginalizing over $S_{h-1},A_{h-1}, O_{j-1}$.
Thus we have proved that $\Delta_h^1=\Delta_h^2$ for reactive policies and consequently $\Delta_h=\Delta_h^1-\Delta_h^2=0$.

\vspace{2mm}
\noindent
\textbf{Case 2: Finite-history policy (Example \ref{example: finite history}).}
Now we have that 
$\Gamma_{h-1}\cup\{A_h,O_h\}=\{A_{l-1},O_{l-1}\}\cup T_h$, where the index $l = \max\{0,h-k\}$.
Similarly, we can first rewrite $\Delta_h^1$ as 
\begin{align*}
    \Delta_h^1&=\mathbb{E}_{\pi^b}\left[
    \frac{\mathcal{P}^{\pi}_h(S_h,\Gamma_{h-1})}{\mathcal{P}_h^{b}(S_h,\Gamma_{h-1})\pi_h^b(A_h|S_h)}b^{\pi}_h(A_h,W_h)\right]\notag\\
    &\eqw{(a)}\int_{\mathcal{S}\times\mathcal{H}_{h-1}}\mathcal{P}_h^b(s_h,\tau_{h-1})\mathrm{d}s_h\mathrm{d}\tau_{h-1}\sum_{a_h\in\mathcal{A}}\pi_h^b(a_h|s_h)\int_{\mathcal{W}}\textcolor{red}{\mathcal{P}_h^b(w_h|s_h,a_h,\tau_{h-1})}\mathrm{d}w_h\notag\\
    &\qquad \cdot \frac{\mathcal{P}^{\pi}_h(s_h,\tau_{h-1})}{\mathcal{P}_h^{b}(s_h,\tau_{h-1})\pi_h^b(a_h|s_h,\tau_{h-1})}b^{\pi}_h(a_h,w_h)\notag\\
    &\eqw{(b)}\sum_{a_h\in\mathcal{A}}\int_{\mathcal{S}\times\mathcal{H}_{h-1}}\mathcal{P}_h^{\pi}(s_h,\tau_{h-1})\mathrm{d}s_h\mathrm{d}\tau_{h-1}\int_{\mathcal{W}}\textcolor{red}{\mathcal{P}_h^b(w_h|s_h)}\mathrm{d}w_h\cdot b^{\pi}_h(a_h,w_h).\notag
\end{align*}
Here step (a) follows from expanding the expectation, and step (b) follows from cancelling the same terms and using the fact that 
$W_h\perp A_h,\Gamma_{h-1}|S_h$ under Assumption \ref{assump: negative control cond independence}.
For $\Delta_h^2$, we can also rewrite it as
\begin{align}
    \Delta_h^2 &= \mathbb{E}_{\pi^b}\left[\frac{\mathcal{P}^{\pi}_{h-1}(S_{h-1},\Gamma_{h-2})\pi_{h-1}(A_{h-1}|O_{h-1})}{\mathcal{P}_{h-1}^b(S_{h-1},\Gamma_{h-2})\pi_{h-1}^b(A_{h-1}|S_{h-1},\Gamma_{h-2})}\sum_{a^{\prime}\in\mathcal{A}}b_h^{\pi}(a^{\prime},W_h)
    \right]\notag\\
    &\eqw{(a)}\int_{\mathcal{S}\times\mathcal{H}_{h-2}}\cancel{\mathcal{P}^b_{h-1}(s_{h-1},\tau_{h-2})\mathrm{d}s_{h-1}}\mathrm{d}\tau_{h-2}\int_{\mathcal{O}}\mathbb{O}_{h-1}(o_{h-1}|s_{h-1})\mathrm{d}o_{h-1}\sum_{a_{h-1}\in\mathcal{A}}\cancel{\pi_{h-1}^b(a_{h-1}|s_{h-1})}\notag\\
    &\qquad\int_{\mathcal{S}}\mathbb{P}_h(s_h|s_{h-1},a_{h-1})\mathrm{d}s_h\int_{\mathcal{W}}\textcolor{red}{\mathcal{P}_h^b(w_h|s_h,s_{h-1},a_{h-1},o_{h-1},\tau_{h-2})}\notag\\
    &\qquad\qquad \cdot\frac{\mathcal{P}^{\pi}_{h-1}(s_{h-1},\tau_{h-2})\pi_{h-1}(a_{h-1}|o_{h-1},\tau_{h-2})}{\cancel{\mathcal{P}_{h-1}^b(s_{h-1},\tau_{h-2})}\cancel{\pi_{h-1}^b(a_{h-1}|s_{h-1})}}\sum_{a_h\in\mathcal{A}}b_h^{\pi}(a_h,w_h)\notag\\
    &\eqw{(b)}\sum_{a_h\in\mathcal{A}}\int_{\mathcal{S}\times\mathcal{H}_{h-2}}\mathcal{P}^{\pi}_{h-1}(s_{h-1},\textcolor{blue}{\tilde{\tau}_{h-2}},a_l,o_l)\mathrm{d}s_{h-1}\mathrm{d}\tilde{\tau}_{h-2}\mathrm{d}a_l\mathrm{d}o_l\int_{\mathcal{O}}\mathbb{O}_{h-1}(\textcolor{blue}{o_{h-1}}|s_{h-1})\mathrm{d}o_{h-1}\notag\\
    &\qquad\sum_{a_{h-1}\in\mathcal{A}}\pi_{h-1}(\textcolor{blue}{a_{h-1}}|o_{h-1},\tau_{h-2})\int_{\mathcal{S}}\mathbb{P}_h(\textcolor{blue}{s_h}|s_{h-1},a_{h-1})\mathrm{d}s_h\int_{\mathcal{W}}\textcolor{red}{\mathcal{P}_h^b(w_h|s_h)} \cdot b_h^{\pi}(a_h,w_h)\label{eq: thm 2.6 proof 2}\\
    &\eqw{(c)}\sum_{a_h\in\mathcal{A}}\int_{\mathcal{S}\times\mathcal{H}_{h-1}}\mathcal{P}^{\pi}_{h}(\textcolor{blue}{s_{h},\tau_{h-1}})\mathrm{d}s_{h}\mathrm{d}\tau_{h-1}\int_{\mathcal{W}}\textcolor{red}{\mathcal{P}_h^b(w_h|s_h)} \cdot b_h^{\pi}(a_h,w_h),\notag
\end{align}
where the index $l=\max\{1, h-1-k\}$. In step (b), we have denoted by $\tilde{\tau}_{h-2}=\tau_{h-2}\setminus\{a_l,o_l\}$ and it holds that $\tau_{h-1}=\tilde{\tau}_{h-2}\cup\{o_{h-1},a_{h-1}\}$.
Here step (a) follows from expanding the expectation, step (b) follows from cancelling the same terms and using the fact that $W_h\perp S_{h-1},A_{h-1},\Gamma_{h-1}|S_h$ under Assumption \ref{assump: negative control cond independence}, and step (c) follows by marginalizing $S_{h-1}, A_l,O_l$.
Thus we have proved that $\Delta_h^1=\Delta_h^2$ for finite-length history policies and consequently $\Delta_h=\Delta_h^1-\Delta_h^2=0$.

\vspace{3mm}
\noindent
\textbf{Case 3: Full-history policy (Example \ref{example: full history}).}
For full history information $T_h$, we have that 
$\Gamma_{h-1}\cup\{A_h,O_h\}=T_h$.
Following the same argument as in Case 2 (Example \ref{example: finite history}), we can first show that 
\begin{align*}
    \Delta_h^1 = \sum_{a_h\in\mathcal{A}}\int_{\mathcal{S}\times\mathcal{H}_{h-1}}\mathcal{P}_h^{\pi}(s_h,\tau_{h-1})\mathrm{d}s_h\mathrm{d}\tau_{h-1}\int_{\mathcal{W}}\textcolor{red}{\mathcal{P}_h^b(w_h|s_h)}\mathrm{d}w_h\cdot b^{\pi}_h(a_h,w_h).\notag
\end{align*}
Besides, for $\Delta_2$, by a similar argument as in Case 2 except that we don't need marginalize over $A_l,O_l$ in \eqref{eq: thm 2.6 proof 2}, we can show that 
\begin{align*}
    \Delta_h^2 &=
    \sum_{a_h\in\mathcal{A}}\int_{\mathcal{S}\times\mathcal{H}_{h-2}}\mathcal{P}^{\pi}_{h-1}(s_{h-1},\textcolor{blue}{\tau_{h-2}})\mathrm{d}s_{h-1}\mathrm{d}\tau_{h-2}\int_{\mathcal{O}}\mathbb{O}_{h-1}(\textcolor{blue}{o_{h-1}}|s_{h-1})\mathrm{d}o_{h-1}\sum_{a_{h-1}\in\mathcal{A}}\pi_{h-1}(\textcolor{blue}{a_{h-1}}|o_{h-1},\tau_{h-2})\notag\\
    &\qquad\int_{\mathcal{S}}\mathbb{P}_h(\textcolor{blue}{s_h}|s_{h-1},a_{h-1})\mathrm{d}s_h\int_{\mathcal{W}}\textcolor{red}{\mathcal{P}_h^b(w_h|s_h,s_{h-1},a_{h-1},o_{h-1},\tau_{h-2})} \cdot b_h^{\pi}(a_h,w_h)\notag\\
    &=\sum_{a_h\in\mathcal{A}}\int_{\mathcal{S}\times\mathcal{H}_{h-1}}\mathcal{P}_h^{\pi}(\textcolor{blue}{s_h,\tau_{h-1}})\mathrm{d}s_h\mathrm{d}\tau_{h-1}\int_{\mathcal{W}}\textcolor{red}{\mathcal{P}_h^b(w_h|s_h)}\mathrm{d}w_h\cdot b^{\pi}_h(a_h,w_h).\notag
\end{align*}
Therefore, we show that $\Delta_h^1=\Delta_h^2$ for full history policies and consequently $\Delta_h=\Delta_h^1-\Delta_h^2=0$.

\vspace{3mm}
Now we have shown that term (B) in \eqref{eq: thm 2.6 proof 1} is actually $0$ for Example \ref{example: reactive}, Example \ref{example: finite history}, and Example \ref{example: full history}, respectively, which allows us to conclude that 
\begin{align*}
    J(\pi)&=\textnormal{(A)}=\mathbb{E}_{\pi^b}\left[\sum_{a\in\mathcal{A}}b_1^{\pi}(a,W_1)\right].
\end{align*}
This finishes the proof of Theorem \ref{thm: identification}.
\end{proof}

\section{Proof of Lemmas in Section \ref{sec: proof sketch}}\label{sec: proof of Lemmas in proof sketch}

We first review and define several notations and quantities that are useful in the proof of the lemmas in Section \ref{sec: proof sketch}.
Firstly, we define mapping $\ell_h^{\pi}:\mathbb{B}\times\mathbb{B}\mapsto\{\mathcal{A}\times\mathcal{Z}\mapsto\mathbb{R}\}$ as
\begin{align}\label{eq: ell}
    \ell^{\pi}_h(b_h,b_{h+1})(A_h,Z_h)&\coloneqq \mathbb{E}_{\pi^{b}}\Big[b_{h}(A_h,W_h)-R_h\pi_h(A_h|O_h,\Gamma_{h-1})\notag\\ 
    &\qquad -\gamma\sum_{a^{\prime}\in\cA}b_{h+1}(a^{\prime},W_{h+1})\pi_h(A_h|O_h,\Gamma_{h-1})\Big|A_h,Z_h\Big].
\end{align}
Furthermore, for each step $h\in[H]$, we define a joint space $\mathcal{I}_h=\mathcal{A}\times\mathcal{W}\times\mathcal{O}\times\mathcal{H}_{h-1}\times\mathcal{W}$ and define mapping $\varsigma_h^{\pi}:\mathbb{B}\times\mathbb{B}\mapsto\{\mathcal{I}_{h}\mapsto\mathbb{R}\}$ as 
\begin{align}\label{eq: varsigma}
    \varsigma^{\pi}_h(b_h,b_{h+1})(A_h,W_h,O_h,\Gamma_{h-1},W_{h+1})&\coloneqq b_{h}(A_h,W_h)-R_h\pi_h(A_h|O_h,\Gamma_{h-1})\notag\\ 
    &\qquad -\gamma\sum_{a^{\prime}\in\cA}b_{h+1}(a^{\prime},W_{h+1})\pi_h(A_h|O_h,\Gamma_{h-1}).
\end{align}
When appropriate, we abbreviate $I_h=(A_h,W_h,O_h,\Gamma_{h-1},W_{h+1})\in\mathcal{I}_h$ in the sequel.
Using definition \eqref{eq: ell} and \eqref{eq: varsigma}, we further introduce two mappings $\Phi_{\pi,h}^{\lambda},\Phi_{\pi,h}:\BB\times\BB\times\GG\mapsto\RR$ as defined by \eqref{eq: population phi lambda},
\begin{align*}
% \label{eq: appendix population Phi lambda}
    \Phi_{\pi,h}^{\lambda}(b_h,b_{h+1};g)\coloneqq\mathbb{E}_{\pi^b}\big[\ell_h^{\pi}(b_h,b_{h+1})(A_h,Z_h)\cdot g(A_h,Z_h)-\lambda g(A_h,Z_h)^2\big],
\end{align*}
\begin{align*}
% \label{eq: appendix population Phi}
    \Phi_{\pi,h}(b_h,b_{h+1};g)\coloneqq\Phi_{\pi,h}^{0}(b_h,b_{h+1};g)=\mathbb{E}_{\pi^b}\big[\ell_h^{\pi}(b_h,b_{h+1})(A_h,Z_h)\cdot g(A_h,Z_h)\big],
\end{align*}
where we define that $\Phi_{\pi,h}=\Phi_{\pi,h}^0$. 
Also, recall from \eqref{eq: empirical phi lambda} that the empirical version of $\Phi_{\pi,h}^{\lambda},\Phi_{\pi,h}$ are defined by $\hat{\Phi}_{\pi,h}^{\lambda},\hat{\Phi}_{\pi,h}$ as
\begin{align*}
% \label{eq: appendix empirical Phi lambda}
    \hat{\Phi}_{\pi,h}^{\lambda}(b_h,b_{h+1};g)\coloneqq\hat{\mathbb{E}}_{\pi^b}\big[\varsigma_h^{\pi}(b_h,b_{h+1})(I_h)\cdot g(A_h,Z_h)-\lambda g(A_h,Z_h)^2\big],
\end{align*}
\begin{align*}
% \label{eq: appendix empirical Phi}
    \hat{\Phi}_{\pi,h}(b_h,b_{h+1};g)\coloneqq\hat{\Phi}_{\pi,h}^{0}(b_h,b_{h+1};g)=\hat{\mathbb{E}}_{\pi^b}\big[\varsigma_h^{\pi}(b_h,b_{h+1})(I_h)\cdot g(A_h,Z_h)\big].
\end{align*}
Recall from \eqref{eq: hat b} that given function $b_{h+1}\in\mathbb{B}$, the minimax estimator $\hat{b}_{h}(b_{h+1})$ is defined as 
\begin{align*}
% \label{eq: appendix hat b}
    \hat{b}_h(b_{h+1})\coloneqq\argmin_{b\in\mathbb{B}}\max_{g\in\mathbb{G}}\hat{\Phi}^{\lambda}_{\pi,h}(b,b_{h+1};g).
\end{align*}
Meanwhile, we define the following quantity for ease of theoretical analysis as
\begin{align}\label{eq: def b star}
    b_h^{\star}(b_{h+1})\coloneqq\argmin_{b\in\mathbb{B}}\max_{g\in\mathbb{G}}\Phi^{\lambda}_{\pi,h}(b,b_{h+1};g).
\end{align}
% By definitions \eqref{eq: population phi lambda} and \eqref{eq: RMSE loss}, $\Phi_{\pi,h}(b_h,b_{h+1};g)=\mathbb{E}_{\pi^b}[\ell^{\pi}_h(b_h,b_{h+1})g]$ and  $\mathcal{L}_h^{\pi}(b_h,b_{h+1})=\mathbb{E}_{\pi^b}[\ell^{\pi}_h(b_h,b_{h+1})^2]$.
By the boundedness assumption on $\mathbb{B}$ in Assumption \ref{assump: dual function class}, we have that
$|\ell_h^{\pi}|,|\varsigma_h^{\pi}|\leq 2M_{\mathbb{B}}$.
By the completeness assumption on $\mathbb{G}$ in Assumption \ref{assump: completeness and realizability}, we also know that $\ell_h^{\pi}(b_h,b_{h+1})/2\lambda\in\mathbb{G}$ for any $b_h,b_{h+1}\in\mathbb{B}$.
% By definition \eqref{eq: empirical phi lambda}, we know that $\hat \Phi_{\pi,h}(b_h,b_{h+1};g)=\hat{\mathbb{E}}_{\pi^b}[\varsigma_h^{\pi}(b_h,b_{h+1})g]$.
Finally, for notational simplicity, we define for each $g\in\mathbb{G}$ that, 
\begin{align*}
% \label{eq: g norm}
    \|g\|_2^2\coloneqq\mathbb{E}_{\pi^b}[g(A_h,Z_h)^2],
\end{align*}
and we denote by $\|g\|_{2,n}^2$ its empirical version, i.e., 
\begin{align*}
    \|g\|_{2,n}^2\coloneqq\hat{ \mathbb{E}}_{\pi^b}[g(A_h,Z_h)^2].
\end{align*}
We remark that we have dropped the dependence of $\|g\|_2^2$ on step $h$ since it is clear from the context when used in the proofs and does not make any confusion.

\subsection{Proof of Lemma \ref{lem: regret decomposition}}\label{subsec: regret decomposition proof}

\begin{proof}[Proof of Lemma \ref{lem: regret decomposition}]
By definition \eqref{eq: define J} of $F(\bbb)$, for any policy $\pi\in\Pi(\cH)$ and vector of functions $\bbb\in\mathbb{B}^{\otimes H}$, it holds that
\begin{align*}
    F(\bbb^{\pi})-F(\bbb)&\eqw{(a)}\mathbb{E}_{\pi^b}\left[\sum_{a\in\mathcal{A}}b^{\pi}_1(a,W_1)-b_1(a,W_1)\right] \notag\\
    &=\mathbb{E}_{\pi^b}\left[\sum_{a\in\mathcal{A}}\frac{\pi^b_1(a|S_1)}{\pi_1^b(a|S_1)}\left(b^{\pi}_1(a,W_1)-b_1(a,W_1)\right)\right]\notag\\
    &\eqw{(b)}\mathbb{E}_{\pi^b}\left[\mathbb{E}_{\pi^b}\left[\frac{1}{\pi_1^b(A_1|S_1)}\left(b^{\pi}_1(a,W_1)-b_1(a,W_1)\right)\middle| S_1,W_1\right]\right]\notag\\
      &=\mathbb{E}_{\pi^b}\left[\frac{1}{\pi_1^b(A_1|S_1)}\left(b^{\pi}_1(A_1,W_1)-b_1(A_1,W_1)\right)\right]\notag
\end{align*}
where step (a) follows from Theorem \ref{thm: identification} and \eqref{eq: define J}, and step (b) holds since $A_1\perp W_1\mid S_1$ by Assumption \ref{assump: negative control cond independence}.
Notice that by definition \eqref{eq: def weight bridge}, at step $h=1$, the weight bridge function $q_h^{\pi}$ satisfies equation 
\begin{align*}
    \mathbb{E}_{\pi^b}[q^{\pi}_1(A_1,Z_1)|A_1,S_1,\Gamma_0]=\frac{\mathcal{P}_h^{\pi}(S_1,\Gamma_0)}{\mathcal{P}_h^{\pi}(S_1,\Gamma_0)\pi^b_1(A_1|S_1)}=\frac{1}{\pi^b_1(A_1|S_1)},
\end{align*}
which further gives that 
\begin{align*}
    F(\bbb^{\pi})-F(\bbb)&=\mathbb{E}_{\pi^b}\big[\mathbb{E}_{\pi^b}\left[q^{\pi}_1(A_1,Z_1)\middle| A_1,S_1,\Gamma_{0}\right]\left(b^{\pi}_1(A_1,W_1)-b_1(A_1,W_1)\right)\big ]\notag\\
    &\eqw{(a)}\mathbb{E}_{\pi^b}\bigl [\mathbb{E}_{\pi^b}\left[q^{\pi}_1(A_1,Z_1)\middle| A_1,S_1,W_1,\Gamma_{0}\right] \cdot \bigl(b^{\pi}_1(A_1,W_1)-b_1(A_1,W_1)\bigr )\bigr]\notag\\
    &=\mathbb{E}_{\pi^b}\bigl [q^{\pi}_1(A_1,Z_1)\bigl(b^{\pi}_1(A_1,W_1)-b_1(A_1,W_1)\bigr )\bigr ]\notag,
\end{align*}
where step (a)holds since $Z_1\perp W_1\mid A_1,S_1,\mathcal{H}_0$ by Assumption \ref{assump: negative control cond independence}.
Now we can further obtain that,
\begin{align*}
    F(\bbb^{\pi})-F(\bbb) &=\mathbb{E}_{\pi^b}\big [q^{\pi}_1(A_1,Z_1)\mathbb{E}_{\pi^b}\left[b^{\pi}_1(A_1,W_1)-b_1(A_1,W_1)\middle| A_1,Z_1\right]\bigr ]\notag\\
    &\eqw{(a)}\mathbb{E}_{\pi^b}\bigg[q^{\pi}_1(A_1,Z_1)\bigg\{\EE_{\pi^b}\bigg[ R_1 \pi_1(A_1|O_1, \Gamma_0) + \gamma \sum_{a'}b^{\pi}_{2}(a', W_{2}) \pi_1(A_1|O_1, \Gamma_0)\bigg| A_1, Z_1 \bigg] \notag\\
    &\qquad -\mathbb{E}_{\pi^b}\left[b_1(A_1,W_1)| A_1,Z_1\right]\bigg\}\bigg],\notag
\end{align*}
where step (a) follows from the definition in \eqref{eq: def value bridge} of value bridge function $b_1^{\pi}$ in Assumption \ref{assump: bridge functions exist}. Now to relate the difference between $F(\bbb^\pi)$ and $F(\bbb)$ with the RMSE loss $\mathcal{L}_1^{\pi}$ defined in \eqref{eq: RMSE loss}, we rewrite the above equation as the following,
\begin{align*}
 &F(\bbb^{\pi})-F(\bbb)\notag\\
    &\qquad =\mathbb{E}_{\pi^b}\bigg[q^{\pi}_1(A_1,Z_1)\bigg\{\EE_{\pi^b}\bigg[ R_1 \pi_h(A_1|O_1, \Gamma_0) + \gamma \sum_{a'}b^{\pi}_{2}(a', W_{2}) \pi_1(A_1|O_1, \Gamma_0)\bigg| A_1, Z_1 \bigg] \notag\\
    &\qquad\qquad -\EE_{\pi^b}\bigg[ R_h \pi_1(A_1|O_1, \Gamma_0) + \gamma \sum_{a'}b_{2}(a', W_{h+1}) \pi_1(A_1|O_1, \Gamma_{0})\bigg| A_1, Z_1 \bigg]\notag\\
    &\qquad\qquad+\EE_{\pi^b}\bigg[ R_1 \pi_1(A_1|O_1, \Gamma_0) + \gamma \sum_{a'}b_{2}(a', W_{2}) \pi_1(A_1|O_1, \Gamma_0)\bigg| A_1, Z_1 \bigg]\notag\\
    &\qquad\qquad -\mathbb{E}_{\pi^b}\bigg[b_1(A_1,W_1)\bigg| A_1,Z_1\bigg]\bigg\}\bigg],
\end{align*} where we add one term and subtract the same term. It immediately follows that 
\begin{align}\label{eq: proof lem 4.2 eq 1}
 &F(\bbb^{\pi})-F(\bbb)\notag\\
    &\qquad =\mathbb{E}_{\pi^b}\bigg[q^{\pi}_1(A_1,Z_1)\bigg\{\gamma\EE_{\pi^b}\bigg[ \sum_{a'}\Big(b^{\pi}_{2}(a', W_{2})-b_2(a,W_2)\Big) \pi_1(A_1|O_1, \Gamma_0)\bigg| A_1, Z_1 \bigg] \notag\\
    &\qquad\qquad+\EE_{\pi^b}\bigg[ R_1 \pi_h(A_1|O_1, \Gamma_0) + \gamma \sum_{a'}b_{2}(a', W_{2}) \pi_h(A_1|O_1, \Gamma_0)-b_1(A_1,W_1)\bigg| A_1, Z_1 \bigg]\bigg\}\bigg].
\end{align}
We deal with the two terms in the right-hand side of \eqref{eq: proof lem 4.2 eq 1} respectively.
On the one hand, the first term equals to
\begin{align*}
    &\gamma\mathbb{E}_{\pi^b}\bigg[q^{\pi}_1(A_1,Z_1)\EE_{\pi^b}\bigg[ \sum_{a'}\Big(b^{\pi}_{2}(a', W_{2})-b_2(a,W_2)\Big) \pi_1(A_1|O_1, \Gamma_0)\bigg| A_1, Z_1 \bigg]\bigg]\notag\\
    &\qquad=\gamma\mathbb{E}_{\pi^b}\bigg[q^{\pi}_1(A_1,Z_1)\sum_{a'}\Big(b^{\pi}_{2}(a', W_{2})-b_2(a,W_2)\Big) \pi_1(A_1|O_1, \Gamma_0)\bigg]\notag\\
    &\qquad=\gamma\mathbb{E}_{\pi^b}\bigg[\mathbb{E}_{\pi^b}\bigg[q^{\pi}_1(A_1,Z_1)\bigg|S_1,A_1,\Gamma_0,O_1,W_2\bigg]\sum_{a'}\Big(b^{\pi}_{2}(a', W_{2})-b_2(a,W_2)\Big) \pi_1(A_1|O_1, \Gamma_0)\bigg]\notag\\
    &\qquad\eqw{(a)}\gamma\mathbb{E}_{\pi^b}\bigg[\mathbb{E}_{\pi^b}\bigg[q^{\pi}_1(A_1,Z_1)\bigg|S_1,A_1,\Gamma_0\bigg]\sum_{a'}\Big(b^{\pi}_{2}(a', W_{2})-b_2(a,W_2)\Big) \pi_1(A_1|O_1, \Gamma_0)\bigg]\notag\\
    &\qquad\eqw{(b)}\gamma\mathbb{E}_{\pi^b}\Bigg[\frac{\mu_1(S_1,\Gamma_0)}{\pi_1^b(A_1|S_1)}\sum_{a'}\Big(b^{\pi}_{2}(a', W_{2})-b_2(a,W_2)\Big) \pi_1(A_1|O_1, \Gamma_0)\Bigg],\notag
\end{align*}
where step (a) follows from the fact that $Z_1\perp O_1,W_2|S_1,A_1,\Gamma_{0}$ according to Assumption \ref{assump: negative control cond independence}, and step (b) follows from the definition \eqref{eq: def weight bridge} of weight bridge function $q_1^{\pi}$ in Assumption \ref{assump: bridge functions exist}. 
Now following the same argument as in showing $\Delta_h=0$ in the proof of Theorem \ref{thm: identification}, we can show that 
\begin{align}\label{eq: proof lem 4.2 eq 2}
    &\mathbb{E}_{\pi^b}\Bigg[\frac{\mu_1(S_1,\Gamma_0)}{\pi_1^b(A_1|S_1)}\sum_{a'}\Big(b^{\pi}_{2}(a', W_{2})-b_2(a,W_2)\Big) \pi_1(A_1|O_1, \Gamma_0)\Bigg]\notag\\
    &\qquad=\mathbb{E}_{\pi^b}\Bigg[q^{\pi}_2(A_2,Z_2)\Big(b^{\pi}_{2}(A_2, W_{2})-b_2(A_2,W_2)\Big) \Bigg].
\end{align}
On the other hand, the second term in the R.H.S. of \eqref{eq: proof lem 4.2 eq 1} can be rewritten and bounded by 
\begin{align}\label{eq: proof lem 4.2 eq 3}
    &\EE_{\pi^b}\bigg[q^{\pi}_1(A_1,Z_1)\EE_{\pi^b}\bigg[ R_1 \pi_1(A_1|O_1, \Gamma_0) + \gamma \sum_{a'}b_{2}(a', W_{2}) \pi_1(A_1|O_1, \Gamma_0)-b_1(A_1,W_1)\bigg| A_1, Z_1 \bigg]\bigg]\notag\\
    &\qquad \leq  \sqrt{C^{\pi}}\cdot \left(\EE_{\pi^b}\left[\bigg(\EE_{\pi^b}\bigg[ R_1 \pi_1(A_1|O_1, \Gamma_0) + \gamma \sum_{a'}b_{2}(a', W_{2}) \pi_1(A_1|O_1, \Gamma_0)-b_1(A_1,W_1)\bigg| A_1, Z_1 \bigg]\bigg)^{2}\right]\right)^{1/2}\notag\\
    &\qquad =  \sqrt{C^{\pi}}\cdot\sqrt{\mathcal{L}^{\pi}_1(b_1,b_2)},
\end{align}
where $C^{\pi}$ is defined as $C^{\pi}\coloneqq\sup_{h\in[H]}\mathbb{E}_{\pi^b}\left[(q^{\pi}_h(A_h,Z_h))^2\right]$, the inequality follows from Cauchy-Schwarz inequality, and the equality follows from the definition of $\mathcal{L}^{\pi}_1$ in \eqref{eq: RMSE loss}.
Combining \eqref{eq: proof lem 4.2 eq 1}, \eqref{eq: proof lem 4.2 eq 2} with \eqref{eq: proof lem 4.2 eq 3}, we can obtain that 
\begin{align}\label{eq: B13}
     &F(\bbb^{\pi})-F(\bbb) \leq  \sqrt{C^{\pi}}\cdot\sqrt{\mathcal{L}^{\pi}_1(b_1,b_2)}+\gamma\mathbb{E}_{\pi^b}\Bigg[q^{\pi}_2(A_2,Z_2)\Big(b^{\pi}_{2}(A_2, W_{2})-b_2(A_2,W_2)\Big) \Bigg].
\end{align}
Now applying the above argument on the second term in the R.H.S. of \eqref{eq: B13} recursively, we can obtain that 
\begin{align*}
   F(\bbb^{\pi})-F(\bbb)\leq \sum_{h=1}^H\gamma^{h-1}\sqrt{C^{\pi}}\cdot\sqrt{\mathcal{L}_h^{\pi}(b_h,b_{h+1})}.
\end{align*}
This finishes the proof of Lemma \ref{lem: regret decomposition}.
\end{proof}

\subsection{Proof of Lemma \ref{lem: true in CR}}\label{subsec: true in CR proof}
\begin{proof}[Proof of Lemma \ref{lem: true in CR}]
By the definition of the confidence region $\CR^{\pi}(\alpha)$ in \eqref{eq: define CR}, we need to show for any policy $\pi\in\Pi(\cH)$ and step $h\in[H]$, it holds that,
\begin{align}\label{eq: proof lem 3.3 eq 1}
    \max_{g \in \mathbb{G}} \hat{\Phi}_{\pi, h}^\lambda ( b_h^{\pi}, b_{h+1}^{\pi};g) -  \max_{g \in \mathbb{G}}  \hat{\Phi}_{\pi, h}^\lambda(\hat{b}_{h}(b_{h+1}^{\pi}), b_{h+1}^{\pi};g) \leq \xi.
\end{align}
Notice that by Assumption \ref{assump: dual function class}, the function class $\mathbb{G}$ is symmetric and star-shaped, which indicates that 
\begin{align*}
    \max_{g \in \mathbb{G}}  \hat{\Phi}_{\pi, h}^\lambda(\hat{b}_{h}(b_{h+1}^{\pi}), b_{h+1}^{\pi};g)\geq\hat{\Phi}_{\pi, h}^\lambda(\hat{b}_{h}(b_{h+1}^{\pi}), b_{h+1}^{\pi};0)= 0.
\end{align*}
Therefore, in order to prove \eqref{eq: proof lem 3.3 eq 1}, it suffices to show that 
\begin{align}\label{eq: proof lem 3.3 eq 2}
    \max_{g \in \mathbb{G}} \hat{\Phi}_{\pi, h}^\lambda ( b_h^{\pi}, b_{h+1}^{\pi};g)\leq \xi.
\end{align}
To relate the empirical expectation $\hat{\Phi}_{\pi, h}^\lambda ( b_h^{\pi}, b_{h+1}^{\pi};g)=\hat \Phi_{\pi,h}(b_h^{\pi},b_{h+1}^{\pi};g)-\lambda\|g\|_{2,n}^2$ to its population version, we need two localized uniform concentration inequalities. On the one hand, to relate $\|g\|_2^2$ and $\|g\|_{2,n}^2$, by Lemma \ref{lem: localized uniform concentration} (Theorem 14.1 of \cite{wainwright2019high}), for some absolute constants $c_1,c_2>0$, it holds with probability at least $1-\delta/2$ that,
\begin{align}\label{eq: proof lem 3.3 concentraion 1}
    \left|\|g\|_{2,n}^2-\|g\|_2^2\right|\leq \frac{1}{2}\|g\|_2^2+\frac{M_{\mathbb{G}}^2\log(2c_1/\zeta)}{2c_2n},\quad \forall g\in\mathbb{G},
\end{align}
where $\zeta = \min\{\delta,2c_1\exp(-c_2n\alpha_{\mathbb{G},n}^2/M_{\mathbb{G}}^2)\}$ and $\alpha_{\mathbb{G},n}$ is the critical radius of function class $\mathbb{G}$ defined in Assumption \ref{assump: dual function class}.
On the other hand, to relate $\hat \Phi_{\pi,h}(b_h,b_{h+1};g)$ and $\Phi_{\pi,h}(b_h,b_{h+1};g)$
we invoke Lemma \ref{lem: localized uniform concentration 2} (Lemma 11 of \citep{foster2019orthogonal}).
Specifically, for any given $b_h,b_{h+1}\in\mathbb{B}$, $\pi\in\Pi(\cH)$, and $h\in[H]$, in Lemma \ref{lem: localized uniform concentration 2} we choose $\mathcal{F}=\mathbb{G}$, $\mathcal{X}=\mathcal{A}\times\mathcal{Z}$, $\mathcal{Y}=\mathcal{I}_h$, and loss function $\ell(g(A_h,Z_h),I_h)\coloneqq\varsigma_h^{\pi}(b_h,b_{h+1})(I_h)\cdot g(A_h,Z_h)$ where $\varsigma^{\pi}_h$ is defined in \eqref{eq: ell}, $I_h\in\mathcal{I}_h$ is defined in the beginning of Appendix \ref{sec: proof of Lemmas in proof sketch}.
It holds that $\ell$ is $L$-Lipschitz continuous in the first argument since for any $g,g^{\prime}\in\mathbb{G}$,
$(A_h,Z_h)\in\mathcal{A}\times\mathcal{Z}$, it holds that
\begin{align*}
    \big|\ell(g(A_h,Z_h),I_h)-\ell(g^{\prime}(A_h,Z_h),I_h)\big|&=|\varsigma_h^{\pi}(b_h,b_{h+1})(I_h)|\cdot|g(A_h,Z_h)-g^{\prime}(A_h,Z_h)|\notag\\
    &\leq2M_{\mathbb{B}}\cdot |g(A_h,Z_h)-g^{\prime}(A_h,Z_h)|,\notag
\end{align*}
which indicates that $L=2M_{\mathbb{B}}$.
Now setting $f^{\star}=0$ in Lemma \ref{lem: localized uniform concentration 2}, we have that $\delta_n$ in Lemma \ref{lem: localized uniform concentration 2} coincides with $\alpha_{\mathbb{G},n}$ in Assumption \ref{assump: dual function class}. 
Then we can conclude that for some absolute constants $c_1,c_2>0$, it holds with probability at least $1-\delta/(2|\mathbb{B}|^2|\Pi(\cH)|H)$ that 
\begin{align}\label{eq: proof lem 3.3 concentraion 2}
    &\left|\hat \Phi_{\pi,h}(b_h,b_{h+1};g)-\Phi_{\pi,h}(b_h,b_{h+1};g)\right|\notag\\&\qquad =\left|\hat{\mathbb{E}}_{\pi^b}[\ell(g(A_h,Z_h),I_h)]-\mathbb{E}_{\pi^b}[\ell(g(A_h,Z_h),I_h)]\right|\notag\\
    &\qquad\leq 18L\|g\|_2\sqrt{\frac{M_{\mathbb{G}}^2\log\big(2c_1|\mathbb{B}|^2|\Pi(\cH)|H/\zeta'\big)}{c_2n}}+\frac{18LM_{\mathbb{G}}^2\log\big(2c_1|\mathbb{B}|^2|\Pi(\cH)|H/\zeta'\big)}{c_2n},\quad \forall g\in\mathbb{G},
\end{align}
where $\zeta' = \min\{\delta,2c_1|\mathbb{B}|^2|\Pi(\cH)|H\exp(-c_2n\alpha_{\mathbb{G},n}^2/M_{\mathbb{G}}^2)\}$.
 Applying a union bound argument over $b_h,b_{h+1}\in\mathbb{B}$, $\pi\in\Pi(\cH)$, and $h\in[H]$, we then have that \eqref{eq: proof lem 3.3 concentraion 2} holds for any $b_h,b_{h+1}\in\mathbb{B}$, $g\in\mathbb{G}$, $\pi\in\Pi(\cH)$, and $h\in[H]$ with probability at least $1-\delta/2$.
Now using these two concentration inequalities \eqref{eq: proof lem 3.3 concentraion 1} and \eqref{eq: proof lem 3.3 concentraion 2}, we can further deduce that, for some absolute constants $c_1,c_2>0$, with probability at least $1-\delta$,
\begin{align*}
    &\max_{g \in \mathbb{G}} \hat{\Phi}_{\pi, h}^\lambda ( b_h^{\pi}, b_{h+1}^{\pi};g)\notag\\
    &\qquad =\max_{g \in \mathbb{G}}\left\{\hat{\Phi}_{\pi, h} ( b_h^{\pi}, b_{h+1}^{\pi};g)-\lambda\|g\|_{2,n}^2\right\}\notag\\
    &\qquad \leq \max_{g \in \mathbb{G}}\Bigg\{\Phi_{\pi, h} ( b_h^{\pi}, b_{h+1}^{\pi};g)-\lambda\|g\|_{2}^2
    +\frac{\lambda}{2}\|g\|_2^2+\frac{\lambda M_{\mathbb{G}}^2\log(2c_1/\zeta)}{2c_2n},\notag\\
    &\qquad\qquad +18L\|g\|_2\sqrt{\frac{M_{\mathbb{G}}^2\log(2c_1|\mathbb{B}|^2|\Pi(\cH)|H/\zeta^{\prime})}{c_2n}}+\frac{18LM_{\mathbb{G}}^2\log(2c_1|\mathbb{B}|^2|\Pi(\cH)|H/\zeta^{\prime})}{c_2n}\Bigg\},\notag\\
\end{align*}
where $\zeta = \min\{\delta, 2c_1\exp(-c_2n\alpha_{\mathbb{G},n}^2/M_{\mathbb{G}}^2)\}$ and $\zeta^{\prime} = \min\{\delta,2c_1|\mathbb{B}|^2|\Pi(\cH)|H\exp(-c_2n\alpha_{\mathbb{G},n}^2/M_{\mathbb{G}}^2)\}$
for any policy $\pi\in\Pi(\cH)$ and step $h$.
Then we can further bound the right-hand side of the above inequality as 
\begin{align*}
    &\max_{g \in \mathbb{G}} \hat{\Phi}_{\pi, h}^\lambda ( b_h^{\pi}, b_{h+1}^{\pi};g)\notag\\
    &\qquad \leq \max_{g\in\mathbb{G}}\Phi_{\pi,h}( b_h^{\pi}, b_{h+1}^{\pi};g)+\max_{g\in\mathbb{G}}\Bigg\{
    -\frac{\lambda}{2}\|g\|_2^2+18L\|g\|_2\sqrt{\frac{M_{\mathbb{G}}^2 \cdot \log(2c_1|\mathbb{B}|^2|\Pi(\cH)|H/\zeta^{\prime})}{c_2n}}\Bigg\}\notag\\
    &\qquad\qquad +\frac{\lambda M_{\mathbb{G}}^2\cdot \log(2c_1/\zeta)}{2c_2n}+ \frac{18LM_{\mathbb{G}}^2\cdot \log(2c_1|\mathbb{B}|^2|\Pi(\cH)|H/\zeta^{\prime})}{c_2n}
    \notag\\
    &\qquad
    \leq \frac{728L^2\cdot M_{\mathbb{G}}^2\cdot \log(2c_1|\mathbb{B}|^2|\Pi(\cH)|H/\zeta^{\prime})}{\lambda n} + \frac{\lambda M_{\mathbb{G}}^2\cdot \log(2c_1/\zeta)}{2c_2n} \notag
    \\ & \qquad\qquad + \frac{18LM_{\mathbb{G}}^2\cdot \log(2c_1|\mathbb{B}|^2|\Pi(\cH)|H/\zeta^{\prime})}{c_2n}.
\end{align*}
Here the last inequality holds from the fact that $\Phi_{\pi, h} ( b_h^{\pi}, b_{h+1}^{\pi};g)=0$ since $b_h^{\pi}$ and $b_{h+1}^{\pi}$ are true bridge functions, and the fact that
$\sup_{\|g\|_2}\{a\|g\|_2-b\|g\|_2^2\}\leq a^2/4b$ for any $b>0$.
Now according to the choice of $\xi$ in Lemma \ref{lem: true in CR}, using the fact that $\zeta<\zeta^{\prime}$ and $L=2M_{\mathbb{B}}$, we can conclude that, with probability at least $1-\delta$,
\begin{align*}
    &\max_{g \in \mathbb{G}} \hat{\Phi}_{\pi, h}^\lambda ( b_h^{\pi}, b_{h+1}^{\pi};g)\notag\\ %-  \max_{g \in \mathbb{G}}  \hat{\Phi}_{\pi, h}^\lambda(\hat{b}_{h}(b_{h+1}^{\pi}), b_{h+1}^{\pi};g)\notag\\
    &\qquad \leq \frac{728L^2M_{\mathbb{G}}^2\cdot \log(2c_1|\mathbb{B}|^2|\Pi(\cH)|H/\zeta^{\prime})}{\lambda n} + \frac{\lambda M_{\mathbb{G}}^2\cdot \log(2c_1/\zeta)}{2c_2n}+ \frac{18LM_{\mathbb{G}}^2\cdot \log(2c_1|\mathbb{B}|^2|\Pi(\cH)|H/\zeta^{\prime})}{c_2n}\notag\\
    &\qquad\lesssim\mathcal{O}\left(\frac{(\lambda+1/\lambda)\cdot M_{\mathbb{B}}^2M_{\mathbb{G}}^2 \cdot \log(|\mathbb{B}||\Pi(\cH)|H/\zeta)}{n}\right)\lesssim \xi.
\end{align*}
This proves \eqref{eq: proof lem 3.3 eq 2}, and thus further indicates \eqref{eq: proof lem 3.3 eq 1}. Therefore,  we finish the proof of Lemma \ref{lem: true in CR}.
\end{proof}

\subsection{Proof of Lemma \ref{lem: accuracy of CR}}\label{subsec: accuracy of CR proof}

In order to prove Lemma \ref{lem: accuracy of CR}, we first introduce the following lemma, which claims that for any $b_{h+1}\in\mathbb{B}$, the $b^{\star}(b_{h+1})$ defined in \eqref{eq: def b star} satisfies that $\max_{g\in\mathbb{G}}\hat{\Phi}^{\lambda}_{\pi,h}(b^{\star}(b_{h+1}),b_{h+1};g)$ is well-bounded.
The proof of lemma follows the same argument as in the proof of Lemma \ref{lem: true in CR}, which we defer to Appendix \ref{subsec: proof Phi b star bound}.
\begin{lemma}\label{lem: Phi b star bound}
    For any function $b_{h+1}\in\mathbb{B}$, policy $\pi\in\Pi(\cH)$, and step $h\in[H]$, it holds with probability at least $1-\delta/2$
    that 
    \begin{align*}
        \max_{g\in\mathbb{G}}\hat{\Phi}_{\pi,h}^{\lambda}(b^{\star}_h(b_{h+1}),b_{h+1};g)\leq \xi+\epsilon_{\mathbb{B}}^{1/2}M_{\mathbb{G}},
    \end{align*}
    where $b^{\star}(b_{h+1})$ is defined in \eqref{eq: def b star} and $\xi$ is defined in Lemma \ref{lem: accuracy of CR}.
\end{lemma}

\begin{proof}[Proof of Lemma \ref{lem: Phi b star bound}]
See Appendix \ref{subsec: proof Phi b star bound} for a detailed proof.
\end{proof} 

With Lemma \ref{lem: Phi b star bound}, we are now ready to give the proof of Lemma \ref{lem: accuracy of CR}.

\begin{proof}[Proof of Lemma \ref{lem: accuracy of CR}]
Let's consider that for any $b_h,b_{h+1}\in\CR^{\pi}(\xi)$, we have that
\begin{align*}
    \max_{g \in \mathbb{G}} \hat{\Phi}_{\pi, h}^\lambda (b_h, b_{h+1};g)
    & =\max_{g \in \mathbb{G}} \Big\{
    \hat{\Phi}_{\pi, h} (b_h, b_{h+1};g)-
    \hat{\Phi}_{\pi, h} (b^{\star}_h(b_{h+1}), b_{h+1};g)-2\lambda\|g\|_{2,n}^2 \notag
    \\ & \qquad+\hat{\Phi}_{\pi, h} (b^{\star}_h(b_{h+1}), b_{h+1};g)
    +\lambda\|g\|_{2,n}^2\Big\}.
    \end{align*}
    We further write the above as
    \begin{align}\label{eq: proof lem 3.3 eq 3}
    \max_{g \in \mathbb{G}} \hat{\Phi}_{\pi, h}^\lambda (b_h, b_{h+1};g)
    & \geq \max_{g \in \mathbb{G}} \Big\{
    \hat{\Phi}_{\pi, h} (b_h, b_{h+1};g)-
    \hat{\Phi}_{\pi, h} (b^{\star}_h(b_{h+1}), b_{h+1};g)-2\lambda\|g\|_{2,n}^2\Big\}\notag \\&\qquad+\min_{g\in\mathbb{G}}\Big\{
    \hat{\Phi}_{\pi, h} (b^{\star}_h(b_{h+1}), b_{h+1};g)
    +\lambda\|g\|_{2,n}^2\Big\}\notag\\
    &\eqw{(a)}\underbrace{\max_{g \in \mathbb{G}} \Big\{
    \hat{\Phi}_{\pi, h} (b_h, b_{h+1};g)-
    \hat{\Phi}_{\pi, h} (b^{\star}_h(b_{h+1}), b_{h+1};g)-2\lambda\|g\|_{2,n}^2\Big\}}_{\displaystyle{(\star)}}\notag
    \\&\qquad -\max_{g\in\mathbb{G}}
    \hat{\Phi}_{\pi, h}^{\lambda} (b^{\star}_h(b_{h+1}), b_{h+1};g).
\end{align}
Here step (a) follows from that $\mathbb{G}$ is symmetric, $\hat{\Phi}_{\pi,h}(b_h,h_{h+1};-g)=-\hat{\Phi}_{\pi,h}(b_h,h_{h+1};g)$, and that 
\begin{align*}
    \min_{g\in\mathbb{G}}\Big\{
    \hat{\Phi}_{\pi, h} (b^{\star}_h(b_{h+1}), b_{h+1};g)
    +\lambda\|g\|_{2,n}^2\Big\}&=\min_{g\in\mathbb{G}}\Big\{
    -\hat{\Phi}_{\pi, h} (b^{\star}_h(b_{h+1}), b_{h+1};-g)
    +\lambda\|g\|_{2,n}^2\Big\}\\
    &=\min_{g\in\mathbb{G}}\Big\{
    -\hat{\Phi}_{\pi, h} (b^{\star}_h(b_{h+1}), b_{h+1};g)
    +\lambda\|g\|_{2,n}^2\Big\}\\
    &=-\max_{g\in\mathbb{G}}\Big\{
    \hat{\Phi}_{\pi, h} (b^{\star}_h(b_{h+1}), b_{h+1};g)
    -\lambda\|g\|_{2,n}^2\Big\}\\
    &=-\max_{g\in\mathbb{G}}\hat{\Phi}^{\lambda}_{\pi, h}(b^{\star}_h(b_{h+1}),b_{h+1};g).
\end{align*}
In the sequel, we upper and lower bound term $(\star)$ respectively.

\vspace{3mm}
\noindent
\textbf{Upper bound of term ($\star$).}
By inequality \eqref{eq: proof lem 3.3 eq 3}, after rearranging terms, we can arrive that 
\begin{align*}
    (\star)&\leq 
    \max_{g\in\mathbb{G}}\hat{\Phi}^{\lambda}_{\pi, h}(b^{\star}_h(b_{h+1}),b_{h+1};g)+\max_{g\in\mathbb{G}}\hat{\Phi}^{\lambda}_{\pi, h}(b_h,b_{h+1};g)\notag\\
    &\leq \max_{g\in\mathbb{G}}\hat{\Phi}^{\lambda}_{\pi, h}(b^{\star}_h(b_{h+1}),b_{h+1};g)\notag\\
    &\qquad +\max_{g\in\mathbb{G}}\hat{\Phi}^{\lambda}_{\pi, h}(b_h,b_{h+1};g)-\max_{g\in\mathbb{G}}\hat{\Phi}^{\lambda}_{\pi, h}(\hat{b}_h(b_{h+1}),b_{h+1};g)\notag\\
    &\qquad +\max_{g\in\mathbb{G}}\hat{\Phi}^{\lambda}_{\pi, h}(\hat{b}_h(b_{h+1}),b_{h+1};g)
\end{align*}
On the one hand, by Lemma \ref{lem: Phi b star bound}, we have that with probability at least $1-\delta/2$,
\begin{align}\label{eq: proof lemma 6.3 1}
        \max_{g\in\mathbb{G}}\hat{\Phi}_{\pi,h}^{\lambda}(b^{\star}_h(b_{h+1}),b_{h+1};g)\leq \xi+\epsilon_{\mathbb{B}}^{1/2}M_{\mathbb{G}},
    \end{align}
and by the definition of $\hat{b}_h(b_{h+1})$ in \eqref{eq: hat b}, it holds simultaneously that
\begin{align}\label{eq: proof lemma 6.3 2}
    \max_{g\in\mathbb{G}}\hat{\Phi}^{\lambda}_{\pi, h}(\hat{b}_h(b_{h+1}),b_{h+1};g)\leq \max_{g\in\mathbb{G}}\hat{\Phi}^{\lambda}_{\pi, h}(b^{\star}_h(b_{h+1}),b_{h+1};g)\leq \xi+\epsilon_{\mathbb{B}}^{1/2}M_{\mathbb{G}}.
\end{align}
On the other hand, by the choice of $\CR^{\pi}(\xi)$, it holds that 
\begin{align}\label{eq: proof lemma 6.3 3}
    \max_{g\in\mathbb{G}}\hat{\Phi}^{\lambda}_{\pi, h}(b_h,b_{h+1};g)-\max_{g\in\mathbb{G}}\hat{\Phi}^{\lambda}_{\pi, h}(\hat{b}_h(b_{h+1}),b_{h+1};g)\leq \xi.
\end{align}
Consequently, by combining \eqref{eq: proof lemma 6.3 1}, \eqref{eq: proof lemma 6.3 2}, and \eqref{eq: proof lemma 6.3 3}, we conclude that with probability at least $1-\delta/2$,
\begin{align}\label{eq: upper bound}
    (\star)\leq 3\xi+2\epsilon_{\mathbb{B}}^{1/2}M_{\mathbb{G}}.
\end{align}

\vspace{3mm}
\noindent
\textbf{Lower bound of term ($\star$).}
For lower bound, we need two localized uniform concentration inequalities similar to \eqref{eq: proof lem 3.3 concentraion 1} and \eqref{eq: proof lem 3.3 concentraion 2} in the proof of Lemma \ref{lem: true in CR}.
On the one hand, by Lemma \ref{lem: localized uniform concentration}, for some absolute constants $c_1,c_2>0$, it holds with probability at least $1-\delta/4$ that,
\begin{align}\label{eq: proof lem 3.4 concentraion 1}
    \left|\|g\|_{2,n}^2-\|g\|_2^2\right|\leq \frac{1}{2}\|g\|_2^2+\frac{M_{\mathbb{G}}^2\log(4c_1/\zeta)}{2c_2n},\quad \forall g\in\mathbb{G},\end{align}
where $\zeta = \min\{\delta, 4c_1\exp(-c_2n\alpha_{\mathbb{G},n}^2/M_{\mathbb{G}}^2)\}$ and $\alpha_{\mathbb{G},n}$ is the critical radius of $\mathbb{G}$ defined in Assumption \ref{assump: dual function class}.
On the other hand, following the same argument as in deriving \eqref{eq: proof lem 3.3 concentraion 2}, 
for any given $b_h,b_h^{\prime},b_{h+1}\in\mathbb{B}$, $\pi\in\Pi(\cH)$, and $h\in[H]$, in Lemma \ref{lem: localized uniform concentration 2} we choose $\mathcal{F}=\mathbb{G}$, $\mathcal{X}=\mathcal{A}\times\mathcal{Z}$, $\mathcal{Y}=\mathcal{I}$, and loss function 
\begin{align*}
\ell(g(A_h,Z_h),I_h)\coloneqq\varsigma_h^{\pi}(b_h,b_{h+1})(I_h)g(A_h,Z_h)-\varsigma_h^{\pi}(b_h^{\prime},b_{h+1})(I_h)g(A_h,Z_h),
\end{align*}
where $\varsigma^{\pi}_h$ is defined in \eqref{eq: ell} and $I_h\in\mathcal{I}_h$ is defined in the beginning of Appendix \ref{sec: proof of Lemmas in proof sketch}. It holds that $\ell$ is $L$-Lipschitz continuous in its first argument with $L=2M_{\mathbb{B}}$.
Now setting $f^{\star}=0$ in Lemma \ref{lem: localized uniform concentration 2}, we have that $\delta_n$ in Lemma \ref{lem: localized uniform concentration 2} coincides with $\alpha_{\mathbb{G},n}$ in Assumption \ref{assump: dual function class}. 
Then we have that for some absolute constants $c_1,c_2>0$, it holds with probability at least $1-\delta/(4|\mathbb{B}|^3|\Pi(\cH)|H)$ that
\begin{align}\label{eq: proof lem 3.4 concentraion 2}
&\left|\Big(\hat{\Phi}_{\pi, h} (b_h, b_{h+1};g)-
    \hat{\Phi}_{\pi, h} (b^{\prime}_h, b_{h+1};g)\Big)-\Big(\Phi_{\pi, h} (b_h, b_{h+1};g)-
    \Phi_{\pi, h} (b^{\prime}_h, b_{h+1};g)\Big)\right|\notag\\
    &\qquad =\left|\hat{\mathbb{E}}_{\pi^b}[\ell(g(A_h,Z_h),I_h)]-\mathbb{E}_{\pi^b}[\ell(g(A_h,Z_h),I_h)]\right|\notag\\
    &\qquad\leq 18L\|g\|_2\sqrt{\frac{M_{\mathbb{G}}^2\cdot \log(4c_1|\mathbb{B}|^3|\Pi(\cH)|H/\zeta^{\prime})}{c_2n}}+\frac{18L\cdot M_{\mathbb{G}}^2\cdot \log\big(4c_1|\mathbb{B}|^3|\Pi(\cH)|H/\zeta^{\prime}\big)}{c_2n},\quad \forall g\in\mathbb{G},
\end{align}
where $\zeta^{\prime} = \min\{\delta,4c_1|\mathbb{B}|^3|\Pi(\cH)|H\exp(-c_2n\alpha_{\mathbb{G},n}^2/M_{\mathbb{G}}^2)\}$.
 Applying a union bound argument over $b_h,b_h^{\prime},b_{h+1}\in\mathbb{B}$, $\pi\in\Pi(\cH)$, and $h\in[H]$, we have that \eqref{eq: proof lem 3.4 concentraion 2} holds for any $b_h,b_h^{\prime},b_{h+1}\in\mathbb{B}$, $g\in\mathbb{G}$, $\pi\in\Pi(\cH)$, and $h\in[H]$ with probability at least $1-\delta/4$.
 Finally, for simplicity, we denote that
 \begin{align}\label{eq: iota}
     \iota_n \coloneqq\sqrt{\frac{M_{\mathbb{G}}^2 \cdot \log(4c_1|\mathbb{B}|^3|\Pi(\cH)|H/\zeta^{\prime})}{c_2n}}, \quad \iota_n^{\prime}\coloneqq\sqrt{\frac{M_{\mathbb{G}}^2\cdot \log(4c_1/\zeta)}{2c_2n}}
 \end{align}

Now we are ready to prove the lower bound on term $(\star)$.
For simplicity, given fixed $b_h,b_{h+1}\in\mathbb{B}$, we denote 
\begin{align*}
    g_h^{\pi}\coloneqq\frac{1}{2\lambda}\ell_h^{\pi}(b_h,b_{h+1})\in\mathbb{G},
\end{align*}
where $\ell_h^\pi$ is defined in \eqref{eq: ell} and $g_h^{\pi}\in\mathbb{G}$ due to Assumption \ref{assump: completeness and realizability}.
Now consider that 
\begin{align*}%\label{eq: proof lem 3.4 1}
    (\star)&=\max_{g \in \mathbb{G}} \Big\{
    \hat{\Phi}_{\pi, h} (b_h, b_{h+1};g)-
    \hat{\Phi}_{\pi, h} (b^{\star}_h(b_{h+1}), b_{h+1};g)-2\lambda\|g\|_{2,n}^2\Big\}\notag\\
    &\geq 
    \hat{\Phi}_{\pi, h} (b_h, b_{h+1};g_h^{\pi}/2)-
    \hat{\Phi}_{\pi, h} (b^{\star}_h(b_{h+1}), b_{h+1}; g_h^{\pi}/2)-\frac{\lambda}{2}\| g_h^{\pi}\|_{2,n}^2,
\end{align*}
where the inequality follows from the fact that $\mathbb{G}$ is star-shaped and consequently $g_h^{\pi}/2 \in\mathbb{G}$.
Then by applying concentration inequality \eqref{eq: proof lem 3.4 concentraion 1} and \eqref{eq: proof lem 3.4 concentraion 2},
we have that 
\begin{align}\label{eq: lower bound}
    (\star)&\geq \Phi_{\pi, h} (b_h, b_{h+1}; g_h^{\pi}/2)-
    \Phi_{\pi, h} (b^{\star}_h(b_{h+1}), b_{h+1}; g_h^{\pi}/2)-18L\iota_n\|g_h^{\pi}\|_2
    -18L\iota_n^2 \notag
    \\ & \qquad-\frac{\lambda}{2}\left(\frac{3}{2}\|g_h^{\pi}\|_2^2+\iota_n^{\prime 2}\right)\notag\\
    &\geq\lambda\|g_h^{\pi}\|_2^2-18L\iota_n\|g_h^{\pi}\|_2-\epsilon_{\mathbb{B}}^{1/2}M_{\mathbb{G}}
    -18L\iota_n^2-\frac{\lambda}{2}\left(\frac{3}{2}\|g_h^{\pi}\|_2^2+\iota_n^{\prime 2}\right)\notag\\
    &=\frac{\lambda}{4}\|g_h^{\pi}\|_2^2-18L\iota_n\|g_h^{\pi}\|_2-18L\iota_n^2-\frac{\lambda}{2}\iota_n^{\prime 2}-\epsilon_{\mathbb{B}}^{1/2}M_{\mathbb{G}},
\end{align} 
where the second inequality follows from  that $\Phi_{\pi, h} (b^{\star}(b_{h+1}), b_{h+1}; g_h^{\pi}/2)\leq \epsilon_{\mathbb{B}}^{1/2}M_{\mathbb{G}}$ (we prove this inequality by \eqref{eq: proof lemma b1 3} in the proof of Lemma \ref{lem: Phi b star bound}) and the fact that
\begin{align*}
    \Phi_{\pi, h} (b_h, b_{h+1}; g_h^{\pi}/2) = \frac{1}{4\lambda}\mathbb{E}_{\pi^b}[\ell_h^{\pi}(b_h,b_{h+1})(A_h,Z_h)^2]=\lambda\|g_h^{\pi}\|_2^2.
\end{align*}

\vspace{3mm}
\noindent
\textbf{Combining upper bound and lower bound of term ($\star$).}
Now we are ready to combine the upper bound and lower bound of $(\star)$ to derive the bound on $\mathcal{L}_h^{\pi}(b_h,b_{h+1})$. By combining upper bound \eqref{eq: upper bound} and lower bound \eqref{eq: lower bound}, we have that with probability at least $1-\delta$, for any $b_{h},b_{h+1}\in\mathbb{B}$, $\pi\in\Pi(\cH)$, and $h\in[H]$,
\begin{align*}
% \label{eq: proof lem accuracy of CR eq 5}
    \frac{\lambda}{4}\|g_h^{\pi}\|_2^2-18L\iota_n\|g_h^{\pi}\|_2-18L\iota_n^2-\frac{\lambda}{2}\iota_n^{\prime 2}-\epsilon_{\mathbb{B}}^{1/2}M_{\mathbb{G}}\leq 3\xi+2\epsilon_{\mathbb{B}}^{1/2}M_{\mathbb{G}},
\end{align*}
This gives a quadratic inequality on $\|g_h^{\pi}\|_2$, i.e.,
\begin{align*}
    \lambda\|g_h^{\pi}\|_2^2-\underbrace{72L\iota_n}_{\displaystyle{(\mathrm{A})}}\|g_h^{\pi}\|_2-\underbrace{4\left(18L\iota_n^2+\frac{\lambda}{2}\iota_n^{\prime 2}+3\xi+3\epsilon_{\mathbb{B}}^{1/2}M_{\mathbb{G}}\right)}_{\displaystyle{(\mathrm{B})}}\leq 0.
    % \|g_h^{\pi}\|_2^2-\underbrace{\left(\Big(36L+\frac{3}{2}\lambda\Big)\iota_n+\frac{\lambda\iota_n^{\prime 2}}{\iota_n}+\frac{6\xi}{\iota_n}+2\epsilon_{\mathbb{B}}^{1/2}M_{\mathbb{G}}\right)}_{\displaystyle{\coloneqq A}}\cdot\|g_h^{\pi}\|_2-\underbrace{\epsilon_{\mathbb{B}}^{1/2}M_{\mathbb{G}}}_{\displaystyle{\coloneqq B}}\leq 0.
\end{align*}
By solving this quadratic equation, 
we have that 
\begin{align*}
    \|g_h^{\pi}\|_2&\leq\frac{1}{2\lambda}\mathrm{A}+\frac{1}{2\lambda}\sqrt{\mathrm{A}^2+
    4\mathrm{B}}\leq \frac{\mathrm{A}}{\lambda}+\frac{\sqrt{\mathrm{B}}}{\lambda}.
\end{align*}
Applying the definition of $\mathrm{A}$ and $\mathrm{B}$, we conclude that, with probability at least $1-\delta$,
\begin{align*}
    \|g_h^{\pi}\|_2&\leq\frac{72}{\lambda}L\iota_n+\frac{2}{\lambda}\left(18L\iota_n^2+\frac{\lambda}{2}\iota_n^{\prime 2}+3\xi+3\epsilon_{\mathbb{B}}^{1/2}M_{\mathbb{G}}\right)^{1/2}\notag\\
    &\leq  \frac{72}{\lambda}L\iota_n+\frac{6\sqrt{2}}{\lambda}L^{1/2}\iota_n+\frac{\sqrt{2}}{\sqrt{\lambda}}\iota_n^{\prime}+\frac{2\sqrt{3}}{\lambda}\xi^{1/2}+\frac{2\sqrt{3}}{\lambda}\epsilon_{\mathbb{B}}^{1/4}M_{\mathbb{G}}^{1/2}\notag
\end{align*}
Therefore, we can bound the RMSE loss $\mathcal{L}_h^{\pi}(b_h,b_{h+1})$ by 
\begin{align*}
    \sqrt{\mathcal{L}_h^{\pi}(b_h,b_{h+1})}=2\lambda\|g_h^{\pi}\|_2\leq (144L+12\sqrt{2}L^{1/2})\iota_n+2\sqrt{2\lambda}\iota_n^{\prime}+4\sqrt{3}\xi^{1/2}+4\sqrt{3}\epsilon_{\mathbb{B}}^{1/4}M_{\mathbb{G}}^{1/2}.\notag
\end{align*}
Plugging in the definition of $\iota_n,\iota_n^{\prime}$ in \eqref{eq: iota}, $\xi$ in Lemma \ref{lem: accuracy of CR}, and that $L=2M_{\mathbb{B}}$, we have that 
\begin{align*}
    &\sqrt{\mathcal{L}_h^{\pi}(b_h,b_{h+1})}
    \\ &\qquad \leq(144L+12\sqrt{2 L })\cdot\sqrt{\frac{M_{\mathbb{G}}^2 \cdot \log(4c_1|\mathbb{B}|^3|\Pi(\cH)|H/\zeta^{\prime})}{c_2n}}+2\sqrt{2\lambda}\cdot\sqrt{\frac{M_{\mathbb{G}}^2\cdot \log(4c_1/\zeta)}{2c_2n}}\notag\\
    &\qquad\qquad+4\sqrt{3}\cdot\sqrt{\frac{C_1(\lambda+1/\lambda)\cdot M_{\mathbb{B}}^2M_{\mathbb{G}}^2\cdot \log(|\mathbb{B}||\Pi(\cH)|H/\zeta^{\prime})}{n}}+4\sqrt{3}\cdot\epsilon_{\mathbb{B}}^{1/4}M_{\mathbb{G}}^{1/2}\notag\\
    &\qquad\leq \tilde{C}_1M_{\mathbb{B}}M_{\mathbb{G}}\sqrt{\frac{(\lambda+1/\lambda)\cdot \log(|\mathbb{B}||\Pi(\cH)|H/\zeta)}{n}}+\tilde{C}_1\epsilon_{\mathbb{B}}^{1/4}M_{\mathbb{G}}^{1/2}.
\end{align*}
for some problem-independent constant $\tilde{C}_1>0$ and $\zeta = \min\{\delta, 4c_1\exp(-c_2n\alpha_{\mathbb{G},n}^2/M_{\mathbb{G}}^2)\}$.
Here in the second inequality we have used the fact that $\zeta<\zeta^{\prime}$.
This finishes the proof of Lemma \ref{lem: accuracy of CR}.
\end{proof}

\subsection{Proof of Lemma \ref{lem: Phi b star bound}}\label{subsec: proof Phi b star bound}
\begin{proof}[Proof of Lemma \ref{lem: Phi b star bound}]
Following the proof of Lemma \ref{lem: true in CR}, we first relate  $\hat{\Phi}_{\pi, h}^\lambda ( b_h, b_{h+1};g)=\hat \Phi_{\pi,h}(b_h,b_{h+1};g)-\lambda\|g\|_{2,n}^2$ and its population version $\Phi_{\pi,h}^{\lambda}(b_h,b_{h+1};g)$ via two localized uniform concentration inequalities. On the one hand, to relate $\|g\|_2^2$ and $\|g\|_{2,n}^2$, by Lemma \ref{lem: localized uniform concentration} (Theorem 14.1 of \cite{wainwright2019high}), for some absolute constants $c_1,c_2>0$, it holds with probability at least $1-\delta/4$ that 
\begin{align}\label{eq: proof lemma b1 concentration 1}
    \left|\|g\|_{2,n}^2-\|g\|_2^2\right|\leq \frac{1}{2}\|g\|_2^2+\frac{M_{\mathbb{G}}^2\cdot \log(4c_1/\zeta)}{2c_2n},\quad \forall g\in\mathbb{G},
\end{align}
where $\zeta = \min\{\delta, 4c_1\exp(-c_2n\alpha_{\mathbb{G},n}^2/M_{\mathbb{G}}^2)\}$ and $\alpha_{\mathbb{G},n}$ is the critical radius of function class $\mathbb{G}$ defined in Assumption \ref{assump: dual function class}.
On the other hand, to relate $\hat \Phi_{\pi,h}(b_h,b_{h+1};g)$ and $\Phi_{\pi,h}(b_h,b_{h+1};g)$, we invoke Lemma \ref{lem: localized uniform concentration 2} (Lemma 11 of \citep{foster2019orthogonal}).
Specifically, for any given $b_h,b_{h+1}\in\mathbb{B}$, $\pi\in\Pi(\cH)$, and step $h$, in Lemma \ref{lem: localized uniform concentration 2} we choose $\mathcal{F}=\mathbb{G}$, $\mathcal{X}=\mathcal{A}\times\mathcal{Z}$, $\mathcal{Y}=\mathcal{I}_h$, and loss function $\ell(g(A_h,Z_h),I_h)\coloneqq\varsigma_h^{\pi}(b_h,b_{h+1})(I_h)g(A_h,Z_h)$ where $\ell^{\pi}_h$ is defined in \eqref{eq: ell} and $I_h\in\mathcal{I}_h$ is defined in the beginning of Appendix \ref{sec: proof of Lemmas in proof sketch}.
We can see that $\ell$ is $L$-Lipschitz continuous in the first argument since for any $g,g^{\prime}\in\mathbb{G}$,
$(A_h,Z_h)\in\mathcal{A}\times\mathcal{Z}$, it holds that
\begin{align*}
    \big|\ell(g(A_h,Z_h),I_h)-\ell(g^{\prime}(A_h,Z_h),I_h)\big|&=|\varsigma_h^{\pi}(b_h,b_{h+1})(I_h)|\cdot|g(A_h,Z_h)-g^{\prime}(A_h,Z_h)|\notag\\
    &\leq2M_{\mathbb{B}}\cdot |g(A_h,Z_h)-g^{\prime}(A_h,Z_h)|,\notag
\end{align*}
which indicates that $L=2M_{\mathbb{B}}$.
Now setting $f^{\star}=0$ in Lemma \ref{lem: localized uniform concentration 2}, we have that $\delta_n$ in Lemma \ref{lem: localized uniform concentration 2} coincides with $\alpha_{\mathbb{G},n}$ in Assumption \ref{assump: dual function class}. 
Then we can conclude that for some absolute constants $c_1,c_2>0$, it holds with probability at least $1-\delta/(4|\mathbb{B}|^2|\Pi(\cH)|H)$ that, for all $g\in\mathbb{G}$,
\begin{align}\label{eq: proof lemma b1 concentration 2}
    &\left|\hat \Phi_{\pi,h}(b_h,b_{h+1};g)-\Phi_{\pi,h}(b_h,b_{h+1};g)\right|\notag\\&\qquad =\left|\hat{\mathbb{E}}_{\pi^b}[\ell(g(A_h,Z_h),A_h,Z_h)]-\mathbb{E}_{\pi^b}[\ell(g(A_h,Z_h),A_h,Z_h)]\right|\notag\\
    &\qquad\leq 18L\|g\|_2\sqrt{\frac{M_{\mathbb{G}}^2\cdot \log\big(4c_1|\mathbb{B}|^2|\Pi(\cH)|H/\zeta\big)}{c_2n}}+\frac{18L \cdot M_{\mathbb{G}}^2\cdot \log\big(4c_1|\mathbb{B}|^2|\Pi(\cH)|H/\zeta\big)}{c_2n}, 
\end{align}
where $\zeta = \min\{\delta,4c_1|\mathbb{B}|^2|\Pi(\cH)|H\exp(-c_2n\alpha_{\mathbb{G},n}^2/M_{\mathbb{G}}^2)\}$.
 Applying a union bound argument over $b_h,b_{h+1}\in\mathbb{B}$, $\pi\in\Pi(\cH)$, and $h\in[H]$, we then have that \eqref{eq: proof lem 3.3 concentraion 2} holds for any $b_h,b_{h+1}\in\mathbb{B}$, $g\in\mathbb{G}$, $\pi\in\Pi(\cH)$, and $h\in[H]$ with probability at least $1-\delta/4$.
Now using these two concentration inequalities \eqref{eq: proof lemma b1 concentration 1} and \eqref{eq: proof lemma b1 concentration 2}, we can further deduce that, for some absolute constants $c_1,c_2>0$, with probability at least $1-\delta/2$,
\begin{align*}
    &\max_{g \in \mathbb{G}} \hat{\Phi}_{\pi, h}^\lambda ( b_h^{\star}(b_{h+1}), b_{h+1}^{\pi};g)\notag\\
    &\qquad =\max_{g \in \mathbb{G}}\left\{\hat{\Phi}_{\pi, h} ( b_h^{\star}(b_{h+1}), b_{h+1};g)-\lambda\|g\|_{2,n}^2\right\}\notag\\
    &\qquad \leq \max_{g \in \mathbb{G}}\Bigg\{\Phi_{\pi, h} (b_h^{\star}(b_{h+1}), b_{h+1};g)-\lambda\|g\|_{2}^2
    +\frac{\lambda}{2}\|g\|_2^2+\frac{\lambda M_{\mathbb{G}}^2\log(4c_1/\zeta)}{2c_2n},\notag\\
    &\qquad\qquad +18L\|g\|_2\sqrt{\frac{M_{\mathbb{G}}^2\cdot \log(4c_1|\mathbb{B}|^2|\Pi(\cH)|H/\zeta^{\prime})}{c_2n}}+\frac{18L\cdot M_{\mathbb{G}}^2\cdot \log(4c_1|\mathbb{B}|^2|\Pi(\cH)|H/\zeta^{\prime})}{c_2n}\Bigg\}\notag\\
    &\qquad \leq \max_{g\in\mathbb{G}}\Phi_{\pi,h}( b_h^{\star}(b_{h+1}), b_{h+1};g)+\max_{g\in\mathbb{G}}\Bigg\{
    -\frac{\lambda}{2}\|g\|_2^2+18L\|g\|_2\sqrt{\frac{M_{\mathbb{G}}^2\cdot \log(4c_1|\mathbb{B}|^2|\Pi(\cH)|H/\zeta^{\prime})}{c_2n}}\Bigg\}\notag\\
    &\qquad\qquad +\frac{\lambda M_{\mathbb{G}}^2\log(4c_1/\zeta)}{2c_2n}+ \frac{18L\cdot M_{\mathbb{G}}^2\cdot\log(4c_1|\mathbb{B}|^2|\Pi(\cH)|H/\zeta^{\prime})}{c_2n}
    \notag\\
    &\qquad
    \leq \epsilon_{\mathbb{B}}^{1/2}M_{\mathbb{G}}+\frac{728L^2\cdot M_{\mathbb{G}}^2\cdot \log(4c_1|\mathbb{B}|^2|\Pi(\cH)|H/\zeta^{\prime})}{\lambda n} + \frac{\lambda M_{\mathbb{G}}^2 \cdot \log(4c_1/\zeta)}{2c_2n}
    \\& \qquad\qquad + \frac{18L\cdot sM_{\mathbb{G}}^2\cdot\log(4c_1|\mathbb{B}|^2|\Pi(\cH)|H/\zeta^{\prime})}{c_2n},
\end{align*}
where $\zeta = \min\{\delta, 4c_1\exp(-c_2n\alpha_{\mathbb{G},n}^2/M_{\mathbb{G}}^2)\}$ and $\zeta^{\prime} = \min\{\delta,4c_1|\mathbb{B}|^2|\Pi(\cH)|H\exp(-c_2n\alpha_{\mathbb{G},n}^2/M_{\mathbb{G}}^2)\}$
for any policy $\pi\in\Pi(\cH)$ and step $h\in[H]$.
Here the last inequality holds from the fact that \begin{align}\label{eq: proof lemma b1 3}
\max_{g\in\mathcal{G}}\Phi_{\pi, h} ( b_h^{\star}(b_{h+1}), b_{h+1};g)\leq \epsilon^{1/2}_{\mathbb{B}}M_{\mathbb{G}},
\end{align}
and that
$\sup_{\|g\|_2}\{a\|g\|_2-b\|g\|_2^2\}\leq a^2/4b$.
Note that inequality \eqref{eq: proof lemma b1 3}
holds according to Assumption \ref{assump: completeness and realizability} and \ref{assump: completeness and realizability}.
In fact, by Assumption \ref{assump: completeness and realizability}, we can first obtain by quadratic optimization that for $\lambda>0$,
\begin{align*}
    \max_{g\in\mathbb{G}}\Phi_{\pi,h}^{\lambda}(b_h,b_{h+1})=\frac{1}{4\lambda}\mathcal{L}_h^{\pi}(b_h,b_{h+1}),
\end{align*}
for any functions $b_h,b_{h+1}\in\mathbb{B}$. Thus we can equivalently express $b^{\star}_{h}(b_{h+1})$ as 
\begin{align*}
    b^{\star}_{h}(b_{h+1})=\argmin_{b\in\mathbb{B}}\frac{1}{4\lambda}\mathcal{L}^{\pi}_h(b,b_{h+1})=\argmin_{b\in\mathbb{B}}\mathcal{L}^{\pi}_h(b,b_{h+1}).
\end{align*}
This further indicates the following bound on $\max_{g\in\mathbb{G}}\Phi_{\pi,h}(b^{\star}_{h}(b_{h+1}),b_{h+1};g)$ that
\begin{align*}
    \max_{g\in\mathbb{G}}\Phi_{\pi,h}(b^{\star}_{h}(b_{h+1}),b_{h+1};g)\leq\max_{g\in\mathbb{G}}\sqrt{\mathcal{L}_h(b^{\star}_h(b_{h+1}),b_{h+1})\cdot\mathbb{E}_{\pi^b}[g(A_h,Z_h)^2]}\leq  \epsilon_{\mathbb{B}}^{1/2}M_{\mathbb{G}},
\end{align*}
by Cauchy-Schwarz inequality and Assumption \ref{assump: completeness and realizability}.
Now according to the choice of $\xi$ in Lemma \ref{lem: true in CR}, using the fact that $\zeta<\zeta^{\prime}$ and $L=2M_{\mathbb{B}}$, we can conclude that, with probability at least $1-\delta/2$,
\begin{align*}
&\max_{g \in \mathbb{G}} \hat{\Phi}_{\pi, h}^\lambda ( b_h^{\star}(b_{h+1}), b_{h+1}^{\pi};g)\notag\\
    &\qquad\leq\frac{728L^2\cdot  M_{\mathbb{G}}^2\cdot \log(4c_1|\mathbb{B}|^2|\Pi(\cH)|H/\zeta^{\prime})}{\lambda n} + \frac{\lambda M_{\mathbb{G}}^2\cdot \log(4c_1/\zeta)}{2c_2n}
    \\ & \qquad\qquad+ \frac{18L\cdot M_{\mathbb{G}}^2\cdot \log(4c_1|\mathbb{B}|^2|\Pi(\cH)|H/\zeta^{\prime})}{c_2n}+\epsilon_{\mathbb{B}}^{1/2}M_{\mathbb{G}}\notag\\
    &\qquad\lesssim\mathcal{O}\left(\frac{(\lambda+1/\lambda)\cdot M_{\mathbb{B}}^2M_{\mathbb{G}}^2\cdot \log(|\mathbb{B}||\Pi(\cH)|H/\zeta)}{n}\right)+\epsilon_{\mathbb{B}}^{1/2}M_{\mathbb{G}} \lesssim\xi+\epsilon_{\mathbb{B}}^{1/2}M_{\mathbb{G}}.
\end{align*}
Therefore, we conclude the proof of Lemma \ref{lem: Phi b star bound}.
\end{proof}

\section{Proof of Theorem \ref{thm: suboptimality}}\label{sec: suboptimality proof}
\begin{proof}[Proof of Theorem \ref{thm: suboptimality}]
By the definition of $F(\bbb)$ and $\hat F(\bbb)$ in \eqref{eq: define J} and the fact that $J(\pi)=F(\bbb^{\pi})$ according to Theorem \ref{thm: identification}, we first  have that
% By the definition of $J$ and $\hat{J}$ in \eqref{eq: define J} and the definition of $\hat{\pi}$ in \eqref{eq: def hat pi}, we have that
\begin{align*}
    &J(\pi^{\star})-J(\hat{\pi})
    \\ &\qquad =F(\bbb^{\pi^{\star}})-F(\bbb^{\hat{\pi}})\notag
    \\ &\qquad =\underbrace{\big(F(\bbb^{\pi^{\star}})-\hat F(\bbb^{\pi^{\star}})\big)}_{\displaystyle{\textnormal{(i)}}}+\underbrace{\big(F(\bbb^{\pi^{\star}})-\hat F(\bbb^{\hat{\pi}})\big)}_{\displaystyle{\textnormal{(ii)}}}+
    \underbrace{\big(\hat F(\bbb^{\hat \pi})-F(\bbb^{\hat{\pi}})\big)}_{\displaystyle{\textnormal{(iii)}}}.
\end{align*}
We can bound term (i) and term (iii) via uniform concentration inequalities, which we present latter. For term (ii), via Lemma \ref{lem: true in CR}, with probability at least $1-\delta$, $\bbb^{\pi^{\star}}\in\CR^{\pi^{\star}}(\xi)$ and $\bbb^{\hat \pi}\in\CR^{\hat \pi}(\xi)$, which indicates that 
\begin{align}\label{eq: proof thm 4.1 eq 2}
    \textnormal{(ii)}=\hat F(\bbb^{\pi^{\star}})-\hat F(\bbb^{\hat{\pi}})\leq \max_{\bbb\in\CR^{\pi^{\star}}(\xi)}\hat F(\bbb)-\min_{\bbb\in\CR^{\hat \pi}(\xi)}\hat F(\bbb).
\end{align}
From \eqref{eq: proof thm 4.1 eq 2}, we can further bound term (ii) as
\begin{align}\label{eq: proof thm 4.1 eq 3}
    \textnormal{(ii)}&\leq \max_{\bbb\in\CR^{\pi^{\star}}(\xi)}\hat F(\bbb)-\max_{\pi\in\Pi(\mathcal{H})}\min_{\bbb\in\CR^{ \pi}(\xi)}\hat F(\bbb)\notag\\
    &\leq \max_{\bbb\in\CR^{\pi^{\star}}(\xi)}\hat F(\bbb)-\min_{\bbb\in\CR^{ \pi^{\star}}(\xi)}\hat F(\bbb)\notag\\
    &=\max_{\bbb\in\CR^{\pi^{\star}}(\xi)}\hat F(\bbb)-\hat F(\bbb^{\pi^{\star}})+\hat F(\bbb^{\pi^{\star}})-\min_{\bbb\in\CR^{ \pi^{\star}}(\xi)}\hat F(\bbb)\notag\\
    &\leq 2\max_{\bbb\in\CR^{\pi^{\star}}(\xi)}\left|\hat F(\bbb)-\hat F(\bbb^{\pi^{\star}})\right|.
\end{align}
Here the first inequality holds because $\max_{\pi\in\Pi(\mathcal{H})}\min_{\bbb\in\CR^{ \pi}(\xi)}\hat F(\bbb) = \min_{\bbb\in\CR^{\hat \pi}(\xi)}\hat F(\bbb)$ by the definition of $\hat\pi$ from \eqref{eq: def hat pi}. The second inequality holds because by definition $\pi^{\star}$ is the optimal policy in $\Pi(\cH)$. The third inequality is trivial. 
Now to further bound \eqref{eq: proof thm 4.1 eq 3} by the RMSE loss defined in \eqref{eq: RMSE loss}, we consider
\begin{align*}
    &2\max_{\bbb\in\CR^{\pi^{\star}}(\xi)}\left|\hat F(\bbb)-\hat F(\bbb^{\pi^{\star}})\right|\notag\\
    &\qquad \leq \underbrace{2\max_{\bbb\in\CR^{\pi^{\star}}(\xi)}\left|\hat F(\bbb)- F(\bbb)\right|}_{\displaystyle{\textnormal{(iv)}}}+\underbrace{2\max_{\bbb\in\CR^{\pi^{\star}}(\xi)}\left| F(\bbb)- F(\bbb^{\pi^{\star}})\right|}_{\displaystyle{\textnormal{(v)}}}+\underbrace{2\left| F(\bbb^{\pi^{\star}})-\hat F(\bbb^{\pi^{\star}})\right|}_{\displaystyle{\textnormal{(vi)}}}, 
\end{align*}
where we can bound term (iv) and term (vi) via uniform concentration inequalities, which we present latter. For term (v), we invoke Lemma \ref{lem: regret decomposition} and obtain that 
\begin{align*}
    \textnormal{(v)}&\leq 2\max_{\bbb\in\CR^{\pi^{\star}}(\xi)} \sum_{h=1}^H\gamma^{h-1}\sqrt{C^{\pi^{\star}}}\cdot\sqrt{\mathcal{L}_h^{\pi^{\star}}(b_h,b_{h+1})}\leq 2\sqrt{C^{\pi^{\star}}}\sum_{h=1}^H\gamma^{h-1}\max_{\bbb\in\CR^{\pi^{\star}}(\xi)}\sqrt{\mathcal{L}_h^{\pi^{\star}}(b_h,b_{h+1})}.
\end{align*}
Now invoking Lemma \ref{lem: accuracy of CR}, with probability at least $1-\delta$, $\max_{\bbb\in\CR^{\pi^{\star}}(\xi)}\sqrt{\mathcal{L}_h^{\pi^{\star}}(b_h,b_{h+1})}$ is bounded by 
\begin{align}\label{eq: proof theorem 4.1 rmse}
    \max_{\bbb\in\CR^{\pi^{\star}}(\xi)}\sqrt{\mathcal{L}_h^{\pi}(b_h,b_{h+1})}\leq \tilde{C}_1M_{\mathbb{B}}M_{\mathbb{G}}\sqrt{\frac{\left(\lambda+1/\lambda\right)\log(|\mathbb{B}||\Pi(\cH)|H/\zeta)}{n}}+\tilde{C}_1\epsilon_{\mathbb{B}}^{1/4}M_{\mathbb{G}}^{1/2},
\end{align}
for each step $h\in[H]$, where $\zeta=\min\{\delta,c_1\exp(-c_2n\alpha_{\mathbb{G},n}^2)\}$.
In the sequel, we turn to deal with term (i), (iii), (iv), and (vi), respectively.
To this end, it suffices to apply uniform concentration inequalities to bound $F(\mathbf{b})$ and $\hat{F}(\mathbf{b})$ uniformly over $\mathbf{b}\in\mathbb{B}^{\otimes H}$.
By Hoeffding inequality, we have that, with probability at least $1-\delta$, 
\begin{align}\label{eq: proof theorem 4.1 concentration}
    \left|J(\pi,\mathbf{b})-\hat{J}(\pi,\mathbf{b})\right|&\leq \sqrt{\frac{2M_{\mathbb{B}}^2\log(|\mathbb{B}|/\delta)}{n}},\quad\forall \pi\in\Pi(\cH),\quad\forall \bbb\in\mathbb{B}^{\otimes H}.
\end{align}
Consequently, all of (i), (iii), (iv), and (vi) are bounded by the right hand side of \eqref{eq: proof theorem 4.1 concentration}.
Finally, by combining \eqref{eq: proof theorem 4.1 rmse} and \eqref{eq: proof theorem 4.1 concentration}, with probability at least $1-3\delta$, it holds that
\begin{align*}
    J(\pi^{\star})-J(\hat{\pi}) &\leq \textnormal{(i)}+\textnormal{(iii)}+\textnormal{(iv)}+\textnormal{(vi)}+\textnormal{(v)}\notag\\
    &\leq 2\sqrt{C^{\pi^{\star}}}\sum_{h=1}^H\gamma^{h-1}\left(\tilde{C}_1M_{\mathbb{B}}M_{\mathbb{G}}\sqrt{\frac{\left(\lambda+1/\lambda\right)\log(|\mathbb{B}||\Pi(\cH)|H/\zeta)}{n}}+\tilde{C}_1\epsilon_{\mathbb{B}}^{1/4}M_{\mathbb{G}}^{1/2}\right)\\
    &\qquad +4\sqrt{\frac{2M_{\mathbb{B}}^2\log(|\mathbb{B}|/\delta)}{n}}\notag\\
    &\leq C_1^{\prime}\sqrt{C^{\pi^{\star}}}\left(\lambda+1/\lambda\right)^{1/2}HM_{\mathbb{B}}M_{\mathbb{G}}\sqrt{\frac{\log(|\mathbb{B}||\Pi(\cH)|H/\zeta)}{n}}+C_1^{\prime}\sqrt{C^{\pi^{\star}}}H\epsilon_{\mathbb{B}}^{1/4}M_{\mathbb{G}}^{1/2},
\end{align*}
for some problem-independent constant $C_1^{\prime}>0$.
We finish the proof of Theorem \ref{thm: suboptimality} by taking $\lambda=1$.
\end{proof}

\section{Proofs for Section \ref{sec: linear case}: Linear Function Approximation}\label{sec: proof of linear case}

\subsection{Auxiliary Results for Linear Function Approximation}\label{sec: basics for LFA}

Here we present results that bound the complexity of certain functions classes in the case of linear function approximation~(Definition \ref{def: LFA}).

Recall the definition of the bridge function class $\BB^{\otimes H}$ where $\BB = \Blin$ is defined as  
\begin{align*}
    \Blin \coloneqq \left\{b  \ \middle| \ b(\cdot,\cdot) = \langle\bphi(\cdot,\cdot), \theta \rangle, \theta \in \RR^d,\|\theta\|_2 \leq L_b,\ \sup_{w\in\mathcal{W}}\big|\sum_{a\in\mathcal{A}}b(a,w)\big|\leq M_{\mathbb{B}} \right\}.
\end{align*}

Denote by $\cN_{\epsilon}^{\infty}(\BB)$ the $\epsilon$-covering number of $\BB$ with respect to the $\ell_\infty$ norm. That is, there exists a collection of functions $\{b_i\}_{i=1}^N$ with $N \leq \cN_{\epsilon}^{\infty}(\BB)$ such that for any $b \in \BB$, we can find some $b' \in \{b_i\}_{i=1}^N$ satisfying \begin{align*}
    \|b - b'\|_\infty \coloneqq \sup_{a \in \cA, w\in\cW} |b(a,w) - b'(a,w)| \leq \epsilon. 
\end{align*}
Recall the policy function class $\Pi(\cH) = \Pilin^{\otimes H}$ where $\Pilin$ is defined as
\begin{align*}
    \Pilin \coloneqq \left\{ \pi \ \middle| \ \pi(a|o, \tau) = \frac{e^{\langle \bpsi(a,o,\tau),\beta\rangle }}{\sum_{a' \in \cA} e^{\langle \bpsi(a',o,\tau),\beta\rangle } }, \ \beta\in\RR^d, \ \|\beta\|_2\leq L_\pi \right\}. 
\end{align*}
Denote by $\cN_{\epsilon}^{\infty,1} (\Pilin)$ the $\epsilon$-covering number of $\Pilin$ with respect to the $\ell_{\infty,1}$ norm, i.e., 
\begin{align*}
    \|\pi - \pi'\|_{\infty,1} \coloneqq \sup_{o\in\cO,\tau\in\cH} \sum_{a \in \cA}|\pi(a|o,\tau) - \pi'(a|o,\tau)|. 
\end{align*}
The upper bounds for these covering numbers are given by the following lemma.

\begin{lemma}[Lemma 6 in \citealt{zanette2021provable}]\label{lem: covering B and Pi}
    For any $\epsilon \in (0,1)$, we have 
    \begin{align*}
        \log \cN_{\epsilon}^{\infty}(\mathbb{B}) & \leq d \log \left(1 + \frac{2L_b}{\epsilon} \right) , 
        \\ \log \cN_{\epsilon}^{\infty,1} (\Pilin) & \leq d \log \left( 1 + \frac{16L_\pi}{\epsilon} \right). 
    \end{align*}
\end{lemma}

\paragraph{The $\epsilon$-nets for the product function classes} 
In the rest of Appendix \ref{sec: proof of linear case}, due to the proof, we need to consider $\epsilon$-nets defined for the product function classes $\BB^{\otimes H}$ and $\Pi(\cH) = \Pilin^{\otimes H}$.   
Specifically, for $\BB^{\otimes H}$, we consider an $\epsilon$-net of $\BB^{\otimes H}$ defined in the following way: for any $\bbb=\{b_h\}_{h=1}^H \in \BB^{\otimes H}$, there exists an $\bbb'=\{b'_h\}_{h=1}^H$ in the $\epsilon$-net, such that 
\begin{align*}
% \label{eq: def of epsilon net BH}
    \|b_h-b'_h\|_\infty \leq \epsilon.
\end{align*}
By Lemma \ref{lem: covering B and Pi}, the cardinality of this $\epsilon$-net is upper bounded by 
\begin{align*}
% \label{eq: covering number of BH}
    \log \cN_{\epsilon}^{\infty}(\mathbb{B}^{\otimes H}) & \leq d H \log \left(1 + \frac{2L_b}{\epsilon} \right) .
\end{align*}

Similarly, we consider an $\epsilon$-net defined for $\Pi(\cH)$ defined as the following: for any $\pi=\{\pi_h\}_{h=1}^H \in \Pi(\cH)$, there exists an $\pi' = \{\pi'_h\}_{h=1}^H$ in the $\epsilon$-net such that 
\begin{align*}
% \label{eq: def of epsilon net PiH}
    \|\pi_{h} - \pi'_h\|_{\infty,1} \leq \epsilon.
\end{align*}Then by Lemma \ref{lem: covering B and Pi}, the cardinality of this $\epsilon$-net is upper bounded by
\begin{align*}
% \label{eq: covering number of PiH}
    \log \cN_{\epsilon}^{\infty,1} (\Pi(\cH)) & \leq d H \log \left( 1 + \frac{16L_\pi}{\epsilon} \right)
\end{align*}

For the dual function class $\Glin$, recall the definition of the critical radius $\alpha_{\mathbb{G},n}$ in Assumption \ref{assump: dual function class}. The next lemma bound the critical radius of the linear dual function class $\mathbb{G} = \Glin$. 

\begin{lemma}[Lemma D.3 in \citealt{duan2021risk}]\label{lem: critical radius of G linear}
    For the function class $\Glin$ defined in Definition \ref{def: LFA}, its critical radius $\alpha_{\mathbb{G},n}$ satisfies
    \begin{align*}
        \alpha_{\mathbb{G},n} = M_{\mathbb{G}} \sqrt{\frac{2d}{n}},
    \end{align*}where $M_{\mathbb{G}} \coloneqq \sup_{g \in \Glin} \|g\|_\infty$. 
\end{lemma}

\subsection{Proof of Corollary~\ref{thm: LFA subopt linear}}\label{sec: proof of LFA subopt}

We first introduce some lemmas needed for proving Corollary \ref{thm: LFA subopt linear}. Their proof is deferred to Appendix \ref{sec: proof of lemma LFA true in CR} and \ref{sec: proof of lemma LFA accuracy of CR}.  

\begin{lemma}[Alternative of Lemma \ref{lem: true in CR} in the linear case]\label{lem: LFA true in CR}
Let the function, policy and dual function class $\BB=\Blin$, $\Pi(\cH)=\Pilin$ and $\mathbb{G}=\Glin$ be defined as in Definition \ref{def: LFA}. 
Then under Assumption \ref{assump: bridge functions exist}, \ref{assump: dual function class}, and \ref{assump: completeness and realizability}, by setting $\xi$ such that
\begin{align*}
    \xi = C_2\cdot \left(\lambda + \frac{1}{\lambda}\right) \cdot \frac{M_{\BB}^2 M_{\GG}^2 dH \log\left({1 + L_b L_\pi H n/\delta}\right) }{n} ,
\end{align*}
for some problem-independent universal constant $C_2>0$, it holds with probability at least $1-\delta$ that 
$\bbb^{\pi}\in\CR^{\pi}(\xi)$ for any policy $\pi\in\Pi(\cH)$.
\end{lemma}

\begin{lemma}[Alternative of Lemma \ref{lem: accuracy of CR} in the linear case]\label{lem: LFA accuracy of CR}
    Under Assumption \ref{assump: bridge functions exist}, \ref{assump: dual function class}, \ref{assump: completeness and realizability}, and  \ref{assump: completeness and realizability}, by setting the same $\xi$ as in Lemma \ref{lem: LFA true in CR}, with probability at least $1-\delta$, for any policy $\pi\in\Pi(\mathcal{H})$, $\bbb\in\CR^{\pi}(\xi)$, and step $h$,
\begin{align*}
    \sqrt{\mathcal{L}^{\pi}_h(b_h,b_{h+1})} \leq \tilde{C}_2 \cdot \left( 1 + \lambda\right)M_{\BB} M_{\GG}\cdot \sqrt{dH\log\left(1+L_b L_{\pi} H n/\delta \right) / n} + \tilde{C}_2 \cdot M_{\GG}^{1/2}\epsilon_{\BB}^{1/4} , 
\end{align*}
for some problem-independent universal constant $\tilde{C}_2>0$.
\end{lemma}

We are now ready to prove Corollary~\ref{thm: LFA subopt linear}. 

\begin{proof}[Proof of Corollary~\ref{thm: LFA subopt linear}]
We follow the proof of Theorem \ref{thm: suboptimality} and write
\begin{align}\label{eq: proof coro subopt eq 1}
    & J(\pi^{\star})-J(\hat{\pi}) \notag
    \\ & \qquad =\underbrace{\big(J(\pi^{\star},\bbb^{\pi^{\star}})-\hat J(\pi^{\star},\bbb^{\pi^{\star}})\big)}_{\displaystyle{\textnormal{(i)}}}+\underbrace{\big(\hat J(\pi^{\star},\bbb^{\pi^{\star}})-\hat J(\hat{\pi},\bbb^{\hat{\pi}})\big)}_{\displaystyle{\textnormal{(ii)}}}+
    \underbrace{\big(\hat J(\hat \pi,\bbb^{\hat \pi})-J(\hat{\pi},\bbb^{\hat{\pi}})\big)}_{\displaystyle{\textnormal{(iii)}}}.
\end{align}

We deal with term (ii) first. By Lemma \ref{lem: LFA true in CR}, with probability at least $1-\delta/2$, $\bbb^{\pi^{\star}}\in\CR^{\pi^{\star}}(\xi)$ and $\bbb^{\hat \pi}\in\CR^{\hat \pi}(\xi)$, which indicates that 
\begin{align*}
% \label{eq: proof coro subopt eq 2}
    \textnormal{(ii)}=\hat J(\pi^{\star},\bbb^{\pi^{\star}})-\hat J(\hat{\pi},\bbb^{\hat{\pi}})\leq \max_{\bbb\in\CR^{\pi^{\star}}(\xi)}\hat J(\pi^{\star},\bbb)-\min_{\bbb\in\CR^{\hat \pi}(\xi)}\hat J(\hat \pi,\bbb).
\end{align*}
Then following \eqref{eq: proof thm 4.1 eq 3}, we can upper bound term \textnormal{(ii)} by
\begin{align}\label{eq: proof coro subopt eq 3}
    \textnormal{(ii)}&\leq  2\max_{\bbb\in\CR^{\pi^{\star}}(\xi)}\left|\hat J(\pi^{\star},\bbb)-\hat J(\pi^{\star},\bbb^{\pi^{\star}})\right| \notag
    \\ & \leq \underbrace{2\max_{\bbb\in\CR^{\pi^{\star}}(\xi)}\left|\hat J(\pi^{\star},\bbb)- J(\pi^{\star},\bbb)\right|}_{\displaystyle{\textnormal{(iv)}}}+\underbrace{2\max_{\bbb\in\CR^{\pi^{\star}}(\xi)}\left| J(\pi^{\star},\bbb)- J(\pi^{\star},\bbb^{\pi^{\star}})\right|}_{\displaystyle{\textnormal{(v)}}}\notag
    \\ & \qquad+\underbrace{2\left| J(\pi^{\star},\bbb^{\pi^{\star}})-\hat J(\pi^{\star},\bbb^{\pi^{\star}})\right|}_{\displaystyle{\textnormal{(vi)}}}.
\end{align}
To bound term (v), we invoke Lemma \ref{lem: regret decomposition} which holds regardless of the underlying function classes and obtain that 
\begin{align*}
    \textnormal{(v)}&=2\max_{\bbb\in\CR^{\pi^{\star}}(\xi)} \sum_{h=1}^H\gamma^{h-1}\sqrt{C^{\pi}}\cdot\sqrt{\mathcal{L}_h^{\pi}(b_h,b_{h+1})} \leq 2\sqrt{C^{\pi^{\star}}}\sum_{h=1}^H\gamma^{h-1}\max_{\bbb\in\CR^{\pi^{\star}}(\xi)}\sqrt{\mathcal{L}_h^{\pi}(b_h,b_{h+1})}.
\end{align*}
Now by Lemma \ref{lem: LFA accuracy of CR}, with probability at least $1-\delta$, $\max_{\bbb\in\CR^{\pi^{\star}}(\xi)}\sqrt{\mathcal{L}_h^{\pi}(b_h,b_{h+1})}$ is bounded by 
\begin{align}\label{eq: proof coro subopt eq 4}
    &\max_{\bbb\in\CR^{\pi^{\star}}(\xi)}\sqrt{\mathcal{L}_h^{\pi}(b_h,b_{h+1})}\notag
    \\ &\qquad \leq \tilde{C}_2 \cdot \left( 1 + \lambda\right)M_{\BB} M_{\GG}\cdot \sqrt{\frac{dH\log\left(1+L_b L_{\pi} H n/\delta \right)}{n}} + \tilde{C}_2 \cdot M_{\GG}^{1/2}\epsilon_{\BB}^{1/4},\quad \forall h\in[H].
\end{align}
Now we deal with the term (i), (iii), (iv), and (vi), respectively.
To this end, we apply uniform concentration inequalities to bound $J(\pi,\mathbf{b})$ and $\hat{J}(\pi,\mathbf{b})$ uniformly over the $\epsilon$-net of $\pi$ and $\mathbf{b}$ as described in the proof of Lemma \ref{lem: LFA true in CR}.
By Hoeffding's inequality, we have that, with probability at least $1-\delta$, for all $\pi$ and $\bbb$ in their $\epsilon$-nets, 
\begin{align*}
    \left|J(\pi,\mathbf{b})-\hat{J}(\pi,\mathbf{b})\right|&\leq \sqrt{\frac{2M_{\mathbb{B}}^2\log(\cN_{\epsilon,\bbb}\cN_{\epsilon,\pi}/\delta)}{n}},
\end{align*}where $\cN_{\epsilon, \pi}$ and $\cN_{\epsilon, \bbb}$ are the covering numbers defined in Appendix \ref{sec: basics for LFA}.
Here we use the regularity assumption that  $|\sum_{a\in\mathcal{A}}b^{\pi}_{1}(a,w)| \leq M_{\BB}$ for all $w \in \cW$ and the definition of $J(\pi,\mathbf{b})$ from \eqref{eq: define J}. 
Consequently, for all $\pi\in\Pi(\cH)$ and $\bbb\in\BB^{\otimes H}$, we have 
\begin{align}\label{eq: proof coro subopt eq concentration 2}
    \left|J(\pi,\mathbf{b})-\hat{J}(\pi,\mathbf{b})\right|&\leq \sqrt{\frac{2M_{\mathbb{B}}^2\log(\cN_{\epsilon,\bbb}\cN_{\epsilon,\pi}/\delta)}{n}} + 2M_{\BB}\epsilon.
\end{align}

Next, all of (i), (iii), (iv), and (vi) are bounded by the R.H.S. of \eqref{eq: proof coro subopt eq concentration 2}.
Finally, by \eqref{eq: proof coro subopt eq 1}, \eqref{eq: proof coro subopt eq 3}, \eqref{eq: proof coro subopt eq 4} and \eqref{eq: proof coro subopt eq concentration 2}, we have that
\begin{align*}
    &J(\pi^{\star})-J(\hat{\pi}) \notag\\
    &\qquad\leq \textnormal{(i)}+\textnormal{(iii)}+\textnormal{(iv)}+\textnormal{(vi)}+\textnormal{(v)}\notag\\
    &\qquad \leq 2\sqrt{C^{\pi^{\star}}}\sum_{h=1}^H\gamma^{h-1}\left[\tilde{C}_2 \cdot \left( 1 + \lambda\right)M_{\BB} M_{\GG}\cdot \sqrt{\frac{dH\log\left(1+L_b L_{\pi} H n/\delta \right)}{n}} + \tilde{C}_2 \cdot M_{\GG}^{1/2}\epsilon_{\BB}^{1/4}\right]\notag\\
    &\qquad\qquad +4\left[\sqrt{\frac{2M_{\mathbb{B}}^2\log(\cN_{\epsilon,\bbb}\cN_{\epsilon,\pi}/\delta)}{n}} + 2M_{\BB}\epsilon\right].
\end{align*}
Finally, by taking $\epsilon=1/n^2$, and plugging in the values of $\cN_{\epsilon,\bbb}$ and $\cN_{\epsilon,\pi}$ from Lemma \ref{lem: covering B and Pi}, we get
\begin{align*}
    & J(\pi^{\star})-J(\hat{\pi})
    \\&\qquad\leq 2\sqrt{C^{\pi^{\star}}}\sum_{h=1}^H\gamma^{h-1}\left[\tilde{C}_2 \cdot \left( 1 + \lambda\right)M_{\BB} M_{\GG}\cdot \sqrt{\frac{dH\log\left(1+L_b L_{\pi} H n/\delta \right)}{n}} + \tilde{C}_2 \cdot M_{\GG}^{1/2}\epsilon_{\BB}^{1/4}\right]\notag\\
    &\qquad\qquad + C_3 M_{\BB} \sqrt{\frac{dH\log(1+L_b L_{\pi}n/\delta)}{n}},
\end{align*}where $C_3$ is some problem-independent universal constant.
We then simplify the expression and use the fact that
\begin{align*}
    M_{\mathbb{G}} = \sup_{a,z} |g(a,z)| = \sup_{a,z} |\langle \nu(a,z), \omega\rangle| \leq \sup_{a,z}\|\nu(a,z)\|_2\cdot \| \omega\|_2 \leq L_g.
\end{align*}
This gives the result of Corollary \ref{thm: LFA subopt linear}.
\end{proof}

\subsection{Proof of Lemmas in Appendix \ref{sec: proof of linear case}}

\subsubsection{Proof of Lemma \ref{lem: LFA true in CR}}\label{sec: proof of lemma LFA true in CR}

\begin{proof}[Proof of Lemma \ref{lem: LFA true in CR}]
First, for any $\epsilon \in (0,1)$, consider arbitrary $\pi = \{\pi_h\}_{h=1}^H$ and $\pi' = \{\pi'_h\}_{h=1}^H$ in $\Pilin$ such that $\|\pi_h - \pi'_h\|_{\infty, 1} \leq \epsilon$ for all $h\in[H]$. 
And consider arbitrary $\bbb = \{b_h\}_{h=1}^H$ and $\bbb' = \{b'_h\}_{h=1}^H$ in $\BB^{\otimes H}$ such that $\|b_h - b'_h\|_{\infty} \leq \epsilon$ for all $h \in[H]$. 
Then by definition of $\Phi^{\lambda}_{\pi,h}(b_h,b_{h+1};g)$ in \eqref{eq: population phi lambda} and $\hat{\Phi}^{\lambda}_{\pi,h}(b_h,b_{h+1};g)$ in \eqref{eq: empirical phi lambda}, and that $\Phi^{\lambda}_{\pi,h} = \Phi^{0}_{\pi,h}$ and $\hat{\Phi}^{\lambda}_{\pi,h} = \hat{\Phi}^{0}_{\pi,h}$, one can easily get that 
\begin{align}\label{eq: proof LFA CR lemma eq 1}
     \left| \Phi_{\pi,h} (b_h,b_{h+1};g) - \Phi_{\pi',h} (b'_h, b'_{h+1};g)  \right|  
     & \leq \left[ 2\epsilon + \gamma \cdot(\epsilon + \epsilon M_{\BB}) \right]\cdot M_{\GG} \leq 4  M_{\BB} M_{\GG}\epsilon, \notag
    \\ \left| \hat{\Phi}_{\pi,h} (b_h,b_{h+1};g) - \hat{\Phi}_{\pi',h} (b'_h, b'_{h+1};g)  \right|  & \leq \left[ 2\epsilon + \gamma \cdot(\epsilon + \epsilon M_{\BB}) \right]\cdot M_{\GG} \leq 4  M_{\BB} M_{\GG}\epsilon,
\end{align}for all $g \in \GG$. 

Now, same as in the proof of Lemma \ref{lem: true in CR}, we want to show: for any $\pi \in \Pi(\cH)$, 
\begin{align*}
% \label{eq: LFA proof lemma eq 2}
    \max_{g \in \mathbb{G}} \hat{\Phi}_{\pi, h}^\lambda ( b_h^{\pi}, b_{h+1}^{\pi};g)\leq \xi.
\end{align*}The rest of the proof would be very similar to that of Lemma \ref{lem: true in CR} with an additional covering argument. 
To begin with, we again write $\hat{\Phi}_{\pi, h}^\lambda ( b_h^{\pi}, b_{h+1}^{\pi};g) = \hat{\Phi}_{\pi, h} ( b_h^{\pi}, b_{h+1}^{\pi};g) - \lambda \|g\|_{2,n}^2$. 

Same as \eqref{eq: proof lem 3.3 concentraion 1}, we have that with probability at least $1-\delta/2$, 
\begin{align}\label{eq: proof LFA CR lemma eq 3}
    \left|\|g\|_{2,n}^2-\|g\|_2^2\right|\leq \frac{1}{2}\|g\|_2^2+\frac{M_{\mathbb{G}}^2\log(2c_1/\zeta)}{2c_2n},\quad \forall g\in\mathbb{G},
\end{align}
where $\zeta = \min\{\delta, 2c_1\exp(-c_2n\alpha_{\mathbb{G},n}^2/M_{\mathbb{G}}^2)\}$ and $c_1, c_2$ are some universal constants. 

Next, we upper bound $|\hat{\Phi}_{\pi',h}(b_h, b_{h+1};g) - \Phi_{\pi',h}(b_h, b_{h+1};g)|$ for any $\pi \in \Pi(\cH)$,  and $\bbb \in \BB^{\otimes H}$. We first prove this for a fixed $\epsilon$-net of $\Pi(\cH)$ and $\BB^{\otimes H}$. Specifically, choose an $\epsilon$-net of $\Pi(\cH)$ such that for any $\pi=\{\pi_h\}_{h=1}^H$ and $\pi'=\{\pi'_h\}_{h=1}^H$ in this $\epsilon$-net, it holds that $\|\pi_h-\pi'_h\|_{\infty,1} \leq \epsilon$ for all $h$. Also choose an $\epsilon$-net of $\BB^{\otimes H}$ such that for any $\bbb = \{b_h\}_{h=1}^H$ and $\bbb' = \{b'_h\}_{h=1}^H$ in the $\epsilon$-net, it holds that $\|b_h - b'_h\|_{\infty} \leq \epsilon$ for all $h$.  
Denote the cardinality of these two $\epsilon$-net by $\cN_{\epsilon,\pi}$ and $\cN_{\epsilon,\bbb}$, respectively. Then by the same argument behind \eqref{eq: proof lem 3.3 concentraion 2}, we get that, with probability at least $1-\delta/2$, for any $\pi$ and $\bbb$ in their $\epsilon$-nets, and for any $g\in\mathbb{G}$,
\begin{align}\label{eq: proof LFA CR lemma eq 4}
    &\left|\hat \Phi_{\pi,h}(b_h,b_{h+1};g)-\Phi_{\pi,h}(b_h,b_{h+1};g)\right|\notag
    \\ &\qquad\leq 18L\|g\|_2\sqrt{\frac{M_{\mathbb{G}}^2\log\big(2c_1 \cN_{\epsilon,\bbb}^2\cN_{\epsilon,\pi} H/\zeta\big)}{c_2n}}+\frac{18LM_{\mathbb{G}}^2\log\big(2c_1\cN_{\epsilon,\bbb}^2 \cN_{\epsilon,\pi} H/\zeta\big)}{c_2n},
\end{align}
where $\zeta' = \min\{\delta,2c_1 \cN_{\epsilon,\bbb}^2 \cN_{\epsilon,\pi} H\exp(-c_2n\alpha_{\mathbb{G},n}^2/M_{\mathbb{G}}^2)\}$. 

Now for any $\pi \in \Pi(\cH)$ and $\bbb \in \BB^{\otimes H}$, by our construction of the $\epsilon$-nets, we can find a $\pi'$ and $\bbb'$ in the $\epsilon$-nets such that $\|\pi_h - \pi'_h\|_{\infty,1} \leq \epsilon$ and $\|b_h - b'_h\|_\infty \leq \epsilon$ for all $h$. 
Then we have that with probability at least $1-\delta/2$, for any $\pi \in \Pi(\cH)$ and $\bbb \in \BB^{\otimes H}$, and for any $g \in \mathbb{G}$, 
\begin{align}\label{eq: proof LFA CR lemma eq 5}
    &\left|\hat \Phi_{\pi,h}(b_h,b_{h+1};g)-\Phi_{\pi,h}(b_h,b_{h+1};g)\right| \notag
    \\ & \qquad \leq |\hat\Phi_{\pi,h}(b_h,b_{h+1};g) - \hat\Phi_{\pi',h}(b'_h,b'_{h+1};g)|+|\hat\Phi_{\pi',h}(b'_h,b'_{h+1};g) - \Phi_{\pi',h}(b'_h,b'_{h+1};g)| \notag
    \\ & \qquad \qquad + |\Phi_{\pi',h}(b'_h,b'_{h+1};g)-\Phi_{\pi,h}(b_h,b_{h+1};g)| \notag
    \\ & \qquad \leq 8 M_{\BB}M_{\GG} \cdot \epsilon + 18L\|g\|_2\sqrt{\frac{M_{\mathbb{G}}^2\log\big(2c_1 \cN_{\epsilon,\bbb}^2\cN_{\epsilon,\pi} H/\zeta\big)}{c_2n}}+\frac{18LM_{\mathbb{G}}^2\log\big(2c_1\cN_{\epsilon,\bbb}^2 \cN_{\epsilon,\pi} H/\zeta\big)}{c_2n}, 
\end{align}where the first step is by the triangle inequality and the second steps is by \eqref{eq: proof LFA CR lemma eq 1} and \eqref{eq: proof LFA CR lemma eq 4}. 

Now combine \eqref{eq: proof LFA CR lemma eq 3} and \eqref{eq: proof LFA CR lemma eq 5} with a union bound, we conclude that, with probability at least $1-\delta$, for any $\pi \in \Pi(\cH)$,
\begin{align}\label{eq: proof LFA CR lemma eq 6}
    &\max_{g \in \mathbb{G}} \hat{\Phi}_{\pi, h}^\lambda ( b_h^{\pi}, b_{h+1}^{\pi};g)\notag\\
    &\qquad =\max_{g \in \mathbb{G}}\left\{\hat{\Phi}_{\pi, h} ( b_h^{\pi}, b_{h+1}^{\pi};g)-\lambda\|g\|_{2,n}^2\right\}\notag\\
    &\qquad \leq \max_{g \in \mathbb{G}}\Bigg\{\Phi_{\pi, h} ( b_h^{\pi}, b_{h+1}^{\pi};g)-\lambda\|g\|_{2}^2
    +\frac{\lambda}{2}\|g\|_2^2+\frac{\lambda M_{\mathbb{G}}^2\log(2c_1/\zeta)}{2c_2n},\notag\\
    &\qquad\qquad +18L\|g\|_2\sqrt{\frac{M_{\mathbb{G}}^2\log\big(2c_1 \cN_{\epsilon,\bbb}^2\cN_{\epsilon,\pi} H/\zeta\big)}{c_2n}}+\frac{18LM_{\mathbb{G}}^2\log\big(2c_1\cN_{\epsilon,\bbb}^2 \cN_{\epsilon,\pi} H/\zeta\big)}{c_2n}\bigg\}+8M_{\BB} M_{\GG} \epsilon \notag\\
    &\qquad \leq \max_{g\in\mathbb{G}}\Phi_{\pi,h}( b_h^{\pi}, b_{h+1}^{\pi};g)+\max_{g\in\mathbb{G}}\Bigg\{
    -\frac{\lambda}{2}\|g\|_2^2+18L\|g\|_2\sqrt{\frac{M_{\mathbb{G}}^2\log\big(2c_1 \cN_{\epsilon,\bbb}^2\cN_{\epsilon,\pi} H/\zeta\big)}{c_2n}}\Bigg\}\notag\\
    &\qquad\qquad +\frac{\lambda M_{\mathbb{G}}^2\log(2c_1/\zeta)}{2c_2n}+ \frac{18LM_{\mathbb{G}}^2\log\big(2c_1\cN_{\epsilon,\bbb}^2 \cN_{\epsilon,\pi} H/\zeta\big)}{c_2n}+8M_{\BB} M_{\GG} \epsilon
    \notag\\
    &\qquad
    \leq \frac{728L^2M_{\mathbb{G}}^2\log(2c_1\cN_{\epsilon,\bbb}^2\cN_{\epsilon,\pi} H/\zeta^{\prime})}{\lambda n} + \frac{\lambda M_{\mathbb{G}}^2\log(2c_1/\zeta)}{2c_2n} \notag
    \\ & \qquad\qquad+ \frac{18LM_{\mathbb{G}}^2\log(2c_1\cN_{\epsilon,\bbb}^2\cN_{\epsilon,\pi}H/\zeta^{\prime})}{c_2n}+8M_{\BB} M_{\GG} \epsilon,
\end{align}
with $\zeta = \min\{\delta, 2c_1\exp(-c_2n\alpha_{\mathbb{G},n}^2/M_{\mathbb{G}}^2)\}$ and $\zeta^{\prime} = \min\{\delta,2c_1\cN_{\epsilon,\bbb}^2\cN_{\epsilon,\pi}H\exp(-c_2n\alpha_{\mathbb{G},n}^2/M_{\mathbb{G}}^2)\}$
for any policy $\pi\in\Pi(\cH)$ and step $h$. Here the first inequality is by \eqref{eq: proof LFA CR lemma eq 3} and \eqref{eq: proof LFA CR lemma eq 5}, the second inequality is trivial, and the last inequality holds from the fact that $\Phi_{\pi, h} ( b_h^{\pi}, b_{h+1}^{\pi};g)=0$ and the fact that
$\sup_{\|g\|_2}\{a\|g\|_2-b\|g\|_2^2\}\leq a^2/4b$. 

Now by Definition \ref{def: LFA}, we apply Lemma \ref{lem: covering B and Pi} with $\|\theta_h\|_2 \leq L_b$ and $\|\beta_h\| \leq L_\pi$ and get that
\begin{align}\label{eq: proof LFA CR lemma eq 7}
    \log \cN_{\epsilon,\pi} & \leq dH \log\left( 1+ \frac{16L_\pi}{\epsilon} \right), \notag 
    \\ \log \cN_{\epsilon, b} & \leq d H \log\left( 1 + \frac{2L_b}{\epsilon} \right).
\end{align}
Now we pick $\epsilon = 1/n^2$, and together with \eqref{eq: proof LFA CR lemma eq 6} and \eqref{eq: proof LFA CR lemma eq 7}, we get that 
\begin{align*}
    &\max_{g \in \mathbb{G}} \hat{\Phi}_{\pi, h}^\lambda ( b_h^{\pi}, b_{h+1}^{\pi};g) \leq C\cdot \frac{\left(\lambda+1/\lambda\right)M_{\BB}^2 M_{\GG}^2\left[ dH\log\left(1 + L_b L_\pi H n/\delta\right) + n \alpha_{\GG,n}^2/M_{\GG}^2\right]}{n} + C\cdot \frac{M_{\BB}M_{\GG}}{n^2}, 
\end{align*}where $C$ is some universal constant. Here we have plugged in the value of $\zeta$, $\zeta^{\prime}$ and $L=2M_{\mathbb{B}}$. Finally, by plugging in the value of $\alpha_{\GG,n}$ from Lemma \ref{lem: critical radius of G linear}, we conclude that 
\begin{align*}
    \max_{g \in \mathbb{G}} \hat{\Phi}_{\pi, h}^\lambda ( b_h^{\pi}, b_{h+1}^{\pi};g)\notag
    \leq C_1\cdot \left(\lambda + \frac{1}{\lambda}\right) \cdot \frac{M_{\BB}^2 M_{\GG}^2 dH \log\left(1 + L_b L_\pi H n/\delta\right) }{n} + C_1\cdot \frac{M_{\BB}M_{\GG}}{n^2},
\end{align*}where $C_1$ is some problem-independent constant. Note that second term on the right hand side is smaller than the first term. Then the result follows from our choice of $\xi$ in Lemma \ref{lem: LFA true in CR}.

\end{proof}

\subsubsection{Proof of Lemma \ref{lem: LFA accuracy of CR}} \label{sec: proof of lemma LFA accuracy of CR}

\begin{proof}[Proof of Lemma \ref{lem: LFA accuracy of CR}]

Consider any $\pi \in \Pi(\cH)$ and $\bbb=\{b_h\}_{h=1}^H \in \CR^{\pi}(\xi)$. Same as \eqref{eq: proof lem 3.3 eq 3}, we have
\begin{align*}
% \label{eq: proof LFA loss lemma eq 1}
    &\max_{g \in \mathbb{G}} \hat{\Phi}_{\pi, h}^\lambda (b_h, b_{h+1};g)\notag\\
    &\qquad\geq\underbrace{\max_{g \in \mathbb{G}} \Big\{
    \hat{\Phi}_{\pi, h} (b_h, b_{h+1};g)-
    \hat{\Phi}_{\pi, h} (b^{\star}_h(b_{h+1}), b_{h+1};g)-2\lambda\|g\|_{2,n}^2\Big\}}_{\displaystyle{(\star)}}\notag
    \\ &\qquad\qquad-\max_{g\in\mathbb{G}}
    \hat{\Phi}_{\pi, h}^{\lambda} (b^{\star}_h(b_{h+1}), b_{h+1};g).
\end{align*}
We again upper and lower bound term $(\star)$ respectively.

\vspace{3mm}
\noindent
\textbf{Upper bound of term ($\star$).} By the same argument as in the proof of Lemma \ref{lem: Phi b star bound}, we have that: for any $\bbb \in \BB^{\otimes H}$, $\pi\in\Pi(\cH)$, and $h\in[H]$, it holds with probability at least $1-\delta/2$ that 
\begin{align*}
    \max_{g\in\mathbb{G}}\hat{\Phi}_{\pi,h}(b^{\star}_h(b_{h+1}),b_{h+1};g)\leq \xi + \epsilon_{\BB}^{1/2} M_{\GG} , 
\end{align*}
where $b^{\star}_h(b_{h+1})$ is defined in \eqref{eq: def b star} and $\xi$ is defined in Lemma \ref{lem: LFA true in CR}. We then get
\begin{align}\label{eq: proof LFA loss lemma eq 2}
    \max_{g\in\mathbb{G}}\hat{\Phi}^{\lambda}_{\pi, h}(\hat{b}_h(b_{h+1}),b_{h+1};g)\leq \max_{g\in\mathbb{G}}\hat{\Phi}^{\lambda}_{\pi, h}(b^{\star}_h(b_{h+1}),b_{h+1};g)\leq \xi + \epsilon_{\BB}^{1/2} M_{\GG},
\end{align}where the first inequality follows from the definition of $\hat b_h(b_{h+1})$ in \eqref{eq: hat b}. Also note that, by the construction of the confidence region $\CR^\pi(\xi)$, we have 
\begin{align}\label{eq: proof LFA loss lemma eq 3}
    \max_{g\in\mathbb{G}}\hat{\Phi}^{\lambda}_{\pi, h}(b_h,b_{h+1};g)-\max_{g\in\mathbb{G}}\hat{\Phi}^{\lambda}_{\pi, h}(\hat{b}_h(b_{h+1}),b_{h+1};g)\leq \xi.
\end{align}
Furthermore, we can write 
\begin{align*}
    (\star)&\leq 
    \max_{g\in\mathbb{G}}\hat{\Phi}^{\lambda}_{\pi, h}(b^{\star}_h(b_{h+1}),b_{h+1};g)+\max_{g\in\mathbb{G}}\hat{\Phi}^{\lambda}_{\pi, h}(b_h,b_{h+1};g)\notag\\
    &\leq \max_{g\in\mathbb{G}}\hat{\Phi}^{\lambda}_{\pi, h}(b^{\star}_h(b_{h+1}),b_{h+1};g)\notag\\
    &\qquad +\max_{g\in\mathbb{G}}\hat{\Phi}^{\lambda}_{\pi, h}(b_h,b_{h+1};g)-\max_{g\in\mathbb{G}}\hat{\Phi}^{\lambda}_{\pi, h}(\hat{b}_h(b_{h+1}),b_{h+1};g)\notag\\
    &\qquad +\max_{g\in\mathbb{G}}\hat{\Phi}^{\lambda}_{\pi, h}(\hat{b}_h(b_{h+1}),b_{h+1};g). 
\end{align*}Combining with \eqref{eq: proof LFA loss lemma eq 2} and \eqref{eq: proof LFA loss lemma eq 3}, we get that, with probability at least $1-\delta/2$, 
\begin{align}\label{eq: proof LFA loss lemma eq 5}
    (\star)&\leq 3 \xi + 2\epsilon_{\BB}^{1/2} M_{\GG}. 
\end{align}

\vspace{3mm}
\noindent
\textbf{Lower bound of term ($\star$).}
First of all, same as \eqref{eq: proof lem 3.4 concentraion 1}, it holds with probability at least $1-\delta/4$ that,
\begin{align}\label{eq: proof LFA loss lemma eq 6}
    \left|\|g\|_{2,n}^2-\|g\|_2^2\right|\leq \frac{1}{2}\|g\|_2^2+\frac{M_{\mathbb{G}}^2\log(4c_1/\zeta)}{2c_2n},\quad \forall g\in\mathbb{G},\end{align}
where $\zeta = \min\{\delta, 4c_1\exp(-c_2n\alpha_{\mathbb{G},n}^2/M_{\mathbb{G}}^2)\}$ for some absolute constants $c_1$ and $c_2$, and $\alpha_{\mathbb{G},n}$ is the critical radius of $\mathbb{G}$ defined in Assumption \ref{assump: dual function class}. 

Second, we fix an $\epsilon$-net of $\Pi(\cH)$ and an $\epsilon$-net of $\BB^{\otimes H}$, as described in Appendix \ref{sec: basics for LFA}. 
Denote by $\cN_{\epsilon,\pi}$ and $\cN_{\epsilon,\bbb}$ their respective covering numbers.
Then by the same argument behind \eqref{eq: proof lem 3.4 concentraion 2} and a union bound, we get that, with probability at least $1-\delta/4$, for all $\pi=\{\pi_h\}_{h=1}^H$, $\bbb=\{b_h\}_{h=1}^H$ and $\bbb'=\{b'_h\}_{h=1}^H$ in their $\epsilon$-nets, and for all $g\in\mathbb{G}$, 
\begin{align}\label{eq: proof LFA loss lemma eq 7}
&\left|\Big(\hat{\Phi}_{\pi, h} (b_h, b_{h+1};g)-
    \hat{\Phi}_{\pi, h} (b^{\prime}_h, b_{h+1};g)\Big)-\Big(\Phi_{\pi, h} (b_h, b_{h+1};g)-
    \Phi_{\pi, h} (b^{\prime}_h, b_{h+1};g)\Big)\right|\notag\\
    &\qquad\leq 18L\|g\|_2\sqrt{\frac{M_{\mathbb{G}}^2\log(4c_1\cN_{\epsilon,\bbb}^3\cN_{\epsilon,\pi}H/\zeta^{\prime})}{c_2n}}+\frac{18LM_{\mathbb{G}}^2\log\big(4c_1\cN_{\epsilon,\bbb}^3\cN_{\epsilon,\pi}H/\zeta^{\prime}\big)}{c_2n},
\end{align}
where $\zeta^{\prime} = \min\{\delta,4c_1\cN_{\epsilon,\bbb}^3\cN_{\epsilon,\pi}H\exp(-c_2n\alpha_{\mathbb{G},n}^2/M_{\mathbb{G}}^2)\}$.

We then use \eqref{eq: proof LFA CR lemma eq 1}, and conclude that, with probability at least $1-\delta/4$, for all $\pi \in \Pi(\cH)$, and $\bbb$, $\bbb' \in \BB^{\otimes H}$, and $g \in \mathbb{G}$, 
\begin{align}\label{eq: proof LFA loss lemma phi diff}
&\left|\Big(\hat{\Phi}_{\pi, h} (b_h, b_{h+1};g)-
    \hat{\Phi}_{\pi, h} (b^{\prime}_h, b_{h+1};g)\Big)-\Big(\Phi_{\pi, h} (b_h, b_{h+1};g)-
    \Phi_{\pi, h} (b^{\prime}_h, b_{h+1};g)\Big)\right|\notag\\
    &\qquad\leq 18L\|g\|_2\sqrt{\frac{M_{\mathbb{G}}^2\log(4c_1\cN_{\epsilon,\bbb}^3\cN_{\epsilon,\pi}H/\zeta^{\prime})}{c_2n}}+\frac{18LM_{\mathbb{G}}^2\log\big(4c_1\cN_{\epsilon,\bbb}^3\cN_{\epsilon,\pi}H/\zeta^{\prime}\big)}{c_2n} + 8 M_{\BB} M_{\GG} \epsilon.
\end{align}

In the sequel, for simplicity, we denote that
\begin{align}\label{eq: proof LFA loss lemma eq 8}
    \iota_n \coloneqq\sqrt{\frac{M_{\mathbb{G}}^2\log(4c_1\cN_{\epsilon,\bbb}^3\cN_{\epsilon,\pi}H/\zeta^{\prime})}{c_2n}}, \quad \iota_n^{\prime}\coloneqq\sqrt{\frac{M_{\mathbb{G}}^2\log(4c_1/\zeta)}{2c_2n}},
\end{align}where $\zeta$ and $\zeta'$ are same as in \eqref{eq: proof LFA loss lemma eq 6} and \eqref{eq: proof LFA loss lemma eq 7}. 
Furthermore, given fixed $b_h,b_{h+1}\in\mathbb{B}$, we denote 
\begin{align}\label{eq: LFA def ghpi}
    g_h^{\pi}\coloneqq\frac{1}{2\lambda}\ell_h^{\pi}(b_h,b_{h+1})\in\mathbb{G},
\end{align}
where $\ell_h^\pi$ is defined by \eqref{eq: ell} and  $g_h^{\pi}\in\mathbb{G}$ follows from Assumption \ref{assump: completeness and realizability}.
We then have  
\begin{align*}
    (\star)&=\max_{g \in \mathbb{G}} \Big\{
    \hat{\Phi}_{\pi, h} (b_h, b_{h+1};g)-
    \hat{\Phi}_{\pi, h} (b^{\star}_h(b_{h+1}), b_{h+1};g)-2\lambda\|g\|_{2,n}^2\Big\}\notag\\
    &\geq 
    \hat{\Phi}_{\pi, h} (b_h, b_{h+1};g_h^{\pi}/2)-
    \hat{\Phi}_{\pi, h} (b^{\star}_h(b_{h+1}), b_{h+1}; g_h^{\pi}/2)-\frac{\lambda}{2}\| g_h^{\pi}\|_{2,n}^2,
\end{align*}where the inequality holds because $g_h^\pi/2 \in \GG$. 

Together with \eqref{eq: proof LFA loss lemma eq 6} and \eqref{eq: proof LFA loss lemma phi diff}, we have 
\begin{align}\label{eq: proof LFA loss lemma star lower 1}
    (\star)&\geq \Phi_{\pi, h} (b_h, b_{h+1}; g_h^{\pi}/2)-
    \Phi_{\pi, h} (b^{\star}_h(b_{h+1}), b_{h+1}; g_h^{\pi}/2)-18L\iota_n\|g_h^{\pi}\|_2
    -18L\iota_n^2 \notag
    \\ &\qquad- 8 M_{\BB} M_{\GG} \epsilon -\frac{\lambda}{2}\left(\frac{3}{2}\|g_h^{\pi}\|_2^2+\iota_n^{\prime 2}\right)\notag\\
    &\geq\lambda\|g_h^{\pi}\|_2^2-18L\iota_n\|g_h^{\pi}\|_2-\epsilon_{\mathbb{B}}^{1/2}M_{\mathbb{G}}
    -18L\iota_n^2-8M_{\BB}M_{\GG}\epsilon-\frac{\lambda}{2}\left(\frac{3}{2}\|g_h^{\pi}\|_2^2+\iota_n^{\prime 2}\right)\notag\\
    &=\frac{\lambda}{4}\|g_h^{\pi}\|_2^2-18L\iota_n\|g_h^{\pi}\|_2-\epsilon_{\mathbb{B}}^{1/2}M_{\mathbb{G}}-18L\iota_n^2- 8M_{\BB}M_{\GG}\epsilon-\frac{\lambda}{2}\iota_n^{\prime 2} ,
\end{align}where the second inequality follows from the same reason as in \eqref{eq: lower bound}. 

Finally, combine \eqref{eq: proof LFA loss lemma star lower 1} and \eqref{eq: proof LFA loss lemma eq 5} and we get 
\begin{align*}
   \frac{\lambda}{4}\|g_h^{\pi}\|_2^2-18L\iota_n\|g_h^{\pi}\|_2-\epsilon_{\mathbb{B}}^{1/2}M_{\mathbb{G}}-18L\iota_n^2- 8M_{\BB}M_{\GG}\epsilon-\frac{\lambda}{2}\iota_n^{\prime 2} \leq 3 \xi + 2\epsilon_{\BB}^{1/2} M_{\GG}.
\end{align*}This gives the following quadratic inequality w.r.t. $\|g_h^\pi\|_2$
\begin{align*}
    \lambda\|g_h^\pi\|_2^2 - \underbrace{72L \iota_n}_{\displaystyle{\mathrm{A}}} \|g_h^\pi\|_2 -\underbrace{ 4 \left[ 18L\iota_n^2+\frac{\lambda}{2}{\iota'_n}^2+3\xi+8M_{\BB}M_{\GG}\epsilon+3\epsilon_{\BB}^{1/2}M_{\GG} \right] }_{\displaystyle{\mathrm{B}}} \leq 0.
\end{align*} By the fact that $x^2-\mathrm{A}x-\mathrm{B}\leq 0 $ implies $ x\leq (\mathrm{A}+\sqrt{\mathrm{A}^2+4\mathrm{B}})/2 \leq \mathrm{A} + \sqrt{\mathrm{B}}$, we have 
\begin{align*}
    \|g_h^\pi\|_2 & \leq \frac{72L \iota_n}{\lambda} + \sqrt{\frac{4}{\lambda} \left[ 18L\iota_n^2+\frac{\lambda}{2}{\iota'_n}^2+3\xi+8M_{\BB}M_{\GG}\epsilon+3\epsilon_{\BB}^{1/2}M_{\GG} \right]}.
\end{align*}
We then plug in the values of $\iota_n$ and $\iota'_n$ from \eqref{eq: proof LFA loss lemma eq 8}, $\xi$ from Lemma \ref{lem: LFA true in CR}, $\zeta$ and $\zeta'$ from below \eqref{eq: proof LFA loss lemma eq 6} and \eqref{eq: proof LFA loss lemma eq 7}, $\cN_{\epsilon,\bbb}$ and $\cN_{\epsilon,\pi}$ from Lemma \ref{lem: covering B and Pi}, $\alpha_{\GG,n}$ from Lemma \ref{lem: critical radius of G linear}, and set $\epsilon = 1/n^2$. Simplify the expression and we get 
\begin{align*}
    \|g_h^\pi\|_2 & \leq C \cdot \left( 1 + \frac{1}{\lambda}\right)M_{\BB} M_{\GG}\cdot \sqrt{\frac{dH\log\left(1+L_b L_{\pi} H n/\delta \right)}{n}} + C\cdot \frac{M_{\GG}^{1/2}\epsilon_{\BB}^{1/4}}{\lambda},
\end{align*}where $C$ is some problem-independent universal constant. By \eqref{eq: LFA def ghpi} and \eqref{eq: RMSE loss}, we have ${\mathcal{L}_h^{\pi}(b_h,b_{h+1})}=\|2 \lambda g_h^{\pi}\|_2^2$. It follows that
\begin{align*}
    \sqrt{\mathcal{L}_h^{\pi}(b_h,b_{h+1})}=2\lambda\|g_h^{\pi}\|_2\leq \tilde{C}_2 \cdot \left( 1 + \lambda\right)M_{\BB} M_{\GG}\cdot \sqrt{\frac{dH\log\left(1+L_b L_{\pi} H n/\delta \right)}{n}} + \tilde{C}_2 \cdot M_{\GG}^{1/2}\epsilon_{\BB}^{1/4},
\end{align*}for some constant $\tilde{C}_2$. 
This finishes the proof.
\end{proof}

\section{Auxiliary Lemmas}\label{sec: auxiliary lemmas}

% \subsection{Lemmas for Localized Uniform Concentration}

We introduce some useful lemmas for the uniform concentration over function classes.
Before we present the lemmas, we first introduce several notations. For a function class $\mathcal{F}$ on a probability space $(\mathcal{X}, P)$, we denote by $\|f\|_2^2$ the expectation of $f(X)^2$, that is $\|f\|_2^2=\mathbb{E}_{X\sim\mathcal{P}}[f(X)^2]$. Also, we denote by 
\begin{align}
    \mathcal{R}_n(\mathcal{F},\delta)\coloneqq\mathbb{E}\left[\sup_{f\in\mathcal{F}:\|f\|_2\leq \delta}\left|\frac{1}{n}\sum_{i=1}^n\epsilon_i f(X_i)\right|\right]
\end{align}
the localized Rademacher complexity of $\mathcal{F}$ with scale $\delta>0$ and size $n\in\mathbb{N}$.
Here $\{\epsilon_i\}_{i=1}^b$ and $\{X_i\}_{i=1}^n$ are i.i.d. and independent. Each $\epsilon_i$ is uniformly distributed on $\{+1,-1\}$ and each $X_i$ is distributed according to $P$.
Finally, we denote by $\textnormal{star}(\mathcal{F})$ the star-shaped set induced by set $\mathcal{F}$ as
\begin{align}
    \textnormal{star}(\mathcal{F})=\left\{\alpha f:\alpha\in[0,1],\,f\in\mathcal{F}\right\}.
\end{align}
Now we are ready to present the lemmas for uniform concentration inequalities.

\begin{lemma}[Localized Uniform Concentration 1 \citep{wainwright2019high}]\label{lem: localized uniform concentration}
Given a star-shaped and $b$-uniformly bounded function class $\mathcal{F}$, let $\delta_{n}$ be any positive solution of the inequality
$$
\mathcal{R}_{n}(\mathcal{F}; \delta) \leq \frac{\delta^{2}}{b}.
$$
Then for any $t \geq \delta_{n}$, we have that
$$
\left|\|f\|_{n}^{2}-\|f\|_{2}^{2}\right| \leq \frac{1}{2}\|f\|_{2}^{2}+\frac{1}{2}t^{2}, \quad \forall f \in \mathcal{F}
$$
with probability at least $1-c_{1} \exp(-c_{2} n t^{2}/b^{2})$. If in addition $n \delta_{n}^{2} \geq 2\log \left(4 \log \left(1 / \delta_{n}\right)\right)/c_2$, then we have that
$$
\big|\|f\|_{n}-\|f\|_{2}\big| \leq c_{0} \delta_{n}, \quad \forall f \in \mathcal{F}
$$
with probability at least $1-c_{1}^{\prime} \exp(-c_{2}^{\prime} n \delta_{n}^{2}/b^{2})$.
\end{lemma}

\begin{proof}[Proof of Lemma \ref{lem: localized uniform concentration}]
See Theorem 14.1 of \cite{wainwright2019high} for a detailed proof.
\end{proof}

\begin{lemma}[Localized Uniform Concentration 2 \citep{foster2019orthogonal}]\label{lem: localized uniform concentration 2}
Consider a star-shaped function class $\mathcal{F}:\mathcal{X}\mapsto\mathbb{R}$ with $\sup _{f \in \mathcal{F}}\|f\|_{\infty} \leq b$, and pick any $f^{\star} \in \mathcal{F}$. 
Also, consider a loss function $\ell:\mathbb{R}\times\mathcal{Y}\mapsto\mathbb{R}$ which is $L$-Lipschitz in its first argument with respect to the $\|\cdot\|_2$-norm.
Now let $\delta_{n}^{2} \geq 4  \log \left(41 \log \left(2 c_{2} n\right)\right)/(c_{2} n)$ be any solution to the inequality:
$$
\mathcal{R}_n(\operatorname{star}(\mathcal{F}-f^{\star});\delta) \leq \frac{\delta^{2}}{b}.
$$
Then for any $t\geq \delta_n$ and some absolute constants $c_{1}, c_{2}>0$, with probability $1-c_{1} \exp (-c_{2} n t^{2}/b^2)$ it holds that
\begin{align}\label{eq: localized uniform concentration 2 1}
&\Big|\Big(\hat{\mathbb{E}}_n[\ell(f(x), y)]-\hat{\mathbb{E}}_{n}\left[\ell\left(f^{\star}(x), y\right)\right]\Big)-\Big(\mathbb{E}[\ell(f(x), y)]-\mathbb{E}\left[\ell\left(f^{\star}(x), y\right)\right]\Big)\Big| \notag
\\ &\qquad\leq 18Lt\left(\left\|f-f^{\star}\right\|_{2}+t\right),
\end{align}
for any $f\in\mathcal{F}$.
If furthermore, the loss function $\ell$ is linear in $f$, i.e., $\ell((f+f^{\prime})(x),y)=\ell(f(x),y)+\ell(f^{\prime}(x),y)$ and $\ell(\alpha f(x),y)=\alpha \ell(f(x),z)$, then the lower bound on $\delta_{n}^{2}$ is not required.
\end{lemma}

\begin{proof}[Proof of Lemma \ref{lem: localized uniform concentration 2}]
See Lemma 11 of \cite{foster2019orthogonal} for a detailed proof.
\end{proof}

\begin{remark}
We note that in the original Lemma 11 of \cite{foster2019orthogonal}, inequality \eqref{eq: localized uniform concentration 2 1} only holds for $\delta_n$, and we extend it to any $t\geq\delta_n$ since according to Lemma 13.6 of \cite{wainwright2019high} we know that $\mathcal{R}_n(\mathcal{F};\delta)/\delta$ is a non-increasing function of $\delta$ on $(0,+\infty)$, which indicates that $t\geq \delta_n$ also solves the inequality.
\end{remark}

\end{document}